\documentclass{article}

 \usepackage[preprint]{neurips_2025}

\usepackage[utf8]{inputenc} 
\usepackage[T1]{fontenc}    
\usepackage{hyperref}       
\usepackage{url}            
\usepackage{booktabs}       
\usepackage{amsfonts}       
\usepackage{nicefrac}       
\usepackage{microtype}      
\usepackage{xcolor}         
\usepackage{algorithm}
\usepackage{algorithmic}
\usepackage{graphicx}
\usepackage{subcaption}

\usepackage{amsmath}
\usepackage{amssymb}
\usepackage{mathtools}
\usepackage{amsthm}
\usepackage{tikz}
\usetikzlibrary{tikzmark}
\usepackage{multirow}
\usepackage{pifont}
\usepackage{adjustbox}
\usepackage[table]{xcolor}
\usepackage[normalem]{ulem}

\makeatletter
\AtBeginDocument{
\renewcommand{\sectionautorefname}{\S\@gobble}

\renewcommand{\subsectionautorefname}{\S\@gobble} 
\renewcommand{\subsubsectionautorefname}{\S\@gobble}
\renewcommand{\appendixautorefname}{Appendix \S\@gobble}

}
\makeatother

\usepackage[capitalize,noabbrev]{cleveref}

\theoremstyle{plain}
\newtheorem{theorem}{Theorem}[section]

\newtheorem{corollary}[theorem]{Corollary}
\theoremstyle{definition}

\theoremstyle{remark}
\newtheorem{remark}[theorem]{Remark}
\usepackage[textsize=tiny]{todonotes}

\definecolor{darkblue}{rgb}{0, 0, 0.5}
\hypersetup{colorlinks=true, citecolor=darkblue, linkcolor=darkblue, urlcolor=darkblue}

\title{Believe Your Model: \\Distribution-Guided Confidence Calibration}

\author{%
  Xizhong Yang$^1$\thanks{Work done during an internship at Kuaishou Technology.}\quad
  Haotian Zhang$^2$\quad
  \setcounter{footnote}{1}
  Huiming Wang$^2$\thanks{Corresponding author: Huiming Wang, Mofei Song.}\quad 
  $\text{Mofei Song}^{1\dag}$ \\[6pt]
  $^1$Southeast University, \quad $^2$Kuaishou Technology\\
  $^\dag$\texttt{huiming\_wang@mymail.sutd.edu.sg}, \quad $^\dag$\texttt{songmf@seu.edu.cn}
}

\begin{document}

\maketitle

\begin{abstract}
    Large Reasoning Models have demonstrated remarkable performance with the advancement of test-time scaling techniques, which enhances prediction accuracy by generating multiple candidate responses and selecting the most reliable answer. While prior work has analyzed that internal model signals like confidence scores can partly indicate response correctness and exhibit a distributional correlation with accuracy, such distributional information has not been fully utilized to guide answer selection.
    Motivated by this, we propose \textit{DistriVoting}, which incorporates distributional priors as another signal alongside confidence during voting. 
    Specifically, our method (1) first decomposes the mixed confidence distribution into positive and negative components using Gaussian Mixture Models, (2) then applies a reject filter based on positive/negative samples from them to mitigate overlap between the two distributions.
    Besides, to further alleviate the overlap from the perspective of distribution itself, we propose \textit{SelfStepConf}, which uses step-level confidence to dynamically adjust inference process, increasing the separation between the two distributions to improve the reliability of confidences in voting.
    Experiments across 16 models and 5 benchmarks demonstrate that our method significantly outperforms state-of-the-art approaches. Code available at \href{https://github.com/yxizhong/SSC}{https://github.com/yxizhong/DistriVoting}.
    
    
\end{abstract}

\section{Introduction}
\label{sec:introduction}
    The advent of techniques like Chain of Thought~\cite{CoT, ReAct} and Test-Time Scaling (TTS)~\cite{TTSSurvey} has led to remarkable performance improvements in Large Reasoning Models (LRMs). However, even though TTS can generate multiple answers or increase token overhead for the same question by increasing test-time computation, enhancing LRMs' test-time performance remains a critical research direction. The core issue is due to the lack of label or reward signals during the test-time phase, which makes it difficult to evaluate the quality of generated answers and dynamically adjust the generation process. To address this problem, some approaches, like s1~\cite{s1}, enhance generation by extending the model's reasoning process without introducing any feedback signals. Other approaches, such as MoB~\cite{MoB} and DORA~\cite{DORA}, utilize external reward models to score multiple generated results and aggregate them to obtain the final answer. In contrast, methods like Self-Consistency~\cite{Self-Consistency}, BoN~\cite{BoN}, Self-Certainty~\cite{Self-Certainty}, and DeepConf~\cite{DeepConf} has demonstrated that internal information can be effective enough to predict the quality of different answers, without introducing additional models. This category of approach balances effectiveness and efficiency, and its benefits become more pronounced with the improvement of the internal information reliability.
    
    Nevertheless, current work on applying internal information to TTS mainly focuses on designing better ways to obtain internal information ~\cite{ConfSurvey}, such as using sentence-level probabilistic measures like perplexity and average log-probability, or using distributional information like Kullback-Leibler divergence and entropy. We provide a more detailed introduction of related works in \autoref{sec:related_works}. 
    While these studies have identified a distributional relationship between internal confidence and answer correctness, where correct and incorrect trajectories typically follow distinct statistical distributions, this observation has only been used to assess the reliability of confidence scores. We argue that explicitly incorporating this distributional prior can further improve the performance of confidence-based voting during the answer selection phase.


    \begin{figure}[t]
        \begin{center}
        \centerline{\includegraphics[width=0.7\columnwidth]{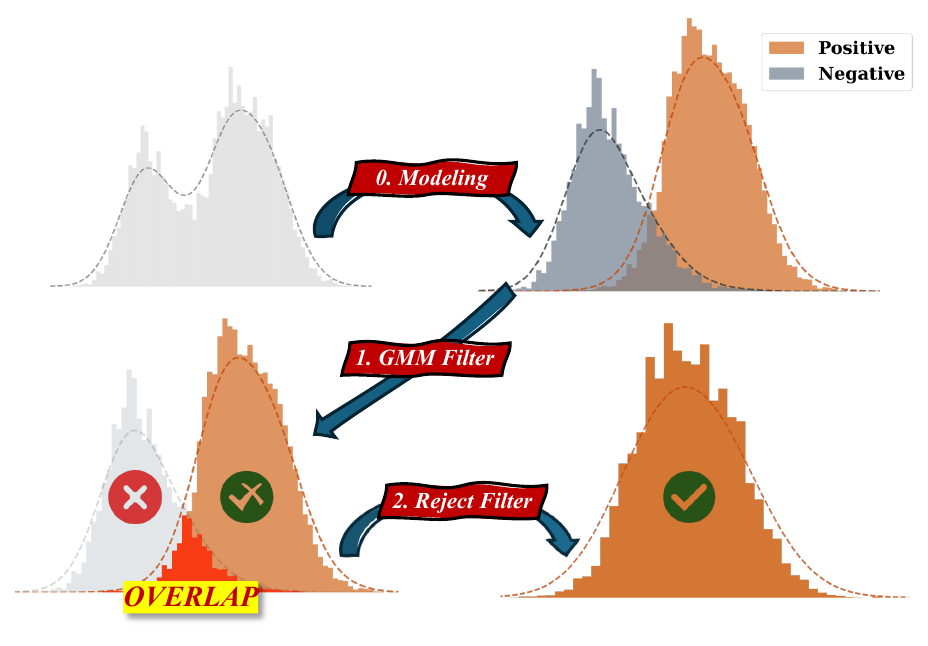}}
        \caption{
            Overview of \textit{DistriVoting}. First model the original distribution using GMM, then filter out negative samples, finally reject false positives from the positive distribution.
        }
        \label{fig:main_fig}
        \end{center}
        \vskip -0.4in
    \end{figure}
    
    
    Motivated by this, we propose \textit{DistriVoting}, which enhances the reliability of confidence in the final voting through a two-step filtering process. Specifically, the process begins by modeling the confidence distribution as a mixture of two Gaussian components using Gaussian Mixture Model (GMM). Subsequently, the distribution is decomposed into positive and negative parts, effectively selecting potential positive answers and filtering out likely incorrect negative answers. However, significant overlap often persists between the two distributions, primarily due to high-confidence incorrect samples and low-confidence correct samples, which may still introduce false positives during the final voting process. 
    To address this issue, we show that the answer voted from negative distribution can be used to reject false positives. 
    Additionally, in order to further alleviate the overlap from a distributional perspective, we propose \textit{SelfStepConf}, which applies confidence to the generation process of individual trajectories, providing real-time supervision signals to intervene in the inference process.


    To evaluate the effectiveness of \textit{DistriVoting} and \textit{SelfStepConf}, we conducted experiments using 16 different models include DeepSeek-Series~\cite{DeepSeek-r1} and Qwen3-Series with thinking and non-thinking modes~\cite{Qwen3}, on 5 reasoning benchmarks include HMMT2025, BRUMO2025~\cite{HMMT&BRUMO}, GPQA-D~\cite{GPQA-D} and AIME~\cite{AIME}. The results show that our proposed \textit{DistriVoting} consistently demonstrates performance gains across different models and benchmarks. Furthermore, \textit{SelfStepConf} increases the separation between the positive and negative distributions, further boosting these improvements.

\section{Preliminaries}
\label{sec:preliminaries}

\subsection{Confidence of Trajectory}


    Prior work shows that LLM's token distributions $\mathbf{P}_i(j)$ can reveal uncertainty and trajectory quality. Following DeepConf~\cite{DeepConf}, we compute confidence via token negative log-probabilities to assess quality during and after inference. For a single trajectory containing $N$ tokens output by the model, we define the trajectory confidence as:
    \begin{equation}
    \label{eq:traj_conf}
    C_{\text{traj}} = -\frac{1}{N_G \times k}\sum_{i \in G}\sum_{j=1}^k \log \mathbf{P}_i(j),
    \end{equation}
    
    where $N_G, k \in \mathbb{N}^+$ and $\mathbf{P} \in \mathbb{R}^{N \times N_v}$.
    Here, $G$ represents the subset of tokens among the $N$ generated tokens that are used to compute the final trajectory-level confidence, which typically corresponds to the last tail step containing the answer. $N_G$ denotes the number of tokens in $G$, and $k$ denotes the number of top-$k$ probabilities from the token logits that participate in computing token-level confidence.

\subsection{Distribution of Confidence}
\label{sec:preliminaries_conf_distribution}
    Following the setting of TTS, we perform multiple repeated sampling for all questions within a single benchmark to form a original distribution. Inspired by previous work on confidence distribution~\cite{Self-Certainty}, we further model it as two normal distributions, where the positive distribution $\mathcal{N}(\mu_{\text{pos}}, \sigma_{\text{pos}}^2)$ with a higher mean $\mu_{\text{pos}}$, and the negative distribution $\mathcal{N}(\mu_{\text{neg}}, \sigma_{\text{neg}}^2)$ with a lower mean $\mu_{\text{neg}}$:
    \begin{equation}
    \begin{aligned}
    X_{\text{pos}} &\sim \mathcal{N}(\mu_{\text{pos}}, \sigma_{\text{pos}}^2), \quad
    X_{\text{neg}} \sim \mathcal{N}(\mu_{\text{neg}}, \sigma_{\text{neg}}^2).
    \end{aligned}
    \end{equation}


\subsection{Distributions Distance and Voting Accuracy}
\label{sec:preliminaries_theorem}
    The primary role of \textit{SelfStepConf} (SSC) in voting is to amplify the separation between the distributions of positive and negative predictions. We formally prove in \autoref{thm:normal_tail_ratio_monotonicity} (proof provided in \autoref{sec:appendix_theorem_proof}) that this enhanced distribution distance directly improves voting accuracy. 
    \begin{theorem}
        \label{thm:normal_tail_ratio_monotonicity}
        Let $f(x) = \frac{1}{\sqrt{2\pi\sigma_1^2}} \exp\left(-\frac{(x-\mu_1)^2}{2\sigma_1^2}\right)$ and $g(x) = \frac{1}{\sqrt{2\pi\sigma_2^2}} \exp\left(-\frac{(x-\mu_2)^2}{2\sigma_2^2}\right)$ be the probability density functions of normal distributions with means $\mu_1, \mu_2$ and variances $\sigma_1^2, \sigma_2^2$ respectively, where $\sigma_1, \sigma_2 > 0$.
        
        Define the integral ratio function:
        \begin{equation}
            R(\mu_1, \mu_2) = \frac{\int_{\frac{\mu_1+\mu_2}{2}}^{\infty} f(x) dx}{\int_{\frac{\mu_1+\mu_2}{2}}^{\infty} g(x) dx},
        \end{equation}
        when $\mu_1 \neq \mu_2$, $R(\mu_1, \mu_2)$ is strictly monotonically increasing with respect to $\delta = \mu_1 - \mu_2$.
    \end{theorem}

    \begin{theorem}
        \label{thm:general_voting_monotonicity}
        Under the conditions of \autoref{thm:normal_tail_ratio_monotonicity}, let $\{w_i\}_{i=1}^n$ be non-negative weights of all $n$ samples, and consider $\mathcal{X}^{0} \sim f(x)$ is the correct answer and $\mathcal{X}^{k} \sim g(x)$ ($k=1,...,m, m\le n-1$) are $m$ incorrect answers, all mutually independent. The weighted voting accuracy:
        \begin{equation}
            P_{\mathrm{vote}}(\delta) = \mathbb{P}\left(\sum_{x_i\in \mathcal{X}^{0}} w_ix_i > \max_{1\leq k\leq m} \sum_{x_j \in \mathcal{X}^{k}}w_jx_j\right),
        \end{equation}
        has a lower bound that increases with $\delta = \mu_1 - \mu_2$.
    \end{theorem}

    Furthermore, \autoref{sec:analysis_ssc_separation_comparison} experimentally validates how \textit{SelfStepConf} achieves superior voting performance through effective distribution separation.

\section{Methodology}
\label{sec:methods}

\subsection{Overview}
\label{sec:methods_overview}

    Follow the sequence of inference and voting,
    we first propose \textit{SelfStepConf} (SSC), a method that dynamically adjusts confidence during reasoning. It monitors step-wise confidence in real-time and triggers self-reflection when confidence declines significantly. 
    After generation, we filter candidate trajectories involved in the final voting via two stages: \textit{GMM Filter} and \textit{Reject Filter}, and use \textit{HierVoting} as the basic voting method for the entire voting process.

\subsection{SelfStepConf}
\label{sec:methods_selfstepconf}
    \paragraph{Reflection Trigger.} Based on the trajectory-level confidence calculation formula in \autoref{eq:traj_conf}, we monitor the confidence of each step in real-time during the inference process. First, for the generated $i$-th token, we calculate the token confidence as:
    \begin{equation}
    C_{\text{token}-i} = -\frac{1}{k}\sum_{j=1}^k \log \mathbf{P}_i(j).
    \end{equation}
    Following each step, we utilize the formula to compute the $m$-th reasoning step confidence, the $G_m$ is defined as:
    \begin{equation}
    G_m = \{ t_i | \text{index}(t_s^{m-1}) < i \leq \text{index}(t_s^m) \}, m \in \mathbb{N},
    \end{equation}
    
    where $t_s^m$ represents the $m$-th step token in the trajectory. 
    Therefore, the current step confidence is expressed as:
    \begin{equation}
    C_{G_m} = -\frac{1}{N_{G_m} \times k}\sum_{i \in G_m}\sum_{j=1}^k \log \mathbf{P}_i(j).
    \end{equation}
    After computing the reasoning step confidence, we evaluate its relative change compared to the dynamically adaptive confidence threshold $\tau_{\text{conf}}$:
    \begin{equation}
    \label{eq:delta_conf}
    \Delta_{\text{conf}} = \frac{C_{G_m}}{\tau_{\text{conf}}}, \quad m \geq 1.
    \end{equation}

    If $\Delta_{\text{conf}} < \delta$, where $\delta$ is a control parameter, the system triggers self-reflection due to poor step quality (without updating $\tau_{\text{conf}}$). Otherwise, $\tau_{\text{conf}}$ is updated via Exponential Moving Average (EMA) using $C_{G_m}$. Furthermore, in real practice, to maintain generation stability, reflection only activates during confidence decline, preventing excessive intervention that could disrupt the reasoning flow. Overall, the self-reflection trigger condition is:
    \begin{equation}
    \label{eq:reflection_trigger}
    \mathbb{I}_{R} = \mathbb{I}[(\Delta_{\text{conf}} < \delta) \wedge (C_{G_m} < C_{G_{m-1}})].
    \end{equation}
    
    The corresponding $\tau_{\text{conf}}$ update method follows:
    \begin{equation}
    \label{eq:conf_threshold_EMA_update}
    \tau_{\text{conf}} := \begin{cases}
    C_{G_m}, & m = 0, \\
    \tau_{\text{conf}}, & m > 0 \wedge \mathbb{I}_{R} = 1, \\
    \alpha \tau_{\text{conf}} + (1-\alpha) C_{G_m}, & m > 0 \wedge \mathbb{I}_{R} = 0.
    \end{cases}
    \end{equation}

    

    \paragraph{Reflection Injection.} To perform reflection, we first define a reflection information token list $T_r = [t_r^1, t_r^2, \ldots, t_r^{N_r}]$. When $C_{G_m}$ significantly decreases, we inject reflection information by forcibly 
    swapping the highest-probability token's probability with that of the corresponding reflection token, and sample with the temperature as 0 for $N_r$ times:
    \begin{equation}
    \label{eq:add_reflection}
    \begin{aligned}
    \mathbf{P}_{r_i} &= [p_1, p_2, \ldots, \tikzmarknode{pk}{p_k^*}, \ldots, \tikzmarknode{pri}{p_{i}}, \ldots, p_{N_v}], \\
    \mathbf{P}_{r_i}' &= [p_1, p_2, \ldots, \tikzmarknode{pri2}{p_{i}}, \ldots, \tikzmarknode{pk2}{p_k^*}, \ldots, p_{N_v}],
    \end{aligned}
    \end{equation}
    \begin{tikzpicture}[overlay, remember picture]
        \draw[->, thick] (pk) to[bend left=15] (pk2);
        \draw[->, thick] (pri) to[bend right=15] (pri2);
    \end{tikzpicture}where $\mathbf{P}_{r_i}$ represents the logits distribution of the $r_i$-th token in the entire generation trajectory, $p_k^*$ represents the highest probability in $\mathbf{P}_{r_i}$, and $p_{i}$ represents the probability corresponding to the $i$-th reflection token $t_r^i$ in $T_{r}$. 

    It is important to note that since the probability swap within $\mathbf{P}_{r_i}$ does not affect its confidence $C_{\text{token}-r_i}$, this reflection injection method does not impact the calculation of step confidence during generation.

\subsection{Distribution Voting}
\label{sec:methods_distribution_voting}
    \paragraph{GMM Modeling.} After inference, for each question's $\text{Budget}$ trajectories $\mathcal{V}$, we compute trajectory-level confidence scores $\mathcal{C}$ via \autoref{eq:traj_conf}, then predict their distribution using GMM.

    Specifically, building on the confidence distributions in \autoref{sec:preliminaries_conf_distribution}, we model the unlabeled trajectory samples using GMM, approximating their bimodal distribution (positive and negative reasoning paths) with two Gaussian components:
    \begin{equation}
    p(x) = \pi_1 \mathcal{N}(x|\mu_1, \sigma_1^2) + \pi_2 \mathcal{N}(x|\mu_2, \sigma_2^2),
    \end{equation}
    where $\pi_1, \pi_2$ are mixing weights satisfying $\pi_1 + \pi_2 = 1$ and $\pi_1, \pi_2 > 0$, and each component $\mathcal{N}(x|\mu_i, \sigma_i^2)$ for $i \in \{1,2\}$ follows a normal distribution:
    \begin{equation}
    \mathcal{N}(x|\mu_i, \sigma_i^2) = \frac{1}{\sqrt{2\pi\sigma_i^2}} \exp\left(-\frac{(x-\mu_i)^2}{2\sigma_i^2}\right).
    \end{equation}
    The mixture distribution form as follow, where $\mu_1, \mu_2$ and $\sigma_1^2, \sigma_2^2$ are the means and variances of the two distributions:
    \begin{equation}
        \begin{aligned}
        p(x) = &\pi_1 \frac{1}{\sqrt{2\pi\sigma_1^2}} \exp\left(-\frac{(x-\mu_1)^2}{2\sigma_1^2}\right) +
        \pi_2 \frac{1}{\sqrt{2\pi\sigma_2^2}} \exp\left(-\frac{(x-\mu_2)^2}{2\sigma_2^2}\right).
        \end{aligned}
    \end{equation}

    

    \paragraph{GMM Filter.} To assign correctness categories, we apply the mean-based mapping $\Phi(\cdot)$ that associates the higher-mean component to the positive distribution and the lower-mean component to the negative distribution:
    \begin{equation}
    \begin{aligned}
    \label{eq:gmm_split}
    X_{\text{pos}} &\sim \mathcal{N}\left(\underset{i}{\arg\max}(\mu_i), \sigma_i^2\right), \quad
     X_{\text{neg}} \sim \mathcal{N}\left(\underset{i}{\arg\min}(\mu_i), \sigma_i^2\right).
    \end{aligned}
    \end{equation}
    
    After modeling the two distributions, we obtain the potentially correct trajectories $\mathcal{V}_{\text{pos}}$ and potentially incorrect trajectories $\mathcal{V}_{\text{neg}}$, along with their confidence sets $\mathcal{C}_{\text{pos}}$ and $\mathcal{C}_{\text{neg}}$. We initially form the candidate voting pool using $\mathcal{V}_{\text{pos}}$:
    \begin{equation}
        \label{eq:get_traj}
        \begin{gathered}
            \mathcal{V}_{\text{pos}} = \{V_\text{traj} \mid V_\text{traj} \in \mathcal{N}(\mu_{\text{pos}}, \sigma_{\text{pos}}^2)\}, \quad
            \mathcal{V}_{\text{neg}} = \{V_\text{traj} \mid V_\text{traj} \in \mathcal{N}(\mu_{\text{neg}}, \sigma_{\text{neg}}^2)\}.
        \end{gathered}
    \end{equation}
    \paragraph{Reject Filter.} Although the GMM initially separates the positive and negative distributions, there is often overlap between them. To address this issue, we aim to use $\mathcal{V}_{\text{neg}}$ to further eliminate false positive samples from the candidate voting pool. Specifically, we first aim to select the most likely incorrect negative answers by using the negative values of $\mathcal{C}_{\text{neg}}$ as weights in the voting process:
    \begin{equation}
    A_{\text{neg}} = f_{\text{vote}}(\mathcal{V}_{\text{neg}}, -\mathcal{C}_{\text{neg}}),
    \end{equation}
    where $f_{\text{vote}}$ can be any voting methods. To prevent low-quality filtering, we further propose \textit{HierVoting} as $f_{\text{HierV}}$ detailed in the next part. We show that all voting methods benefit from the reject-filtering setup, while using \textit{HierVoting} significantly stronger performance as analyzed in  \autoref{sec:appendix_complete_main_results}.

    Afterward, we further filter the candidate voting pool based on this negative answer when $A_{\text{pos}} \neq A_{\text{neg}}$ (note that there is a one-to-many relationship between correct and incorrect answers), to prevent eliminating the true positive answer:
    \begin{equation}
    \label{eq:reject_filer}
    \hat{\mathcal{V}}_{\text{pos}} = \{V_\text{traj} \mid V_\text{traj} \in \mathcal{N}(\hat{\mu}_{\text{pos}}, \hat{\sigma}_{\text{pos}}^2)\},
    \end{equation}
    where $\mathcal{N}(\hat{\mu}_{\text{pos}}, \hat{\sigma}_{\text{pos}}^2)$ is obtained using the filtered trajectories $
    \Phi(\text{GMM}(\mathcal{C}_{A_{\text{traj}} \neq A_{\text{neg}}}))$. After applying the negative answer reject filter, we use the final positive pool $\hat{\mathcal{V}}_{\text{pos}}$ to vote for the final answer:
    \begin{equation}
    A_{\text{final}} = f_{\text{HierV}}(\hat{\mathcal{V}}_{\text{pos}}, \hat{\mathcal{C}}_{\text{pos}}).
    \end{equation}
    \paragraph{Hierarchical Voting.} In the basic voting process, we input a set of trajectories $\mathcal{V}$ and their confidences $\mathcal{C}$. Considering that the ratio of correct to incorrect answers varies across different confidence intervals, we adopt a \textit{hierarchical voting} approach in voting. 
    First, we divide the confidences into $N_{\mathcal{C}}$ sub-intervals:
    \begin{equation}
    \label{eq:divide_conf2interval}
        \mathcal{C}^i = [c_{\min} + (i-1)h, c_{\min} + ih], \quad i \in [1, N_{\mathcal{C}}],
    \end{equation}
    where $c_{\min} = \min(\mathcal{C})$ and $h = \frac{\max(\mathcal{C}) - \min(\mathcal{C})}{N_{\mathcal{C}}}$. Then perform $f_{\text{WMaj}}$ within each interval to select an interval answer:
    \begin{equation}
    \label{eq:voting_interval}
    \mathcal{A}_{\text{sub}} = f_{\text{WMaj}}(\{V_\text{traj}, C_\text{traj} \mid \inf_{\mathcal{C}^i} < C_{\text{traj}} \leq \sup_{\mathcal{C}^i}\}).
    \end{equation}
    Finally, we perform weighted majority voting on multiple interval answers to select the voting answer:
    \begin{equation}
    \label{eq:voting_final}
    f_{\text{HierV}}(\mathcal{V}, \mathcal{C}) = f_{\text{WMaj}}(\{V_\text{sub}^i, \mathcal{\bar{C}}^i_{A_{\text{traj}}=A_{\text{sub}}^i} \mid A_{\text{sub}}^i \in \mathcal{A}_{\text{sub}}\}),
    \end{equation}
    where $f_{\text{WMaj}}(\cdot)$ represents the \textit{Weighted Majority Voting}:
    \begin{equation}
    \label{eq:voting_maj}
    f_{\text{WMaj}}(\mathcal{V}, \mathcal{C}) = \arg\max_{\text{ans}} \sum_{\text{traj} \in \mathcal{V}} \mathbb{I}(A_{\text{traj}} = \text{ans}) \cdot C_{\text{traj}}.
    \end{equation}
    In conclusion, by removing true negative and false positive answers through the two filtering processes, we have enhanced the reliability of the confidence in the final voting.


\section{Experiment}
\label{sec:experiment}

\subsection{Experimental Description}
    In this section, we primarily validate the effectiveness of the proposed voting methods, including \textit{DistriVoting} and \textit{SelfStepConf} (SSC). For \textit{DistriVoting}, we primarily compare its effectiveness with existing test-time scaling voting methods, such as Self-Consistency (SC)~\cite{Self-Consistency}, BoN~\cite{BoN}, MoB~\cite{MoB}, Weighted-SC (WSC)~\cite{WSC}, and DeepConf (WSC-Top50)~\cite{DeepConf}. We further evaluate the superiority of the \textit{GMM Filter} component over the Top50 Filter in this section, while the analysis of \textit{Reject Filter} and \textit{HierVoting} is provided in \autoref{sec:analysis_distribution_voting},  \autoref{sec:appendix_parameters_analysis_Nc} and \autoref{sec:appendix_complete_main_results}. For \textit{SelfStepConf}, we focus on its contribution to voting accuracy, while other detailed analysis presented in \autoref{sec:analysis_ssc_separation_comparison}, \autoref{sec:analysis_ssc_pass@k} and \autoref{sec:analysis_ssc_token}. Additionally, for the Qwen3 model series, we employ \textit{thinking mode} by default, with asterisk-marked \textbf{models*} denoting \textit{non-thinking} variants. All experiments use temperature $t=0.6$ unless otherwise noted, except for non-thinking models which use $t=0.7$. Detailed descriptions provided in \autoref{sec:exp_descrip}, pseudo code provided in \autoref{sec:appendix_pseudo_code}, parameter analysis provided in \autoref{sec:appendix_parameters_analysis}, step split methods and reflection information provided in \autoref{sec:supp_exp}.

\subsection{Main Results}
\label{sec:experiment_main_results}
    \begin{table*}[t]
    \caption{
    Main results of \textit{SelfStepConf} (SSC) and \textit{DistriVoting} across benchmarks. Budget is set to 128 with 64 repetitions. SC denotes Self-Consistency (Majority Voting), WSC represents Weighted Self-Consistency (Weighted MajVoting), BoN indicates Best of N. The * marks answers generated using SSC, \textbf{bold} and \uline{underline} respectively indicate the optimal and suboptimal results, and $\pm$ represents the variation range of multiple repetitions.
    Additional experiments and ablation provided in
    \autoref{tab:appendix_complete_main_results_thinking}, \autoref{tab:appendix_complete_main_results_nonthinking}, \autoref{tab:appendix_complete_main_results_ablation} and 
    \autoref{tab:appendix_complete_main_results_ablation_qwen3_32B} in \autoref{sec:appendix_complete_main_results}.}
    \label{tab:main_results}
    \begin{center}
    \resizebox{0.97\textwidth}{!}{
        \begin{tabular}{llcccccc}
        \toprule
        Model & Method & HMMT2025 & GPQA-D & AIME2024 & AIME2025 & BRUMO2025 & Avg.\\
        \midrule
        \multirow{12}{*}{DeepSeek-R1-8B}
        & SC                & 69.11{\scriptsize$\pm$0.23} & 67.50{\scriptsize$\pm$0.07} & 86.67{\scriptsize$\pm$0.00} & 80.36{\scriptsize$\pm$0.13} & 93.07{\scriptsize$\pm$0.13} & 73.09{\scriptsize$\pm$0.06} \\
        & WSC               & 69.69{\scriptsize$\pm$0.18} & 67.65{\scriptsize$\pm$0.04} & 86.67{\scriptsize$\pm$0.00} & 80.78{\scriptsize$\pm$0.13} & 93.33{\scriptsize$\pm$0.00} & 73.30{\scriptsize$\pm$0.03} \\
        & BoN               & 70.05{\scriptsize$\pm$0.34} & 67.91{\scriptsize$\pm$0.17} & 90.52{\scriptsize$\pm$0.11} & 77.34{\scriptsize$\pm$0.30} & 92.24{\scriptsize$\pm$0.19} & 73.43{\scriptsize$\pm$0.13} \\
        & MoB-Adaptive      & 76.98{\scriptsize$\pm$0.76} & 68.70{\scriptsize$\pm$0.23} & 89.95{\scriptsize$\pm$0.84} & 84.53{\scriptsize$\pm$0.37} & 93.33{\scriptsize$\pm$0.39} & 75.30{\scriptsize$\pm$0.33} \\
        \cmidrule(lr){2-8}
        & WSC-Top50         & 73.80{\scriptsize$\pm$0.44} & 68.58{\scriptsize$\pm$0.04} & 90.00{\scriptsize$\pm$0.18} & 82.55{\scriptsize$\pm$0.34} & 93.33{\scriptsize$\pm$0.00} & 74.75{\scriptsize$\pm$0.07} \\ 
        & \cellcolor{yellow!20}WSC-GMM           & \cellcolor{yellow!20}82.50{\scriptsize$\pm$0.49} & \cellcolor{yellow!20}69.58{\scriptsize$\pm$0.13} & \cellcolor{yellow!20}93.13{\scriptsize$\pm$0.15} & \cellcolor{yellow!20}83.91{\scriptsize$\pm$0.99} & \cellcolor{yellow!20}93.59{\scriptsize$\pm$0.08} & \cellcolor{yellow!20}76.64{\scriptsize$\pm$0.11} \\
        & \cellcolor{yellow!30}WSC-GMM*          & \cellcolor{yellow!30}\uline{84.17{\scriptsize$\pm$0.54}} & \cellcolor{yellow!30}\uline{70.11{\scriptsize$\pm$0.07}} & \cellcolor{yellow!30}\textbf{93.33{\scriptsize$\pm$0.36}} & \cellcolor{yellow!30}84.38{\scriptsize$\pm$0.59} & \cellcolor{yellow!30}\uline{94.17{\scriptsize$\pm$0.44}} & \cellcolor{yellow!30}\uline{77.24{\scriptsize$\pm$0.14}} \\
        \cmidrule(lr){2-8}
        & \cellcolor{orange!10}DIS-Top50         & \cellcolor{orange!10}79.27{\scriptsize$\pm$0.65} & \cellcolor{orange!10}69.52{\scriptsize$\pm$0.19} & \cellcolor{orange!10}93.18{\scriptsize$\pm$0.13} & \cellcolor{orange!10}84.43{\scriptsize$\pm$0.23} & \cellcolor{orange!10}93.28{\scriptsize$\pm$0.18} & \cellcolor{orange!10}76.32{\scriptsize$\pm$0.05} \\
        & \cellcolor{orange!20}DIS-GMM           & \cellcolor{orange!20}82.55{\scriptsize$\pm$0.31} & \cellcolor{orange!20}69.82{\scriptsize$\pm$0.09} & \cellcolor{orange!20}93.13{\scriptsize$\pm$0.18} & \cellcolor{orange!20}\uline{85.52{\scriptsize$\pm$0.60}} & \cellcolor{orange!20}93.70{\scriptsize$\pm$0.10} & \cellcolor{orange!20}76.95{\scriptsize$\pm$0.10} \\
        & \cellcolor{orange!30}DIS-GMM*     & \cellcolor{orange!30}\textbf{84.95{\scriptsize$\pm$0.86}} & \cellcolor{orange!30}\textbf{70.63{\scriptsize$\pm$0.17}} & \cellcolor{orange!30}\uline{93.23{\scriptsize$\pm$0.05}} & \cellcolor{orange!30}\textbf{86.64{\scriptsize$\pm$0.55}} & \cellcolor{orange!30}\textbf{94.27{\scriptsize$\pm$0.17}} & \cellcolor{orange!30}\textbf{77.84{\scriptsize$\pm$0.28}} \\
        \midrule
        \midrule
        \multirow{12}{*}{Qwen3-32B}
        & SC                & 62.08{\scriptsize$\pm$0.31} & 70.30{\scriptsize$\pm$0.20} & 86.46{\scriptsize$\pm$0.33} & 76.98{\scriptsize$\pm$0.26} & 93.33{\scriptsize$\pm$0.00} & 73.85{\scriptsize$\pm$0.21} \\
        & WSC               & 62.24{\scriptsize$\pm$0.39} & 70.55{\scriptsize$\pm$0.14} & 86.88{\scriptsize$\pm$0.67} & 77.08{\scriptsize$\pm$0.10} & 93.33{\scriptsize$\pm$0.00} & 74.07{\scriptsize$\pm$0.16} \\
        & BoN               & 59.48{\scriptsize$\pm$0.46} & 72.00{\scriptsize$\pm$0.14} & 87.29{\scriptsize$\pm$0.28} & 74.01{\scriptsize$\pm$0.35} & 87.08{\scriptsize$\pm$0.32} & 73.87{\scriptsize$\pm$0.31} \\
        & MoB-Adaptive      & 63.85{\scriptsize$\pm$0.31} & 72.08{\scriptsize$\pm$0.53} & 88.44{\scriptsize$\pm$0.38} & 78.13{\scriptsize$\pm$0.40} & 92.66{\scriptsize$\pm$0.13} & 75.36{\scriptsize$\pm$0.04} \\
        \cmidrule(lr){2-8}
        & WSC-Top50         & 63.96{\scriptsize$\pm$0.29} & 72.03{\scriptsize$\pm$0.17} & 88.28{\scriptsize$\pm$0.31} & 76.93{\scriptsize$\pm$0.26} & 92.71{\scriptsize$\pm$0.00} & 75.22{\scriptsize$\pm$0.12} \\
        & \cellcolor{yellow!20}WSC-GMM           & \cellcolor{yellow!20}64.84{\scriptsize$\pm$0.23} & \cellcolor{yellow!20}72.42{\scriptsize$\pm$0.11} & \cellcolor{yellow!20}88.65{\scriptsize$\pm$0.31} & \cellcolor{yellow!20}79.22{\scriptsize$\pm$0.34} & \cellcolor{yellow!20}92.71{\scriptsize$\pm$0.03} & \cellcolor{yellow!20}75.79{\scriptsize$\pm$0.10} \\
        & \cellcolor{yellow!30}WSC-GMM*          & \cellcolor{yellow!30}\uline{65.21{\scriptsize$\pm$0.57}} & \cellcolor{yellow!30}72.43{\scriptsize$\pm$0.25} & \cellcolor{yellow!30}\textbf{90.16{\scriptsize$\pm$0.21}} & \cellcolor{yellow!30}\uline{80.00{\scriptsize$\pm$0.42}} & \cellcolor{yellow!30}\uline{93.28{\scriptsize$\pm$0.18}} & \cellcolor{yellow!30}\uline{76.10{\scriptsize$\pm$0.18}} \\
        \cmidrule(lr){2-8}
        & \cellcolor{orange!10}DIS-Top50         & \cellcolor{orange!10}64.95{\scriptsize$\pm$0.24} & \cellcolor{orange!10}72.33{\scriptsize$\pm$0.33} & \cellcolor{orange!10}88.70{\scriptsize$\pm$0.52} & \cellcolor{orange!10}79.06{\scriptsize$\pm$0.34} & \cellcolor{orange!10}\uline{93.28{\scriptsize$\pm$0.03}} & \cellcolor{orange!10}75.79{\scriptsize$\pm$0.20} \\
        & \cellcolor{orange!20}DIS-GMM           & \cellcolor{orange!20}64.43{\scriptsize$\pm$0.21} & \cellcolor{orange!20}\uline{72.74{\scriptsize$\pm$0.25}} & \cellcolor{orange!20}89.01{\scriptsize$\pm$0.39} & \cellcolor{orange!20}78.70{\scriptsize$\pm$0.23} & \cellcolor{orange!20}93.23{\scriptsize$\pm$0.08} & \cellcolor{orange!20}75.99{\scriptsize$\pm$0.09} \\
        & \cellcolor{orange!30}DIS-GMM*          & \cellcolor{orange!30}\textbf{65.73{\scriptsize$\pm$0.05}} & \cellcolor{orange!30}\textbf{73.18{\scriptsize$\pm$0.02}} & \cellcolor{orange!30}\uline{89.11{\scriptsize$\pm$0.50}} & \cellcolor{orange!30}\textbf{80.05{\scriptsize$\pm$0.44}} & \cellcolor{orange!30}\textbf{93.33{\scriptsize$\pm$0.08}} & \cellcolor{orange!30}\textbf{76.53{\scriptsize$\pm$0.08}} \\
        \bottomrule
        \end{tabular}
    }
    \end{center}
    \vskip -0.1in
\end{table*}
    In our main experiments, we primarily evaluated two models: DeepSeek-R1-8B and Qwen3-32B, testing their voting performance on five mathematical reasoning benchmarks. Beyond comparing with other TTS methods, including SC, BoN, MoB, and WSC, we demonstrate the effectiveness improvements brought by \textit{DistriVoting} 
    and \textit{SelfStepConf}.

    Specifically, as shown in \autoref{tab:main_results}, we first compared the adaptive \textit{GMM Filter} (WSC/DIS-GMM) with the fixed top-threshold filter (WSC/DIS-Top50, with the rationale for Top50 selection analyzed in \autoref{sec:analysis_top50_selection}). The results show that across the two models, GMM improved performance from 74.75\% and 75.22\% to 76.64\% and 75.79\%, respectively, compared to Top50 using WSC, and improved performance from 76.32\% and 75.79\% to 76.95\% and 75.99\% using DIS.
    
    Building upon the filtering mechanism, we then evaluated our proposed \textit{DistriVoting} against the naive weighted voting method WSC. The results demonstrate that \textit{DistriVoting} consistently delivers superior performance compared to WSC across all three models under both Top50 filter and \textit{GMM Filter} settings, showcasing the effectiveness of incorporating distribution-aware voting mechanisms. 
    
    To strengthen confidence distributions discrimination, we investigated the impact of \textit{SSC} by comparing GMM and \textbf{GMM*} variants. The results show that \textit{SSC} provides substantial and consistent performance gains across all models for both WSC and \textit{DistriVoting} approaches, validating the benefit of enhanced confidence differentiation in the generation process (see \autoref{sec:analysis_ssc_separation_comparison} for detailed analysis).

\subsection{Ablation Study}
\label{sec:experiment_ablation}

    \paragraph{GMM.} \textit{DistriVoting} is based on the premise that the confidences of correct and incorrect trajectories follow a bimodal normal distribution, leveraging GMM to partition confidences for predicting trajectory correctness. Essentially, GMM is treated as a distribution-based clustering method, which can be replaced by other clustering approaches to achieve the same objective. Here, we compare clustering methods including K-Means and MeanShift, analyzing their impact on final \textit{DistriVoting} results, prediction accuracy, and computational efficiency. Additionally, we evaluate the quality of confidences within the positive intervals obtained by different clustering methods using metrics: AUROC, $\text{Acc}=\frac{1}{{|\mathcal{V}|}}\sum_{\text{traj} \in \mathcal{V}} \mathbb{I}(A_{\text{traj}}=A_{\text{gt}})$ and $\text{WAcc} = \frac{1}{{|\mathcal{V}|}}\sum_{\text{traj} \in \mathcal{V}} \mathbb{I}(A_{\text{traj}}=A_{\text{gt}})\cdot C_{\text{traj}}$ (Weighted Acc).

    
    
    \begin{table}[ht]
  \caption{
  Ablation of clustering methods using DeepSeek-R1-8B, sampling 128 trajectories/query via \textit{DistriVoting} (64 repeats).
  Complete results provided in \autoref{tab:appendix_complete_gmm_ablation} of \autoref{sec:appendix_complete_gmm_ablation}.}
  \label{tab:ablation_gmm}
  \begin{center}
    \resizebox{0.7\textwidth}{!}{
    \begin{tabular}{lccc>{\columncolor{gray!30}}c}
      \toprule
      Method    & Top50 & K-Means & MeanShift & GMM \\
      \midrule
      Acc (\%)              & 74.47  & 75.29  & \uline{76.55}  & \textbf{77.60} \\
      WAcc (\%)             & 74.61  & 75.32  & \uline{76.57}  & \textbf{77.68} \\
      AUROC ($\uparrow$)    & 0.5117 & \uline{0.5204} & 0.5158 & \textbf{0.5831} \\
      \midrule
      Voting Acc (\%)       & 75.10  & 75.19  & \uline{75.50}  & \textbf{76.95}  \\
      Predict Acc (\%)      & 53.64  & 56.19  & \uline{59.34}  & \textbf{60.46}  \\
      Predict Time (ms/it)  & \textbf{0.0935} & 0.6014 & 1.8492 & \uline{0.3369} \\
      \bottomrule
    \end{tabular}
    }
  \end{center}
\end{table}


    As shown in \autoref{tab:ablation_gmm}, GMM achieves $1.78\times$ the efficiency of K-Means and $5.49\times$ that of MeanShift, while significantly outperforming both in trajectory correctness prediction accuracy, leading to superior voting performance. The confidences quality evaluation metrics further highlight GMM's outstanding performance, demonstrating its suitability for prediction tasks involving bimodal normal distributions. Despite this, clustering methods like K-Means and MeanShift still offer clear advantages over a fixed Top50 filter.

    \paragraph{Budget.}

    \vspace{-6pt}
    Budget is a crucial parameter for voting methods. We evaluate six budget settings (8, 16, 32, 64, 128, 256) for our proposed SSC and \textit{DistriVoting} approaches. As shown in \autoref{tab:ablation_budget}, SSC consistently outperforms BasicInference across all budgets, confirming that SSC improves voting by enhancing distribution separation. Additionally, the adaptive \textit{GMM Filter} also consistently surpasses fixed Top50 filtering in both WSC and \textit{DistriVoting} settings.


    \begin{table}[ht]
    \caption{
    Ablation results for varying $\text{Budget}$ using DeepSeek-R1-8B, sampling $\text{B}$ trajectories/question (64 repeats). \textbf{* denotes \textit{SSC}-generated.}
    Complete results provided in \autoref{tab:appendix_complete_budget_ablation} of ~\autoref{sec:appendix_complete_budget_ablation}.}
    \label{tab:ablation_budget}
    \begin{center}
    \resizebox{0.7\textwidth}{!}{
    \begin{tabular}{lcccccc}
      \toprule
      Method    & 8 & 16 & 32 & 64 & 128 & 256  \\
      \midrule
      WSC-Top50 & 73.17 & 73.86 & 74.56 & 74.62 & 74.75 & 74.79\\
      \rowcolor{yellow!20} WSC-GMM   & \uline{73.18} & 74.53 & 75.83 & 76.30 & 76.64 & 77.04\\
      \rowcolor{yellow!30} WSC-GMM*  & \textbf{73.22} & 74.68 & \uline{75.84} & \uline{76.78} & \uline{77.24} & \uline{77.56}\\
      \midrule
      \rowcolor{orange!10} DIS-Top50 & 72.95 & 73.72 & 74.64 & 75.71 & 76.32 & 76.75\\
      \rowcolor{orange!20} DIS-GMM   & 73.12 & \uline{74.73} & 75.71 & 76.34 & 76.95 & 77.53\\
      \rowcolor{orange!30} DIS-GMM*  & \uline{73.18} & \textbf{74.74} & \textbf{75.90} & \textbf{76.87} & \textbf{77.84} & \textbf{78.18}\\
      \bottomrule
    \end{tabular}
    }
    \end{center}
\end{table}

    Furthermore, \textit{DistriVoting} shows significant advantages over conventional methods when Budget $\geq$ 16, while maintaining comparable performance at smaller budgets. This stems from its reliance on distributional information: small sample sizes yield noisy distributions, while larger samples provide more reliable discriminative information, allowing \textit{DistriVoting} to fully utilize its distributional advantages.

\section{Analysis}
\label{sec:analysis}

\subsection{The Effectiveness of DistriVoting}
\label{sec:analysis_distribution_voting}
    To further analyze the impact of \textit{GMM Filter} and \textit{Reject Filter} on voting effectiveness, we calculated Acc and WAcc defined in \autoref{sec:experiment_ablation} across three stages: \textbf{(\uppercase\expandafter{\romannumeral1})} Before filter (all samples); \textbf{(\uppercase\expandafter{\romannumeral2})} After \textit{GMM Filter} (candidate positive samples); \textbf{(\uppercase\expandafter{\romannumeral3})} After \textit{Reject Filter} (final positive samples).

    \begin{table}[ht]
  \caption{
  Effectiveness Analysis of \textit{DistriVoting} using DeepSeek-R1-8B,
  sampling 128 trajectories/question (64 repeats).
  Complete results are provided in \autoref{tab:appendix_complete_filter_analysis} of \autoref{sec:appendix_complete_filter_analysis}.}
  \label{tab:analysis_distribution_voting}
  \begin{center}
    \resizebox{0.6\textwidth}{!}{
    \begin{tabular}{llc>{\columncolor{gray!15}}c>{\columncolor{gray!30}}c}
      \toprule
      Metric & Benchmark & Stage \uppercase\expandafter{\romannumeral1}  & Stage \uppercase\expandafter{\romannumeral2} & Stage \uppercase\expandafter{\romannumeral3} \\
      \midrule
      \multirow{7}{*}{Acc}
      & HMMT2025  & 60.36 & \uline{76.71}  & \textbf{77.43} \\
      & GPQA-D    & 64.54 & \uline{71.69}  & \textbf{75.86} \\
      & AIME2024  & 86.83 & \uline{93.90}  & \textbf{94.08} \\
      & AIME2025  & 79.92 & \uline{87.79}  & \textbf{88.75} \\
      & BRUMO2025 & 81.22 & \uline{91.06}  & \textbf{91.34} \\
      \cmidrule{2-5}
      & Avg.      & 69.27 & \uline{77.60}  & \textbf{80.41} \\
      \midrule
      \multirow{7}{*}{WAcc}
      & HMMT2025  & 61.59 & \uline{76.92} & \textbf{77.63} \\
      & GPQA-D    & 64.73 & \uline{71.72} & \textbf{75.89} \\
      & AIME2024  & 87.40 & \uline{94.00} & \textbf{94.17} \\
      & AIME2025  & 80.70 & \uline{87.86} & \textbf{88.83} \\
      & BRUMO2025 & 82.31 & \uline{91.24} & \textbf{91.52} \\
      \cmidrule{2-5}
      & Avg.      & 69.74 & \uline{77.68} & \textbf{80.48} \\
      \bottomrule
    \end{tabular}
    }
  \end{center}
  \vskip -0.1in
\end{table}
    As shown in \autoref{tab:analysis_distribution_voting}, both Acc and WAcc increase progressively across the three stages, indicating that \textit{\textbf{\textit{GMM Filter} (Stage \uppercase\expandafter{\romannumeral1} $\to$ \uppercase\expandafter{\romannumeral2}) and \textit{Reject Filter} (Stage \uppercase\expandafter{\romannumeral2} $\to$ \uppercase\expandafter{\romannumeral3}) respectively improve the correct sample ratio in their trajectory pools, thereby boosting the final voting accuracy.}}

\subsection{SSC's Role in Enhancing Voting Effectiveness}
\label{sec:analysis_ssc_separation_comparison}

    From the experimental results, we observe that \textit{SSC} significantly enhances voting performance. To analyze the source of this performance gain, we computed the confidence distributions for \textit{SelfStepConf} and \textit{BasicInference}. As shown in \autoref{fig:analysis_ssc_separation_comparison} (Up), \textbf{\textit{the SSC distribution exhibits greater separation than the BasicInference, with less overlap, indicating that SSC amplifies the distance between distributions.}}

    \begin{figure}[ht]
        \begin{center}
        \centerline{\includegraphics[width=0.8\columnwidth]{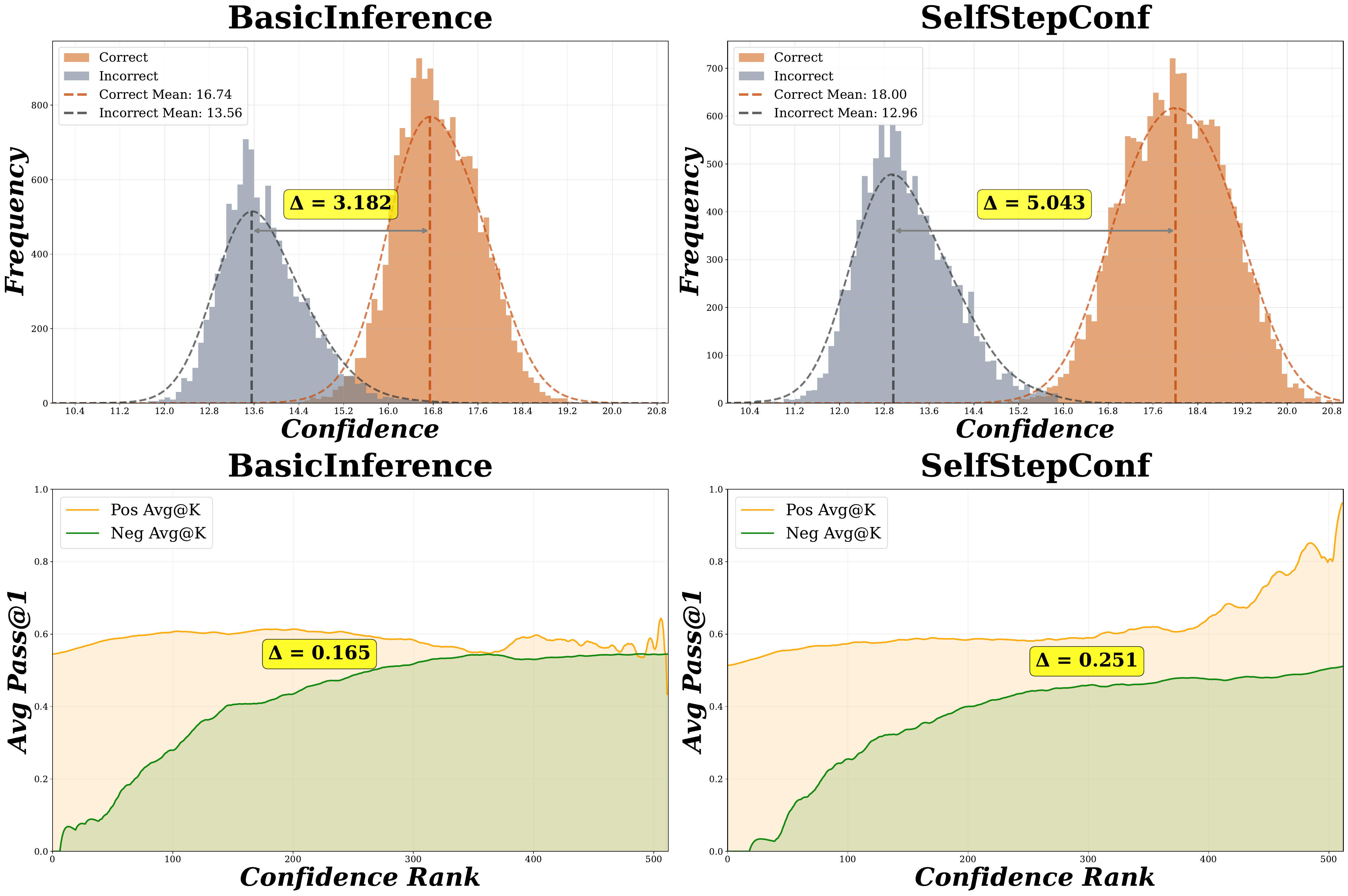}}
        \caption{
          Confidence distribution (HMMT2025) via DeepSeek-R1-8B, sampling 512 trajectories/query. \textbf{Up:} Confidence frequency histogram; \textbf{Down:} Accuracy curves for positive/negative intervals.
          Complete results are provided in \autoref{tab:appendix_complete_ssc_distribution} of \autoref{sec:appendix_complete_ssc_distribution}.
        }
        \label{fig:analysis_ssc_separation_comparison}
        \end{center}
        \vskip -0.2in
    \end{figure}

    Furthermore, in \autoref{sec:preliminaries_theorem}, we theoretically prove that increasing the difference between $\mu_{\text{pos}}$ and $\mu_{\text{neg}}$ leads to a higher proportion of correct samples among trajectories after \textit{GMM Filter}, and this proportional change results in improved final answer accuracy during the voting process.
    

    As shown in \autoref{fig:analysis_ssc_separation_comparison} (Down), to illustrate this effect intuitively, we analyzed confidence-sorted trajectories by calculating average correctness rates in both the positive interval (right side, predicted incorrect by \textit{GMM Filter}, shown in orange) and negative interval (left side, predicted correct by \textit{GMM Filter}, shown in green). Higher correctness density indicates more reliable voting information. The difference in densities reflects the correctness advantage of the positive interval, with larger differences indicating more effective \textit{Reject Filter} performance. Specifically, SSC achieves a higher correctness density in the positive interval at high confidence ranks compared to the BasicInference, increasing the density difference from 0.165 to 0.251, thus enhancing voting effectiveness through better distribution separation.

\subsection{Adaptive Threshold Selection vs. Top-50 Filtering}
\label{sec:analysis_top50_selection}
    In the main results section, we compare two filter voting methods: GMM and Top-50\% (i.e, DeepConf). To analyze the rationality of this setting, we combined the description of top-threshold in DeepConf and tested the optimal partition thresholds under different models and benchmarks.

    \begin{figure}[ht]
        \begin{center}
        \centerline{\includegraphics[width=0.8\textwidth]{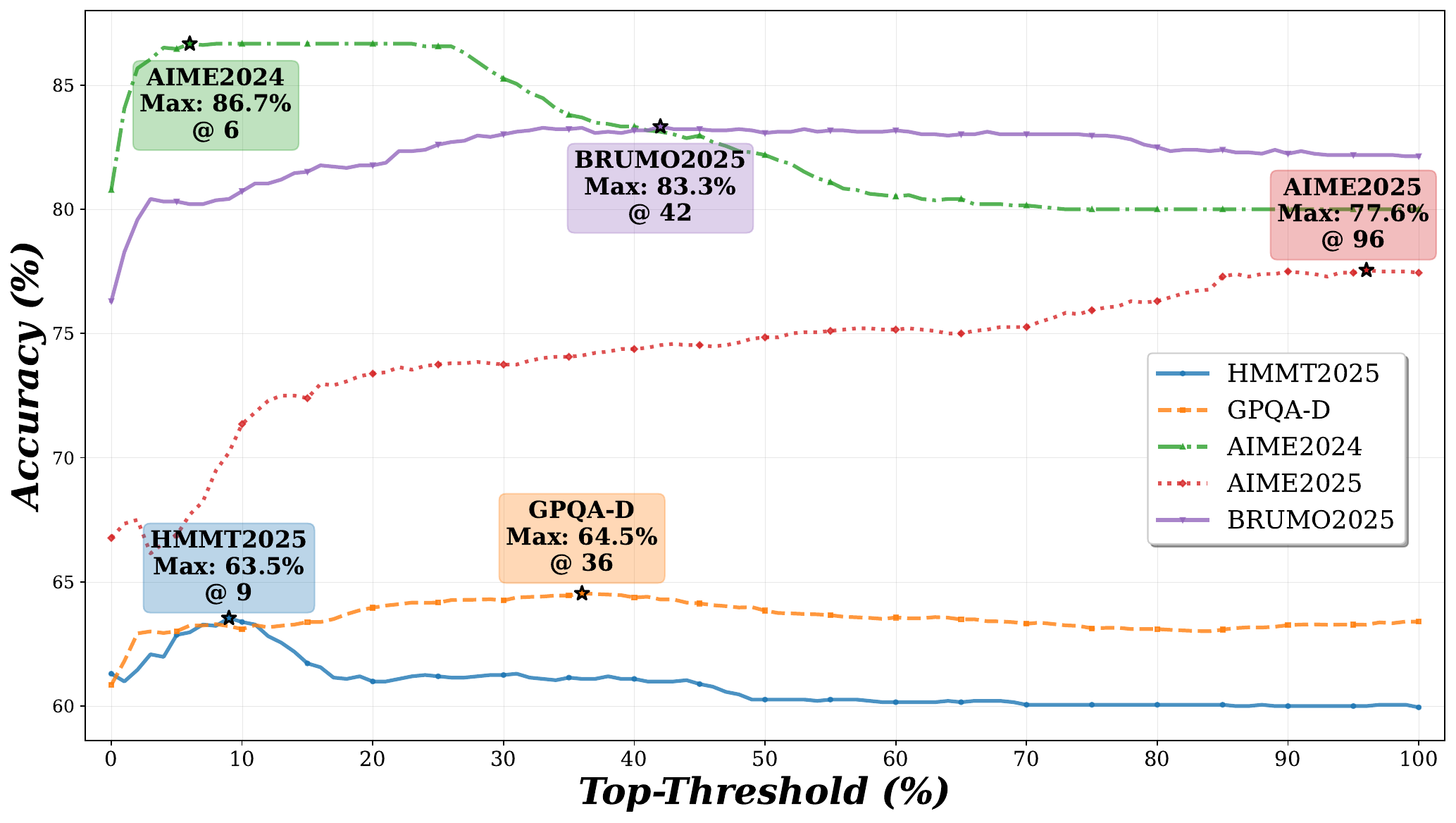}}
        \caption{
          Optimal top-threshold traversal results using Qwen3-8B, sampling 256 responses/question (64 repeats).
          Complete results are provided in \autoref{fig:appendix_complete_top50_analysis} of \autoref{sec:appendix_complete_top50_analysis}.
        }
        \label{fig:analysis_top50_selection}
        \end{center}
    \end{figure}


    As shown in \autoref{fig:analysis_top50_selection}, we evaluated WSC-TopK with top-thresholds at 1\% intervals, applying each threshold uniformly to all questions in the corresponding benchmark. The results showed optimal thresholds varied significantly across benchmarks (6\%, 9\%, 36\%, 42\%, and 96\%) and models (see \autoref{sec:appendix_parameters_analysis_Nc} and \autoref{sec:appendix_complete_top50_analysis}). Therefore, we uniformly use 50\% threshold as the representative baseline for fixed-filter voting methods. In contrast, our proposed \textit{DistriVoting} can \textbf{\textit{adaptively select optimal trajectories for voting at both benchmark and individual question levels}}. Moreover, as demonstrated in our supplementary results, this question-level adaptive approach can outperform even the best-performing benchmark-level fixed-threshold voting methods (see  \autoref{fig:appendix_parameters_analysis_Nc} in \autoref{sec:appendix_parameters_analysis_Nc}).

\subsection{Interpretation of Gaussian Components in GMM}
\label{sec:analysis_gmm_components}

    \textit{DistriVoting} uses GMM to decompose confidences, selecting trajectories with higher correct probabilities for final voting. This builds on our confidence distribution analysis, where GMM's Gaussian components correspond to correct and incorrect distributions. Furthermore, we explored the relationship between confidences and answers, treating each answer as a Gaussian distribution.


    \begin{figure}[ht]
        \begin{center}
        \centerline{\includegraphics[width=0.8\columnwidth]{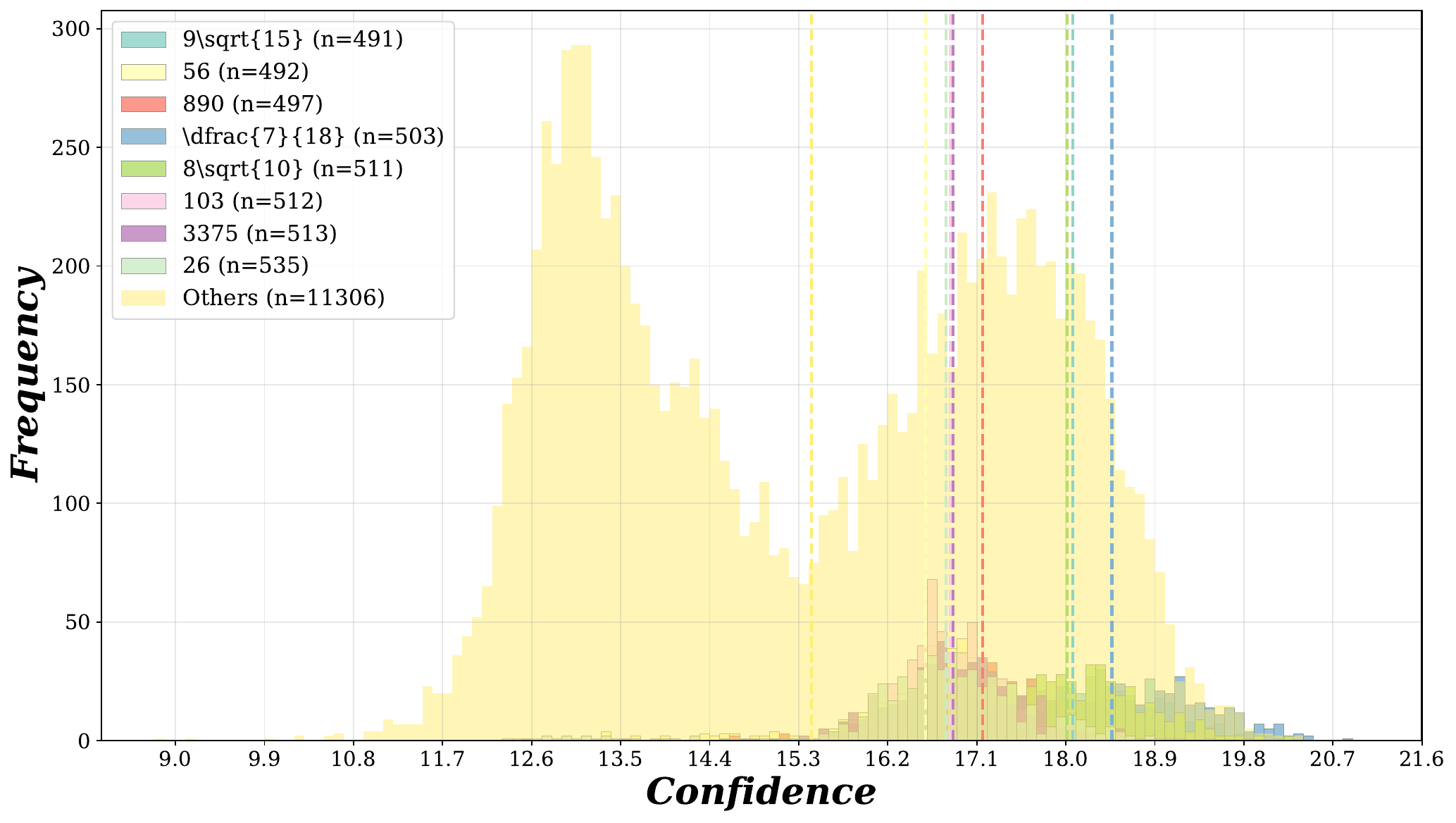}}
        \caption{
          Visualizing answer distribution as Gaussian components in GMM using DeepSeek-R1-8B, sampling 512 responses/query on HMMT2025.
          Complete results provided in \autoref{fig:appendix_complete_gmm_components_analysis} of \autoref{sec:appendix_complete_gmm_components_analysis}.
        }
        \label{fig:analysis_gmm_components}
        \end{center}
        \vskip -0.3in
    \end{figure}

    Specifically, instead of labeling confidence with correctness, we use the trajectory's answer. As shown in \autoref{fig:analysis_gmm_components}, \textbf{\textit{the top 8 frequent answers in HMMT2025 follow normal distributions but overlap significantly.}} For example, the means of the top 1 and 2 answers differ by only 0.022, making it hard to distinguish answers using confidence alone. Since correctness is more critical, using it directly as the label for clustering avoids mapping confidence to answers and then to correctness, reducing information loss and improving the reliability of final voting.


\subsection{SSC's Impact on Model Sampling Behavior}
\label{sec:analysis_ssc_pass@k}
    This section examines SSC's effect on trajectory-level confidence changes to understand its influence on sampling distribution. First, we calculate the trajectory-level confidence of results generated by \textit{SelfStepConf} (SSC) and BasicInference. As shown in \autoref{fig:analysis_ssc_pass@k_confidence}, SSC significantly improves confidence across all benchmarks, generating trajectories with higher correctness probability. And combined with the analysis in \autoref{sec:analysis_ssc_separation_comparison}, it can be known that this improvement mainly comes from the correct samples.


    \begin{figure}[ht]
        \begin{center}
        \centerline{\includegraphics[width=0.8\columnwidth]{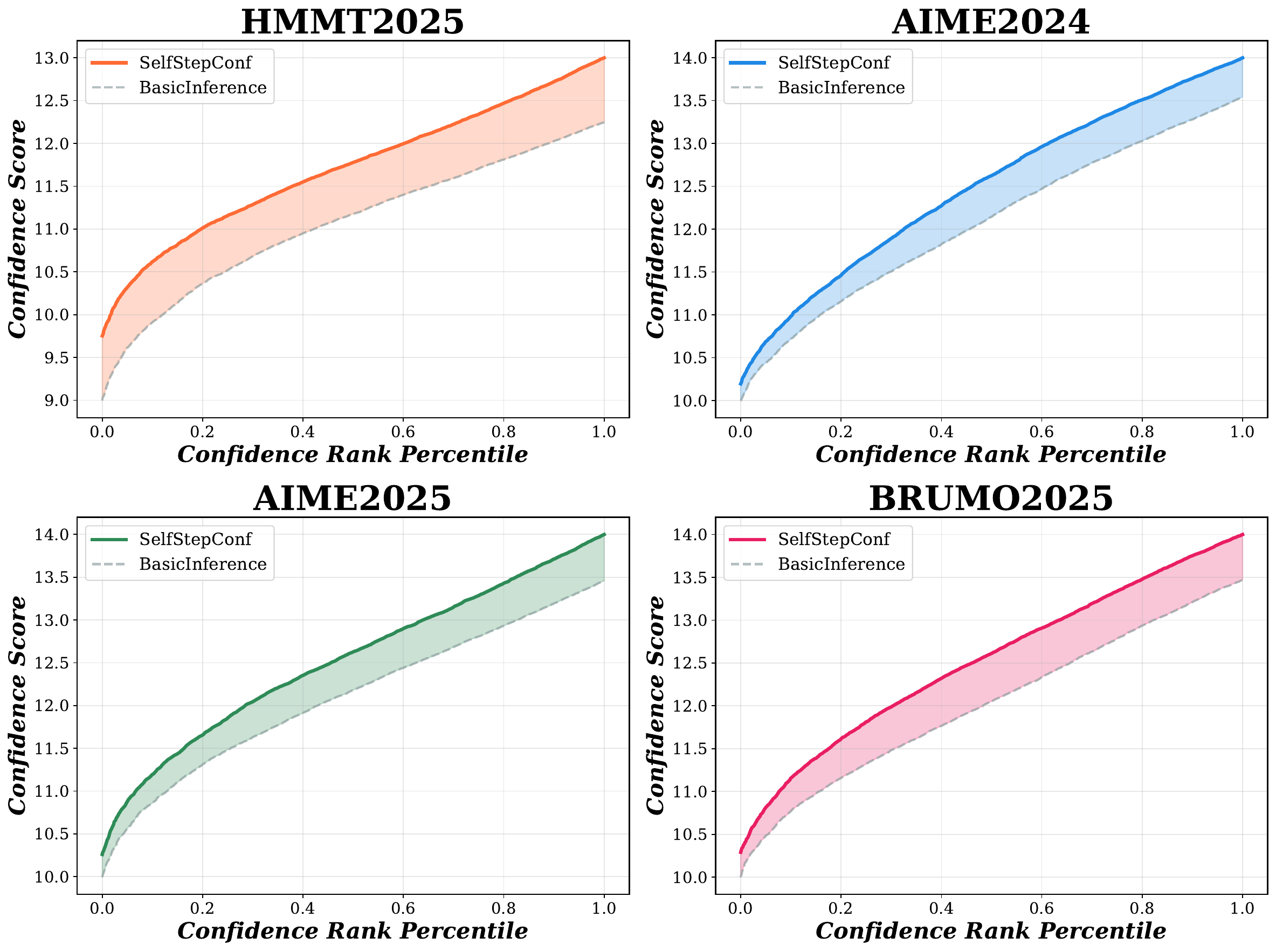}}
        \caption{
            Comparing trajectory-level confidence between SSC and BasicInference using Qwen3-14B-NonThinking, sampling 512 responses/query.
            Complete results provided in \autoref{fig:appendix_complete_trajlevel_conf_analysis} of \autoref{sec:appendix_complete_trajlevel_conf_analysis}.
        }
        \label{fig:analysis_ssc_pass@k_confidence}
        \end{center}
        \vskip -0.2in
    \end{figure}


    \begin{figure}[ht]
        \begin{center}
        \centerline{\includegraphics[width=0.8\columnwidth]{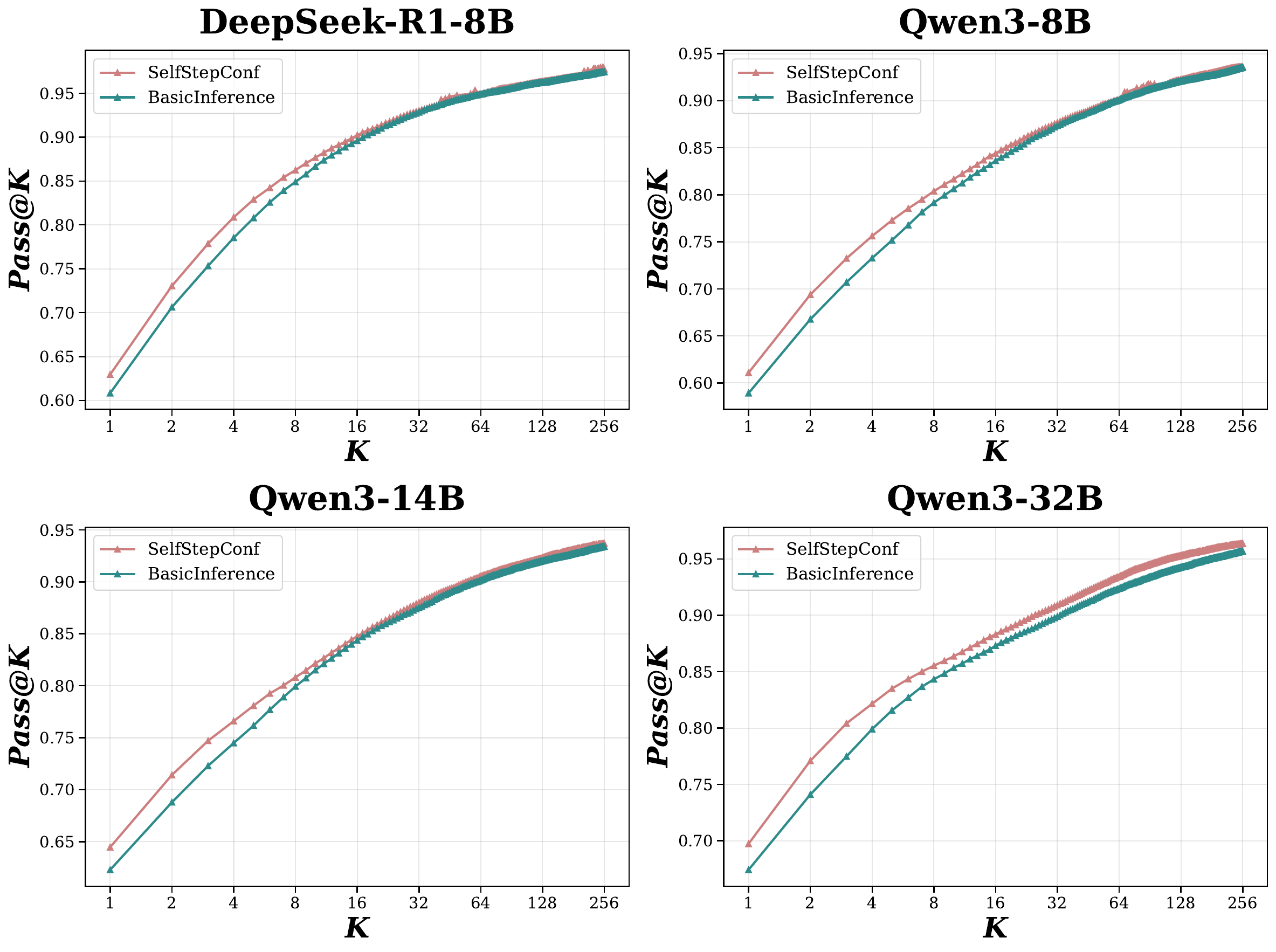}}
        \caption{
            Pass@K comparison between SSC and BasicInference using different models, sampling 256 responses/query on GPQA-D (64 repeats).
            Complete results provided in \autoref{fig:appendix_complete_pass@k} of \autoref{sec:appendix_complete_pass@k}.
        }
        \label{fig:analysis_ssc_pass@k_pass@k}
        \end{center}
        \vskip -0.2in
    \end{figure}

    Furthermore, \autoref{fig:analysis_ssc_pass@k_pass@k} illustrates the trend of pass@K performance as the number of samples increases. At K=1, SSC achieves notably higher pass@1 values than the BasicInference. However, as K increases, the pass@K performance of SSC and the BasicInference converges. This pattern aligns with the conclusion from ~\cite{Nips6666-Does} that \textit{RL improves sampling efficiency but not reasoning limits.} Similarly, SSC, as a depth-oriented test-time scaling approach, \textbf{\textit{improves sampling efficiency without expanding fundamental reasoning limits.}} Regarding SSC's improvement in sampling efficiency, \autoref{fig:analysis_ssc_pass@k_pass@1} shows that the enhancement primarily appears in moderately performing base models, exhibiting an arched growth pattern. For underperforming models with weaker reasoning capabilities, reflection injection is less effective. For high-performing models, the improvement is relatively weaker due to diminishing returns.

    \begin{figure}[ht]
        \vskip -0.1in
        \begin{center}
        \centerline{\includegraphics[width=0.75\columnwidth]{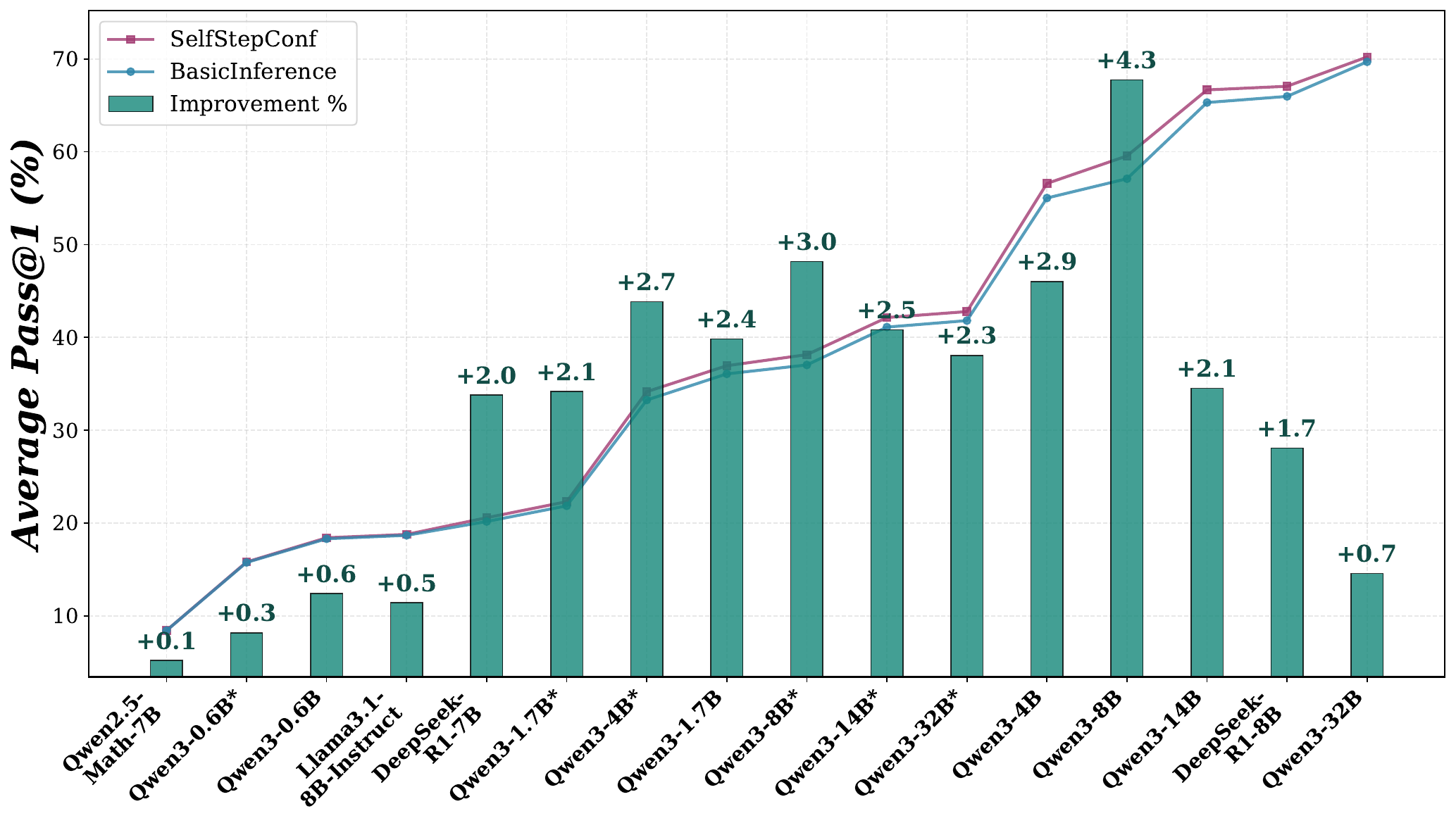}}
        \vskip -0.1in
        \caption{
            Pass@1 comparison between SSC and BasicInference using different models, sampling 1 response/query (64 repeats).
            Complete results are provided in \autoref{tab:appendix_complete_pass@1} of \autoref{sec:appendix_complete_pass@1}.
        }
        \label{fig:analysis_ssc_pass@k_pass@1}
        \end{center}
        \vskip -0.2in
    \end{figure}

\subsection{SelfStepConf Effects on Inference Dynamics}
\label{sec:analysis_ssc_token}


    During model inference process, SSC monitors step confidence in real-time and intervenes when reflection is triggered, altering the trajectory compared to the BasicInference, particularly in confidence and response length. To assess this impact, we set temperature=0 to ensure SSC's trajectory aligns with the BasicInference before intervention occurs. We analyzed the similarities and differences between SSC and BasicInference trajectories, recording confidence changes and reflection trigger points, and comparing these with the BasicInference's inference process.

    \begin{figure}[ht]
        \begin{center}
        \centerline{\includegraphics[width=0.75\columnwidth]{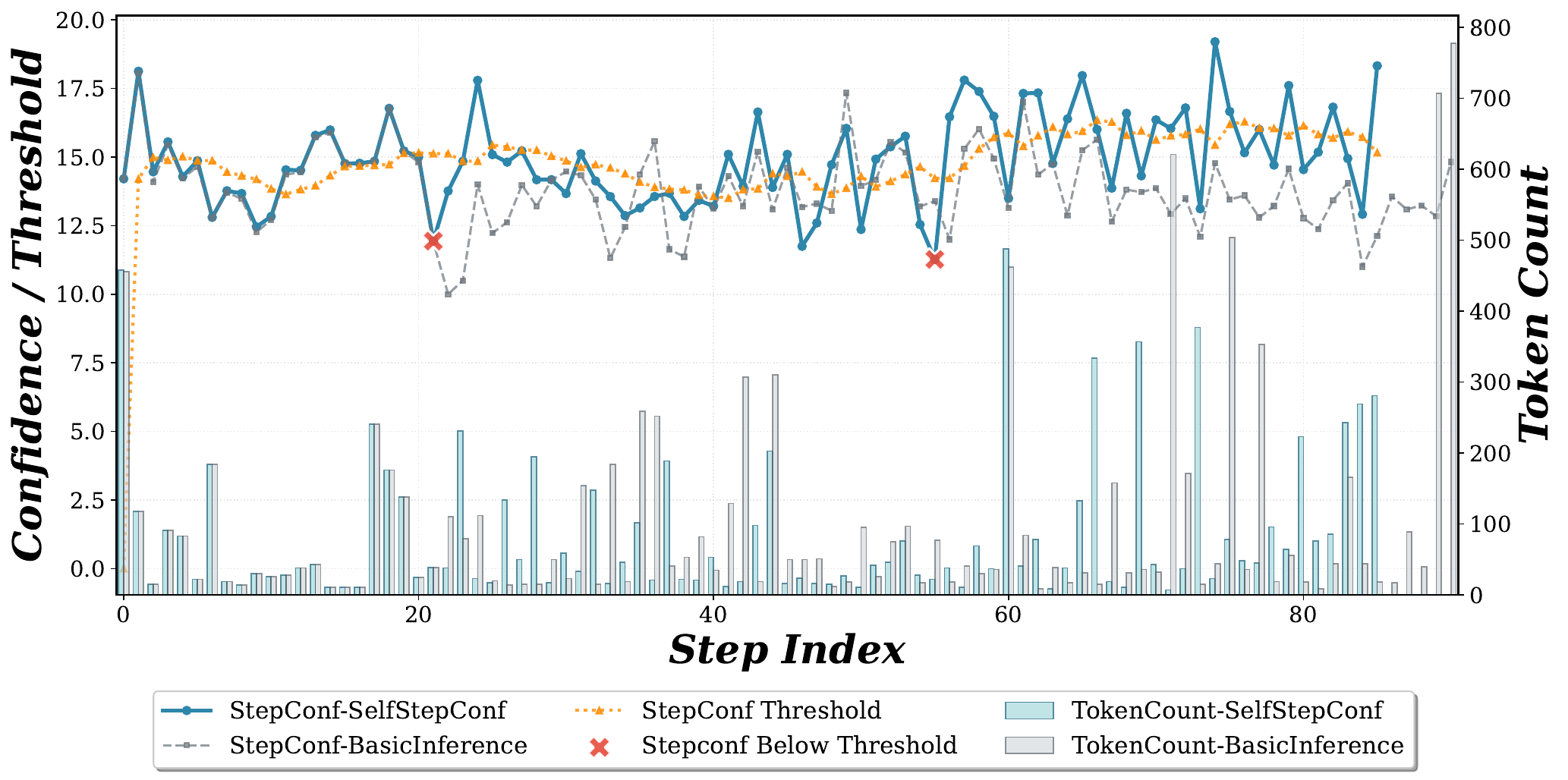}}
        \caption{
          Confidence and token count trends of the SSC and BasicInference. The trajectory is generated by DeepSeek-R1-8B on the 188th question of GPQA-D \textbf{at temperature=0}.
        }
        \label{fig:analysis_ssc_token}
        \end{center}
        \vskip -0.3in
    \end{figure}


    As shown in \autoref{fig:analysis_ssc_token}, two reflection triggers occurred (step 21 and 55). Before the first reflection, both models' trajectories were consistent. However, after reflection, significant differences in confidence and token count emerged. \textbf{\textit{Reflection enhanced subsequent confidence, while BasicInference's confidence declined.}} After the second reflection, SSC maintained high confidence, leading to a correct answer, while the BasicInference resulted in an incorrect one.


    Regarding token count, SSC matched the BasicInference before reflection, with differences appearing upon reflection. In \autoref{fig:analysis_ssc_token}, the BasicInference produced 91 steps with 9,472 tokens (104.1 tokens/step), while SSC generated 86 steps with 7,750 tokens (87.8 tokens/step). Furthermore, as shown in \autoref{tab:analysis_ssc_token}, SSC generally showed fewer steps and tokens than the BasicInference across benchmarks, indicating that reflection improves trajectory-level confidence without increasing response length. Notably, while SSC doesn't increase response length, it performs additional checks per token, slightly raising time complexity from $O(n)$ to $O(n+k)$, where $k$ depends on step splitting (\autoref{sec:supp_exp_step_split_ablation}). Using ``\textbackslash n\textbackslash n'' as the delimiter, SSC's runtime increased by only 2.31\% compared to the BasicInference.


    \begin{table}[ht]
  \vskip -0.1in
  \caption{
  Response length changes between SSC and BasicInference using DeepSeek-R1-8B \textbf{at temperature=0}, sampling 1 response/query.
  Complete results provided in \autoref{tab:appendix_complete_token_analysis} of \autoref{sec:appendix_complete_token_analysis}.}
  \label{tab:analysis_ssc_token}
  \begin{center}
    \resizebox{0.8\textwidth}{!}{
    \begin{tabular}{lccccccc}
      \toprule
      \multirow{2}{*}{Benchmark} & \multicolumn{3}{c}{\textit{BasicInference}} & \multicolumn{3}{c}{\textit{SelfStepConf}} \\
      \cmidrule(lr){2-4} \cmidrule(lr){5-7}
                & Step & Token & Confs. & Step & Token & Confs. \\
      \midrule
      HMMT2025  & 154.00 & 28266.73 & 17.03 & 154.40 & 28604.80 & 17.20 \\
      GPQA-D    & 30.80  & 9560.65  & 13.31 & 29.27  & 9411.71  & 13.33 \\
      AIME2024  & 88.63  & 21239.20 & 17.96 &  83.33 & 20733.97 & 17.92 \\
      AIME2025  & 123.23 & 26673.87 & 17.87 & 128.10 & 26280.73 & 17.91 \\
      BRUMO2025 & 135.40 & 23137.50 & 17.44 & 124.50 & 22291.60 & 17.45 \\
      \midrule
      Avg.      & 66.47  & 15322.42 & 14.92 & 64.48 & 15097.02  & 14.95 \\
      \cmidrule(lr){2-4} \cmidrule(lr){5-7}
      Time (ms/it)      & \multicolumn{3}{c}{207.70} & \multicolumn{3}{c}{212.51} \\
      \bottomrule
    \end{tabular}
    }
  \end{center}
\end{table}


\section{Conclusion}
    This paper addresses the issue of confidently wrong predictions when using model-inherent confidence for voting in test-time scaling. Specifically, we propose \textit{DistriVoting}, which leverages the prior information of confidence distribution to optimize the voting process. It uses \textit{GMM Filter} to remove true negative samples and \textit{Reject Filter} to eliminate false positive samples from the original sampling distribution. Additionally, \textit{HierVoting} is employed to compensate for performance deficiencies when filter quality is low. Besides, \textit{SelfStepConf} (SSC) from the  of distribution, dynamically adjusts the model's inference process, increasing the distance between distributions to further enhance the reliability of distribution information. Detailed experiments and analyses are conducted to qualitatively and quantitatively demonstrate the effectiveness of \textit{DistriVoting} and SSC.
    

\section*{Impact Statement}
    This is propose a confidence-based test-time scaling method that enhances voting accuracy using solely model-internal information, whose goal is to advance the field of Machine Learning. There are many potential societal consequences of our work, none which we feel must be specifically highlighted here.

\bibliography{stepconf}
\bibliographystyle{unsrt}

\appendix
\newpage
\appendix
\onecolumn

\section{Related Works}
\label{sec:related_works}

\paragraph{Test-Time Scaling (TTS).} Current TTS research primarily focuses on increasing the inference depth of single sampling or improving information utilization among multiple sampling results to enhance the accuracy of final answers~\cite{ReasoningWithSampling,ArtOfScaling,ValueGuidedSearch}. The s1~\cite{s1} study investigated the scaling relationship between response length and accuracy in single sampling results, using \textit{BudgetForcing} to specify upper and lower bounds for thinking tokens, or forcibly replacing end tokens with wait tokens to extend inference length. Self-Consistency~\cite{Self-Consistency} leverages consistency information among multiple repeated sampling results to select the final consensus answer. Building upon this foundation, works such as MoB~\cite{MoB} and DORA~\cite{DORA} further utilize Reward Models to provide quality assessments for each sampling result, thereby refining the weights among different results and optimizing the voting methods through this weight information. In comparison, while our work similarly leverages weight information among different answers to improve final answer accuracy, we employ model-intrinsic information for \textbf{self-weighted}. Furthermore, we not only utilize this verification information to optimize voting schemes but also leverage it to dynamically adjust the model's generation process, achieving \textbf{self-enhancement} through self-weighted.

\paragraph{Intrinsic Information of LLMs.} The widespread application of RL has enhanced the reasoning capabilities of LLMs, leading to the emergence of Large Reasoning Models (LRMs)~\cite{DeepSeek-v3}. An increasing number of studies suggest that intrinsic information generated during the generation process of LRMs can, to some extent, reflect the quality of model responses~\cite{RevisitingUncertainty}. This information can be utilized not only in the training process to further extend the application scenarios of RLVR ~\cite{EndoRM, TTRL}, but also in TTS to evaluate the quality of multiple answers and leverage this quality assessment to achieve serial, parallel, and tree-based answer search optimization. Average Log-Probability and Perplexity evaluate model response confidence by utilizing sentence-level probability information~\cite{PerplexityConf,PerplexityConf-Adherence,PerplexityConf-Prompt, VerbalizedConf-MultiAgent, VerbalizedConf-MultiStep}. Entropy and Self-Certainty~\cite{Self-Certainty} further employ distribution-level confidence to assess relative confidence among multiple answers. Although existing work has identified that different confidence calculation methods correspond to distinct quality distribution characteristics, they have not further applied the discriminative information between these distributions to voting mechanisms. We leverage the confidence distribution differences between correct and incorrect answers, employing adaptive filtering and reject sampling to improve the reliability of answers participating in final voting, thereby enhancing voting results.

\newpage
\section{Theorem Proof}
\label{sec:appendix_theorem_proof}

\subsection{Proof of \autoref{thm:normal_tail_ratio_monotonicity}}
\begin{proof}
    Without loss of generality, assume $\mu_1 > \mu_2$, i.e., $\delta = \mu_1 - \mu_2 \geq 0$.
    
    Let $\Phi(z) = \int_{-\infty}^{z} \frac{1}{\sqrt{2\pi}} e^{-t^2/2} dt$ be the cumulative distribution function of the standard normal distribution. For a normal distribution $N(\mu, \sigma^2)$, we have:
    \begin{equation}
    \int_a^{\infty} \frac{1}{\sqrt{2\pi\sigma^2}} \exp\left(-\frac{(x-\mu)^2}{2\sigma^2}\right) dx = 1 - \Phi\left(\frac{a-\mu}{\sigma}\right).
    \end{equation}
    
    Therefore:
    \begin{equation}
        \begin{aligned}
        I_1(\delta) &= \int_{\frac{\mu_1+\mu_2}{2}}^{\infty} f(x) dx = 1 - \Phi\left(\frac{\frac{\mu_1+\mu_2}{2} - \mu_1}{\sigma_1}\right)
        = 1 - \Phi\left(\frac{\mu_2-\mu_1}{2\sigma_1}\right) \\
        &= 1 - \Phi\left(-\frac{\delta}{2\sigma_1}\right) 
        =  1 - \left(1 - \Phi\left(\frac{\delta}{2\sigma_1}\right)\right) = \Phi\left(\frac{\delta}{2\sigma_1}\right),
        \end{aligned}
    \end{equation}
    
    \begin{equation}
        \begin{aligned}
        I_2(\delta) &= \int_{\frac{\mu_1+\mu_2}{2}}^{\infty} g(x) dx = 1 - \Phi\left(\frac{\frac{\mu_1+\mu_2}{2} - \mu_2}{\sigma_2}\right)= 1 - \Phi\left(\frac{\mu_1-\mu_2}{2\sigma_1}\right) \\
        &= 1 - \Phi\left(\frac{\delta}{2\sigma_2}\right).
        \end{aligned}
    \end{equation}

    The ratio function can be expressed as:
    \begin{equation}
        R(\delta) = \frac{\Phi\left(\frac{\delta}{2\sigma_1}\right)}{1 - \Phi\left(\frac{\delta}{2\sigma_2}\right)}.
    \end{equation}
    
    Let $u = \frac{\delta}{2\sigma_1}$, $v = \frac{\delta}{2\sigma_2}$, and $\phi(z) = \frac{1}{\sqrt{2\pi}}e^{-z^2/2}$ be the probability density function of the standard normal distribution. Then:
    
    \begin{align}
    \frac{dR}{d\delta} &= \frac{1}{(1-\Phi(v))^2} \left[\frac{\phi(u)(1-\Phi(v))}{2\sigma_1} + \frac{\Phi(u)\phi(v)}{2\sigma_2}\right].
    \end{align}
    
    Since $\phi(z) > 0$ for all real $z$, $1-\Phi(v) > 0$ when $v$ is finite, and $\Phi(u) \geq 0$, we have $\frac{dR}{d\delta} > 0$ for all $\delta \geq 0$.
    
    Therefore, $R(\mu_1, \mu_2)$ is strictly monotonically increasing with respect to $\delta = \mu_1 - \mu_2$.
\end{proof}

Using \autoref{thm:normal_tail_ratio_monotonicity} we immediately get the following result:
\begin{corollary}
    Under the conditions of \autoref{thm:normal_tail_ratio_monotonicity}, if $\sigma_1 = \sigma_2 = \sigma$, then:
    \begin{equation}
    R(\delta) = \frac{\Phi\left(\frac{\delta}{2\sigma}\right)}{1 - \Phi\left(\frac{\delta}{2\sigma}\right)},
    \end{equation}
    and $\frac{dR}{d\delta} = \frac{\phi\left(\frac{\delta}{2\sigma}\right)}{2\sigma\left[1 - \Phi\left(\frac{\delta}{2\sigma}\right)\right]^2} > 0$.
\end{corollary}


\begin{remark}
    \cref{thm:normal_tail_ratio_monotonicity} shows that as the difference between the means of two normal distributions increases, the ratio of their right-tail integrals (with the midpoint of the means as the boundary) monotonically increases.
\end{remark}

\subsection{Proof of \autoref{thm:general_voting_monotonicity}}
\begin{proof}
    Consider the weighted sums for correct ($S_f$) and incorrect ($S_g$) classifiers:
    \begin{equation}
        \begin{aligned}
            S_f &= \sum_{i\in I_f} w_i X_i \sim \mathcal{N}\left(\mu_1 W_f, \sigma_1^2 W_f^{(2)}\right), \\
            S_g &= \sum_{j\in I_g} w_j X_j \sim \mathcal{N}\left(\mu_2 W_g, \sigma_2^2 W_g^{(2)}\right).
        \end{aligned}
    \end{equation}
    where:
    \begin{itemize}
    \item $W_f = \sum_{i\in I_f} w_i$, $W_g = \sum_{j\in I_g} w_j$ are the total weights,
    \item $W_f^{(2)} = \sum_{i\in I_f} w_i^2$, $W_g^{(2)} = \sum_{j\in I_g} w_j^2$ are the squared weight sums,
    \item $\mu_1 = \mu + \delta$, $\mu_2 = \mu$ are the means with $\delta > 0$,
    \item $\sigma_1^2$, $\sigma_2^2$ are the variances.
    \end{itemize}
    
    The voting accuracy is bounded below by:
    \begin{equation}
    P_{\mathrm{vote}}(\delta) \geq \mathbb{P}(S_f > S_g) =: P_{\mathrm{lower}}(\delta),
    \end{equation}
    
    where \(\mathbb{P}(S_f > S_g)\) serves as a lower bound for \(P_{\mathrm{vote}}(\delta)\). The difference distribution is:
    \begin{equation}
    S_f - S_g \sim \mathcal{N}\left(\delta W_f + \mu_2(W_f - W_g), \sigma_1^2 W_f^{(2)} + \sigma_2^2 W_g^{(2)}\right).
    \end{equation}
    
    Expressed via the standard normal Cumulative Distribution Function (CDF):
    \begin{equation}
    P_{\mathrm{lower}}(\delta) = \Phi\left(\frac{\delta W_f + \mu_2(W_f - W_g)}{\sqrt{\sigma_1^2 W_f^{(2)} + \sigma_2^2 W_g^{(2)}}}\right).
    \end{equation}
    
    The derivative with respect to $\delta$ is:
    \begin{equation}
    \frac{d}{d\delta}P_{\mathrm{lower}}(\delta) = \phi\left(\frac{\delta W_f + \mu_2(W_f - W_g)}{\sqrt{\sigma_1^2 W_f^{(2)} + \sigma_2^2 W_g^{(2)}}}\right)\cdot\frac{W_f}{\sqrt{\sigma_1^2 W_f^{(2)} + \sigma_2^2 W_g^{(2)}}} > 0.
    \end{equation}
    where $\phi(\cdot)>0$ is the standard normal PDF (positive everywhere), $W_f>0$ (non-trivial weights), and the denominator is positive ($\sigma_1,\sigma_2>0$).
    
    Since $P_{\mathrm{vote}}(\delta) \geq P_{\mathrm{lower}}(\delta)$ and the lower bound is strictly increasing in $\delta$, the voting accuracy $P_{\mathrm{vote}}(\delta)$ must also be strictly increasing in $\delta$.
\end{proof}

\newpage
\section{Experiment Description}
\label{sec:exp_descrip}

\subsection{Baseline and Competitors}
\label{sec:experiment_baseline}
    The core contributions of our proposed method lie in three key aspects: confidence computation, dynamic adjustment mechanisms within single question inference, and distribution-based voting methods. For confidence computation, we adopt the negative average log-probability from Self-Certainty~\cite{Self-Certainty} and the token group concept from DeepConf~\cite{DeepConf}, selecting the tail group containing the answer to calculate the overall trajectory confidence. However, our approach differs in that we do not use fixed window groups, but instead partition semantic steps according to reasoning blocks. The experimental section analyzes the advantages of this approach (\autoref{fig:supp_exp_step_split_ablation} in \autoref{sec:supp_exp_step_split_ablation}). Regarding the single inference process adjustment mechanism, since dynamic adjustment of reasoning processes at test-time has not been extensively explored, we primarily compare against basic inference methods. For the voting method, we use weighted majority voting as our baseline and compare with other test-time scaling voting approaches, including Self-Consistency (SC)~\cite{Self-Consistency}, BoN~\cite{BoN}, MoB (using confidence as the reward)~\cite{MoB}, Weighted-SC (WSC)~\cite{WSC}, DeepConf (WSC-Top50)~\cite{DeepConf}.

\subsection{Implementation Details}
\label{sec:experiment_implement_detail}
    Regarding the implementation details of our methods, for trajectory step segmentation, we select ``\textbackslash n\textbackslash n'' as the reasoning block signal, which has been thoroughly studied in StepPruner~\cite{StepPruner}. Regarding reflection information, we simply used ``wait'' as the predefined reflection token and supplemented more comparative experiments in \autoref{tab:ablation_reflection_token} of \autoref{sec:supp_exp_reflection_info_ablation}. For the EMA update parameter $\alpha$ controlling $\tau_{\text{conf}}$ in \textit{SelfStepConf} and the parameter $\delta$ controlling reflection trigger conditions, we set both to 0.8. For $N_{\mathcal{C}}$ in base voting, we simply set it to 10. These parameters are analyzed accordingly in the \autoref{sec:appendix_parameters_analysis}. For the device used in the experiment, unless otherwise specified, all experiments were conducted on NVIDIA H-Series GPU. In addition, for the parameter top-$p$ in the inference process, we set it to 0.95.

    \paragraph{Prompt Template.} In our experiments, we used two different prompt templates to handle different benchmark. For GPQA-D, the format of the prompt is:
    \begin{small}
    \begin{verbatim}
    {
        "role": "user",
        "content": "Return your final response within \\boxed{} and only 
                    include the letter choice (A, B, C, or D) as your 
                    final response. {Question}"
    }
    \end{verbatim}
    \end{small}
    
    For HMMT2025, AIME2024, AIME2025, and BRUMO2025, the specific format is:
    \begin{small}
    \begin{verbatim}
    {
        "role": "user",
        "content": "{Question}\nPlease reason step by step, 
                    and put your final answer within \\boxed{}."
    }
    \end{verbatim}
    \end{small}

    \paragraph{Evaluation.}For the main experimental setup, we employ the Qwen3~\cite{Qwen3} series models ranging from 0.6B to 32B (operating in thinking mode by defaul, and models marked with * use non-thinking mode), along with DeepSeek-R1-0528-Qwen3-8B and DeepSeek-R1-Distill-Qwen-7B models~\cite{DeepSeek-r1} (referred to as DeepSeek-R1-8B and DeepSeek-R1-7B, respectively). In addition, we also used the Qwen2.5-Math-7B~\cite{Qwen2.5-Math} and Llama-3.1-8B-Instruct~\cite{Llama3.1} in analysis. We evaluate our approach on five mathematical reasoning benchmarks: HMMT2025~\cite{HMMT&BRUMO}, GPQA-D~\cite{GPQA-D}, AIME2024~\cite{AIME}, AIME2025~\cite{AIME}, and BRUMO2025~\cite{HMMT&BRUMO}. Following baseline configurations, we set the context length to 32k for Qwen3 series models and 64k for DeepSeek-Distill models. Unless otherwise specified, the temperature is configured as 0.7 for Qwen3 non-thinking mode and 0.6 for all other settings. All pass@1 metrics are computed by averaging results across 64 independent evaluations, while voting experiments use a $\text{Budget}$ of 128 trajectories by default. In the experiments, all average results are computed using weighted averages based on the number of questions in each benchmark. 

\newpage
\section{Pseudo Code}
\label{sec:appendix_pseudo_code}

\begin{algorithm}[htb]
\caption{SelfStepConf with Voting}
\label{alg:selfstepconf}
\begin{algorithmic}
\STATE {\bfseries Input:} Question $q$, Budget $B$, Reflection tokens $T_r$, Threshold $\delta$, EMA factor $\alpha$
\STATE {\bfseries Output:} Final answer $A_{\text{final}}$
\STATE
\STATE Initialize $\mathcal{V} = \emptyset$, $\mathcal{C} = \emptyset$
\STATE
\STATE {\color{gray} \textit{// SelfStepConf}}
\FOR{$b = 1$ {\bfseries to} $B$}
    \STATE Initialize $\tau_{\text{conf}} = 0$, $m = 0$
    \WHILE{not end of generation}
        \STATE Generate token $t_i$ and compute $C_{\text{token}-i}$
        \IF{step boundary detected}
            \STATE $m \leftarrow m + 1$, compute $C_{G_m}$
            \STATE Compute $\Delta_{\text{conf}}$ and $\mathbb{I}_{R}$ using Equations~\ref{eq:delta_conf} and~\ref{eq:reflection_trigger}
            \IF{$\mathbb{I}_{R} = 1$}
                \STATE Inject reflection tokens using Equation~\ref{eq:add_reflection}
            \ELSE
                \STATE Update $\tau_{\text{conf}}$ using Equation~\ref{eq:conf_threshold_EMA_update}
            \ENDIF
        \ENDIF
    \ENDWHILE
    \STATE Compute trajectory confidence and add to $\mathcal{V}$, $\mathcal{C}$
\ENDFOR
\STATE
\STATE {\color{gray} \textit{// GMM Filter}}
\STATE Fit GMM and map to correct/incorrect using Equation~\ref{eq:gmm_split}
\STATE Filter $\mathcal{V}_{\text{pos}}$ and $\mathcal{V}_{\text{neg}}$ using Equation ~\ref{eq:get_traj}
\STATE {\color{gray} \textit{// Reject Filter}}
\STATE Get $A_{\text{neg}} = f_{\text{HierV}}(\mathcal{V}_{\text{neg}}, -\mathcal{C}_{\text{neg}})$
\STATE Get $A_{\text{pos}} = f_{\text{HierV}}(\mathcal{V}_{\text{pos}}, \mathcal{C}_{\text{pos}})$
\IF{$A_{\text{pos}} \neq A_{\text{neg}}$}
    \STATE \textit{Reject Filter} trajectories and re-fit GMM to get $\hat{\mathcal{V}}_{\text{pos}}$ using Equation~\ref{eq:reject_filer}
    \STATE $A_{\text{final}} = f_{\text{HierV}}(\hat{\mathcal{V}}_{\text{pos}}, \hat{\mathcal{C}}_{\text{pos}})$
\ELSE
    \STATE $A_{\text{final}} = A_{\text{pos}}$
\ENDIF
\STATE \textbf{return} $A_{\text{final}}$
\end{algorithmic}
\end{algorithm}

\begin{algorithm}[hb]
\caption{Base Voting with Hierarchical Strategy}
\label{alg:basevoting}
\begin{algorithmic}
\STATE {\bfseries Input:} Trajectory set $\mathcal{V}$, Confidence set $\mathcal{C}$, Number of intervals $N_{\mathcal{C}}$
\STATE {\bfseries Output:} Voting result $A$
\STATE
\STATE Divide confidence range into $N_{\mathcal{C}}$ intervals using Equation~\ref{eq:divide_conf2interval}
\FOR{$i = 1$ {\bfseries to} $N_{\mathcal{C}}$}
    \STATE Get $\mathcal{V}^i$ for interval $\mathcal{C}^i$
    \IF{$\mathcal{V}^i \neq \emptyset$}
        \STATE Compute $A_{\text{sub}}^i = f_{\text{WMaj}}(\mathcal{V}^i, \mathcal{C}^i)$ using Equation~\ref{eq:voting_interval}
        \STATE Add $(A_{\text{sub}}^i, \mathcal{\bar{C}}^i_{A_{\text{traj}}=A_{\text{sub}}^i})$ to $\mathcal{A}_{\text{sub}}$
    \ENDIF
\ENDFOR
\STATE $A = f_{\text{WMaj}}(\mathcal{A}_{\text{sub}})$ using Equation~\ref{eq:voting_final}
\STATE \textbf{return} $A$
\end{algorithmic}
\end{algorithm}

\begin{algorithm}[hb]
\caption{Weighted Majority Voting}
\label{alg:wmaj}
\begin{algorithmic}
\STATE {\bfseries Input:} Trajectory set $\mathcal{V}$, Weight set $\mathcal{W}$
\STATE {\bfseries Output:} Selected answer $A$
\STATE
\STATE Extract unique answers $\mathcal{A} = \{A_{\text{traj}} \mid \text{traj} \in \mathcal{V}\}$
\FOR{each $\text{ans} \in \mathcal{A}$}
    \STATE Compute $\text{scores}[\text{ans}]$ using Equation~\ref{eq:voting_maj}
\ENDFOR
\STATE $A = \arg\max_{\text{ans} \in \mathcal{A}} \text{scores}[\text{ans}]$
\STATE \textbf{return} $A$
\end{algorithmic}
\end{algorithm}

\clearpage
\newpage
\section{Parameters Analysis}
\label{sec:appendix_parameters_analysis}

According to the implementation logic of \textit{SelfStepConf} and \textit{DistriVoting}, the main parameters include: \textbf{(1)} top-$k$ for computing token-level confidence; \textbf{(2)} EMA factor $\alpha$ for dynamically updating confidence threshold in SSC; \textbf{(3)} $\delta$ for determining reflection trigger conditions in SSC; \textbf{(4)} the number of intervals $N_{\mathcal{C}}$ specified in \textit{HierVoting}.

\subsection{Analysis of top-$k$}
\label{sec:appendix_parameters_analysis_k}
First, for top-k, we follow the setting from prior work~\cite{DeepConf}, setting $k=20$. This choice is empirically justified by the probability distribution characteristics of modern LLMs. Through analysis of randomly sampled inference traces, we observe that at over 99.7\% of token positions, the top-20 tokens account for more than 99\% of the cumulative probability mass. This demonstrates that $K=20$ effectively captures nearly all meaningful probability signal while maintaining computational efficiency. Additionally, this setting aligns well with practical deployment considerations, as standard inference engines like vLLM typically provide at most 20 log-probabilities by default.

\subsection{Analysis of $\alpha$}
\label{sec:appendix_parameters_analysis_alpha}
For $\alpha$, $\delta$, and $N_{\mathcal{C}}$, we primarily analyze their robustness through experiments. The EMA factor $\alpha$ mainly determines the proportion of the previous threshold when updating the threshold, corresponding to the case where $\mathbb{I}_R=0$ in Equation~\ref{eq:conf_threshold_EMA_update}. A larger $\alpha$ gives more weight to the confidence from previous steps, resulting in a smoother threshold curve, while a smaller $\alpha$ makes the threshold tend toward the step confidence $G_m$.

    \begin{figure}[ht]
        \begin{center}
        \begin{subfigure}[b]{0.39\textwidth}
                \centering
                \includegraphics[width=\textwidth]{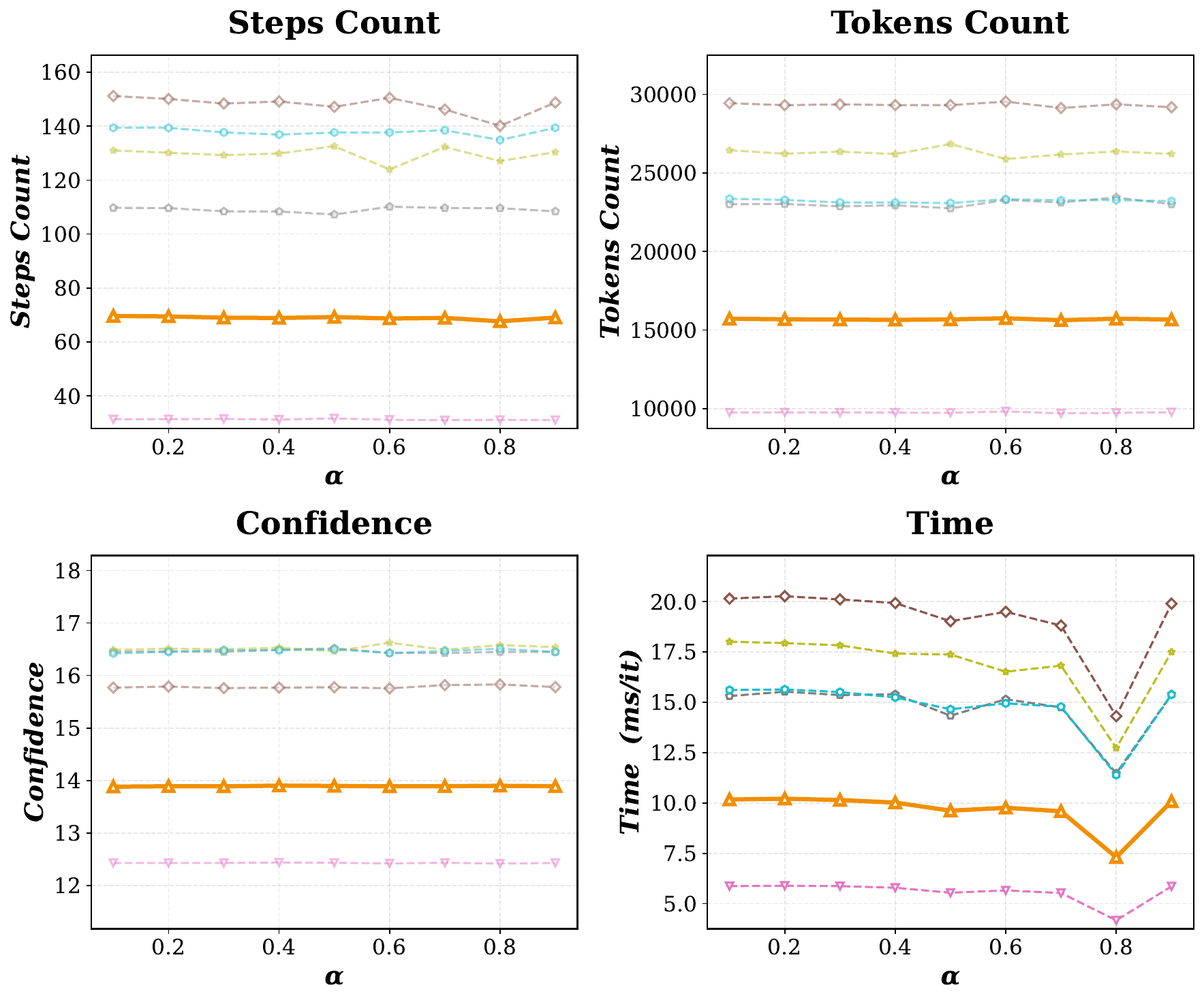}
                \caption{Analysis Metrics}
        \end{subfigure}
        \begin{subfigure}[b]{0.51\textwidth}
                \centering
                \includegraphics[width=\textwidth]{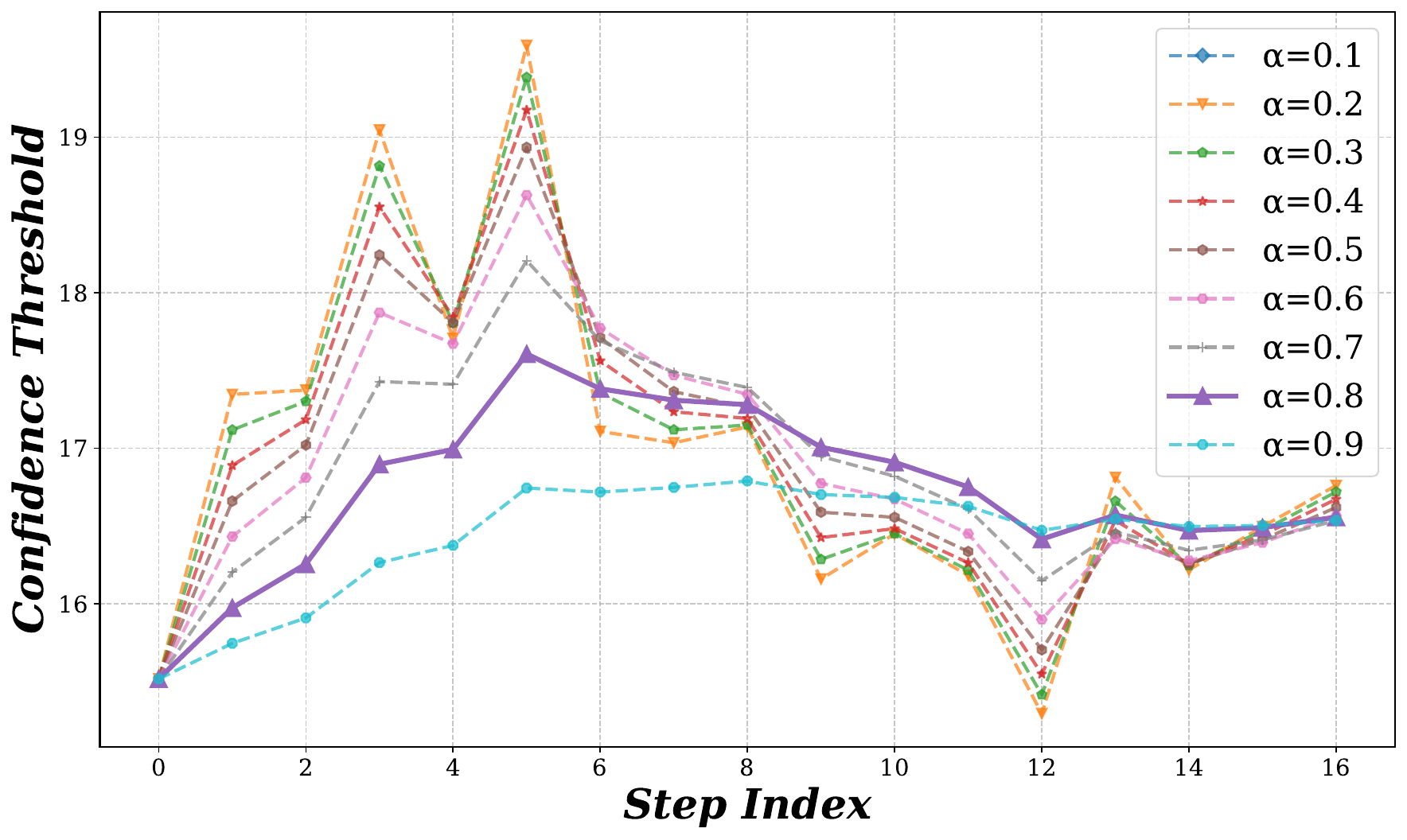}
                \caption{Variation of adaptive confidence threshold $\tau_{\text{conf}}$ with $\alpha$}
        \end{subfigure}
        \vskip 0.1in
        \begin{subfigure}[b]{0.85\textwidth}
                \centering
                \includegraphics[width=\textwidth]{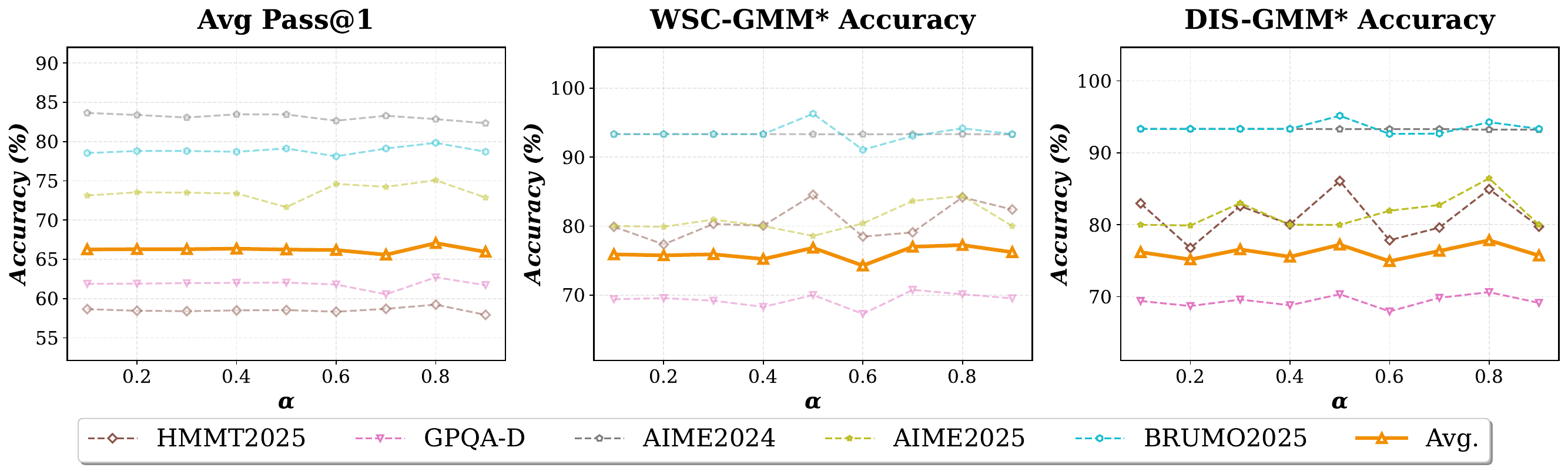}
                \caption{Accuracy Metrics}
        \end{subfigure}
        \caption{
          Parameter sensitivity analysis of $\alpha$ under different metrics. Using DeepSeek-R1-8B to generate 128 trajectories per query, with experiments repeated 64 times. Complete results are provided in \autoref{tab:appendix_complete_parameters_alpha_analysis} of \autoref{sec:appendix_complete_parameters_alpha_analysis}.
        }
        \label{fig:appendix_parameters_analysis_alpha}
        \end{center}
        \vskip -0.1in
    \end{figure}

We analyzed the sensitivity of the parameter $\alpha$ by examining both Analysis Metrics (Steps Count, Tokens Count, Confidence, and Time) and Accuracy Metrics (Avg Pass@1, WSC-GMM* Accuracy, and DIS-GMM* Accuracy). Note that there are some discrepancies between the Analysis Metrics and the results in \autoref{tab:analysis_ssc_token} of \autoref{sec:analysis_ssc_token}, as we used a temperature of 0.6 and repeated the process 64 times here (NVIDIA H-Series GPU), whereas the main text used a temperature of 0 (NVIDIA A-Series GPU). As shown in \autoref{fig:appendix_parameters_analysis_alpha} (a) and (c), within the range of $\alpha$ from 0.1 to 0.9, all metrics except time show minimal variation. This indicates that our proposed \textit{SelfStepConf} and \textit{DistriVoting} are not sensitive to the EMA update factor of the confidence threshold. Therefore, considering both efficiency and effectiveness, we opted for a setting of $\alpha=0.8$.

Additionally, to visually demonstrate the impact of different $\alpha$ values on \textit{SelfStepConf}, we plotted the change in confidence threshold over steps for the same query under different $\alpha$ values in \autoref{fig:appendix_parameters_analysis_alpha} (b) (with temperature set to 0 and sample 1 trajectory for Question 7 of HMMT2025). It is evident that as $\alpha$ increases, the variation in the confidence threshold decreases, which aligns with the role of $\alpha$ as described in \autoref{eq:conf_threshold_EMA_update}.

\subsection{Analysis of $\delta$}
\label{sec:appendix_parameters_analysis_delta}
For $\delta$ in Equation~\ref{eq:reflection_trigger}, it determines how much the current step confidence needs to drop relative to the confidence threshold to trigger reflection. This value affects the strictness of the confidence checking mechanism: a larger $\delta$ means that even a slight drop in step confidence relative to the threshold will trigger reflection, while a smaller $\delta$ shows greater tolerance for drops in step confidence.

    \begin{figure}[ht]
        \begin{center}
        \begin{subfigure}[b]{0.38\textwidth}
                \centering
                \includegraphics[width=\textwidth]{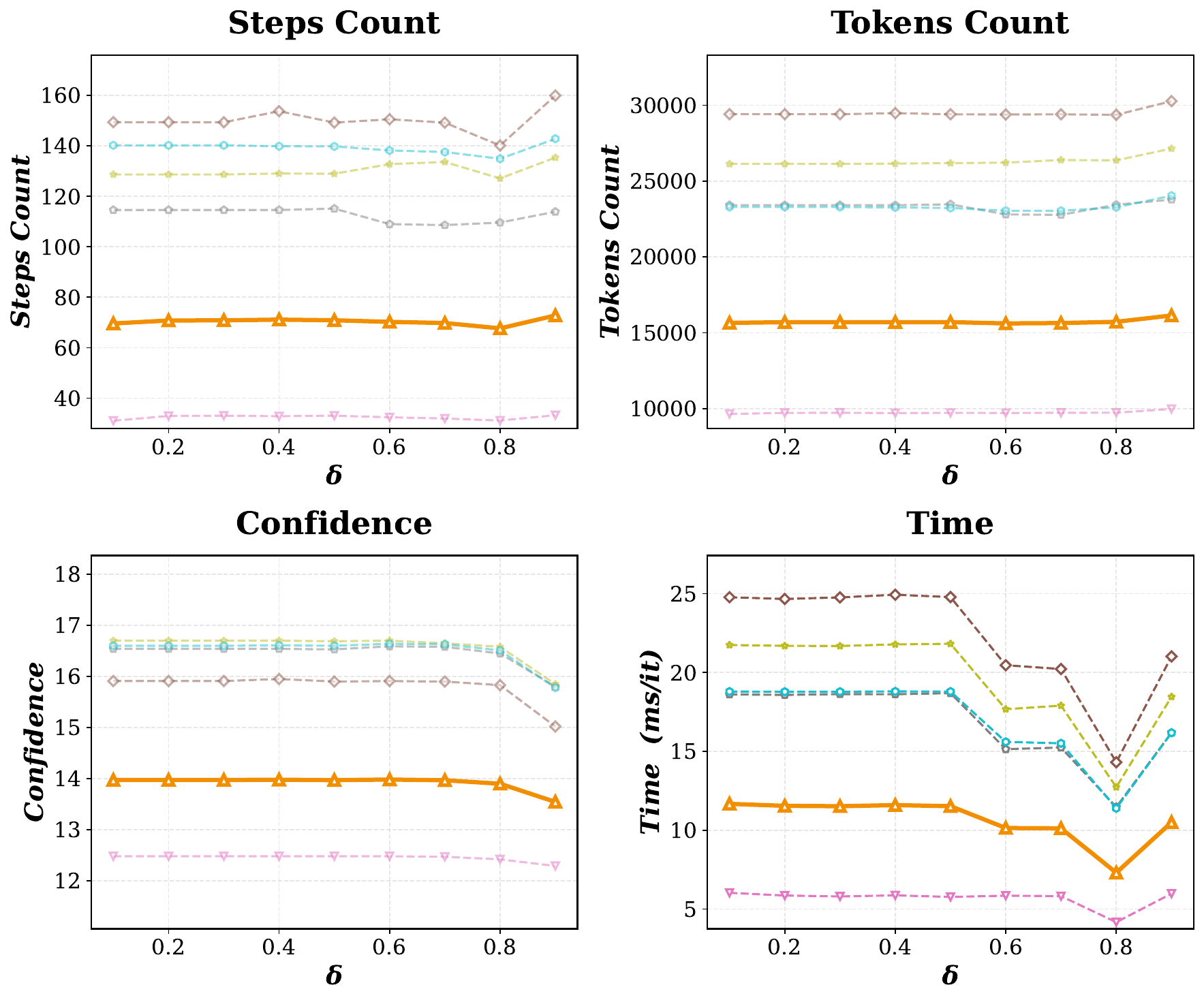}
                \caption{Analysis Metrics}
        \end{subfigure}
        \begin{subfigure}[b]{0.52\textwidth}
                \centering
                \includegraphics[width=\textwidth]{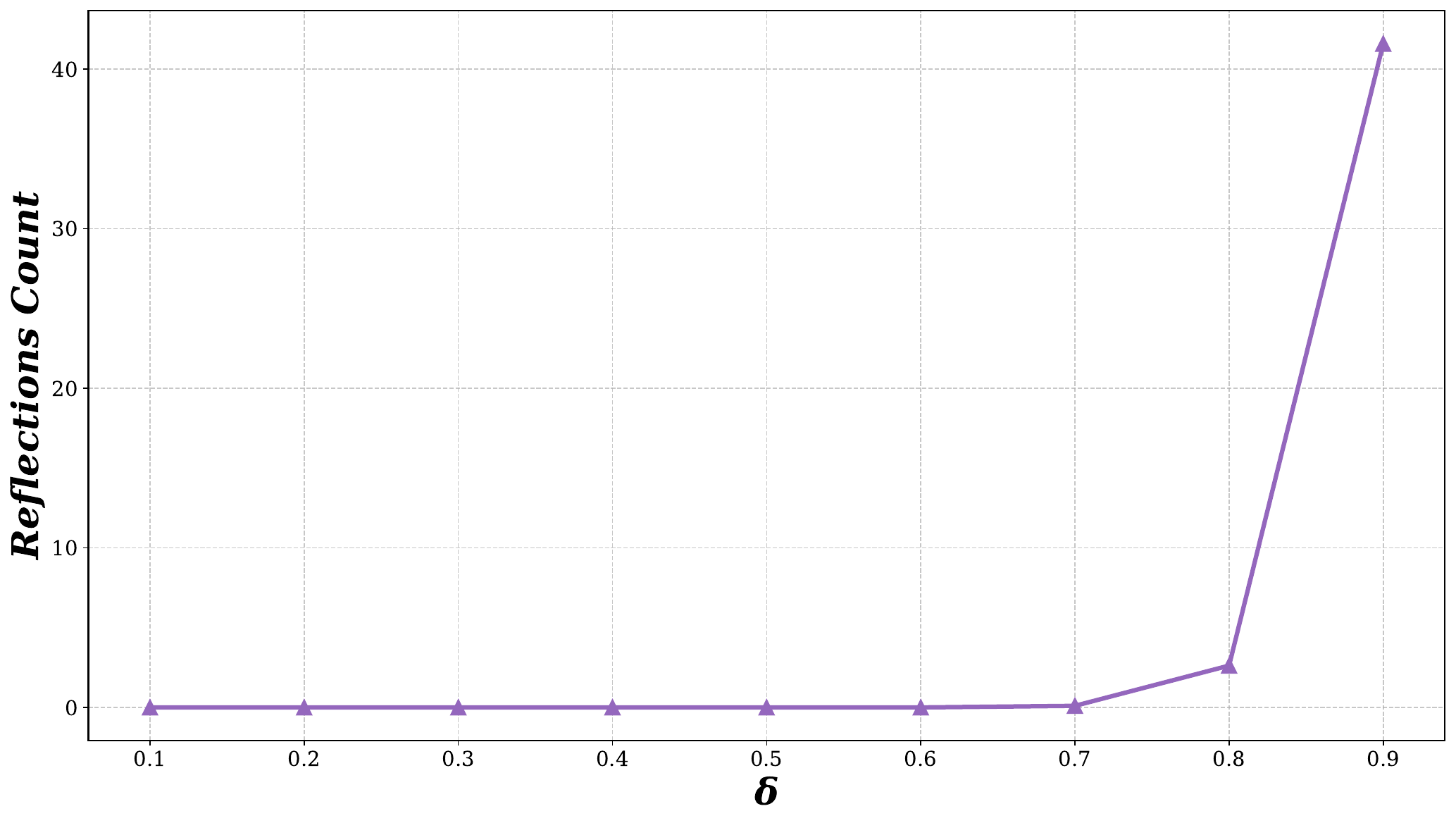}
                \caption{Variation of adaptive reflection triggers counts with $\delta$}
        \end{subfigure}
        \vskip 0.1in
        \begin{subfigure}[b]{0.85\textwidth}
                \centering
                \includegraphics[width=\textwidth]{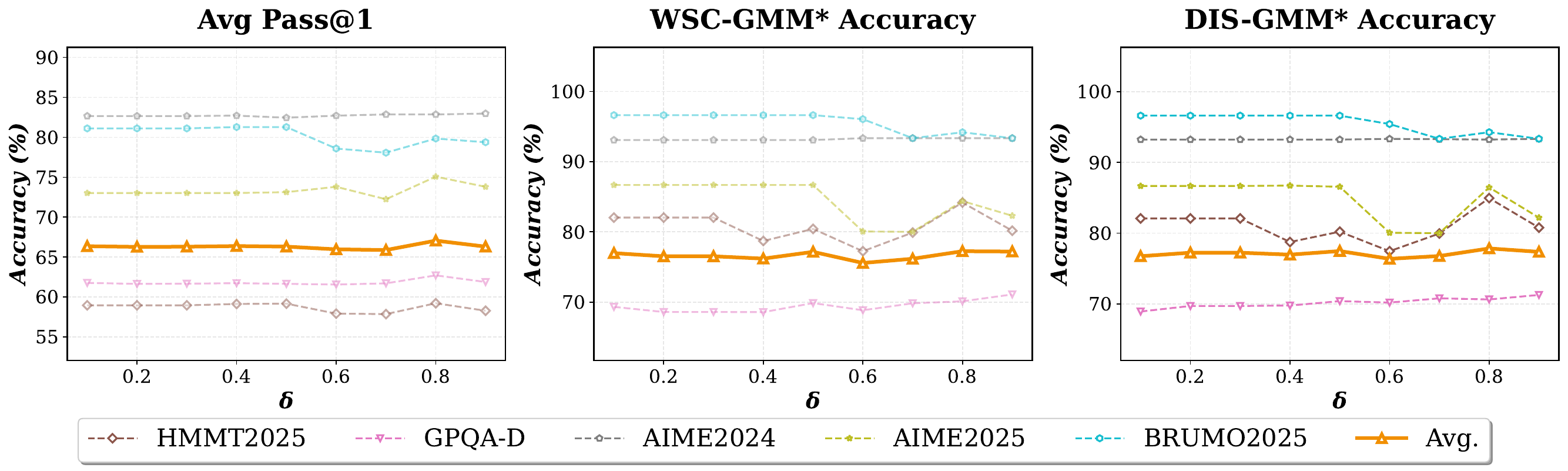}
                \caption{Accuracy Metrics}
        \end{subfigure}
        \caption{
          Parameter sensitivity analysis of $\delta$ under different metrics. Using DeepSeek-R1-8B to generate 128 trajectories per query, with experiments repeated 64 times. Complete results are provided in \autoref{tab:appendix_complete_parameters_delta_analysis} of \autoref{sec:appendix_complete_parameters_delta_analysis}.
        }
        \label{fig:appendix_parameters_analysis_delta}
        \end{center}
        \vskip -0.1in
    \end{figure}

Similar to the analysis of parameter $\alpha$, we examined the sensitivity of parameter $\delta$ from the perspectives of Analysis Metrics (Steps Count, Tokens Count, Confidence, and Time) and Accuracy Metrics (Avg Pass@1, WSC-GMM* Accuracy, and DIS-GMM* Accuracy), maintaining the same experimental setup. As shown in \autoref{fig:appendix_parameters_analysis_delta} (a) and (c), parameter $\delta$ exhibits similar sensitivity to $\alpha$, being insensitive to all metrics except time. Therefore, considering both efficiency and effectiveness, we chose $\delta=0.8$.

Additionally, the role of $\delta$ primarily involves controlling the strictness of reflection triggers in \autoref{eq:reflection_trigger}. To visually demonstrate its impact on \textit{SelfStepConf}, we calculated the average number of reflection triggers for all questions in HMMT2025 under different $\delta$ settings with temperature set to 0. As shown in \autoref{fig:appendix_parameters_analysis_delta} (b), the number of reflection triggers increases with $\delta$, indicating that $\delta$ is directly proportional to the strictness of the reflection trigger conditions. Furthermore, when $\delta \leq 0.7$, the number of reflections drops to zero, causing \textit{SelfStepConf} to degrade to BasicInference. Thus, from the perspective of reflection trigger conditions, $\delta=0.8$ is also an appropriate choice.

\subsubsection{Analysis of $\alpha \times \delta$}
\label{sec:appendix_parameters_analysis_alphaDelta}

To further analyze the sensitivity of parameters $\alpha$ and $\delta$ in SelfStepConf, we conducted a joint analysis using HMMT2025. Specifically, we set the range of $\alpha$ and $\beta$ to $[0.1, 0.9]$, and then calculated the Analysis Metrics (Steps Count, Tokens Count, Confidence, Reflections Count, and Time) and Accuracy Metrics (Avg Pass@1, WSC-GMM* Accuracy, and DIS-GMM* Accuracy) under the group parameters. As shown in \autoref{fig:appendix_parameters_analysis_alphaDelta}, it can be observed that under joint analysis, $\alpha=0.8$ and $\delta=0.8$ remain a preferable choice. Furthermore, when $\delta$ is small ($< 0.2$), it can be observed that the various metrics show almost no change, which further proves that SelfStepConf degrades to Basic Inference in this case.

    \begin{figure}[ht]
        \begin{center}
        \centerline{\includegraphics[width=0.95\columnwidth]{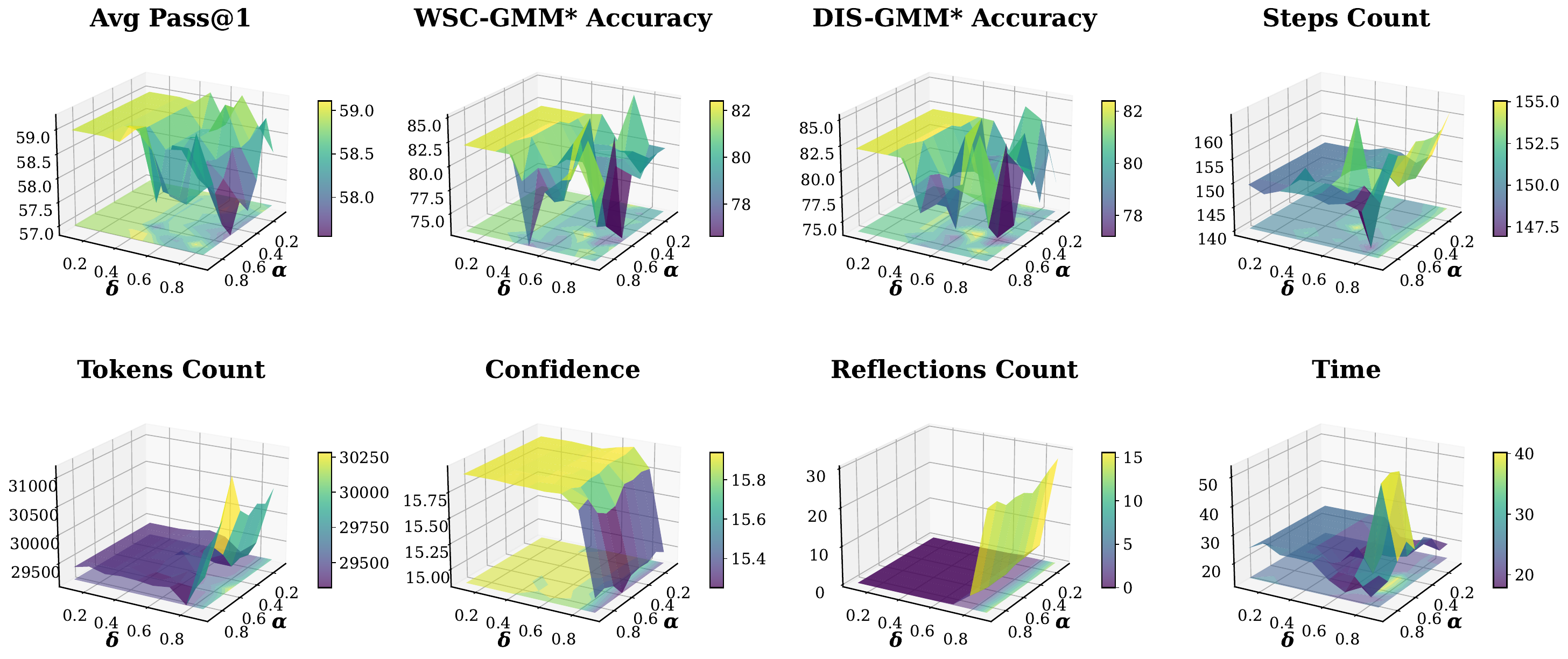}}
        \caption{
          Parameter sensitivity analysis of $\alpha \times \delta$ under different metrics. Using DeepSeek-R1-8B to generate 128 trajectories per query of HMMT2025, with experiments repeated 64 times. Complete results are provided in \autoref{tab:appendix_complete_parameters_alphaDelta_analysis_1} and \autoref{tab:appendix_complete_parameters_alphaDelta_analysis_2} of \autoref{sec:appendix_complete_parameters_alphaDelta_analysis}.
        }
        \label{fig:appendix_parameters_analysis_alphaDelta}
        \end{center}
        \vskip -0.1in
    \end{figure}

\subsection{Analysis of $N_{\mathcal{C}}$}
\label{sec:appendix_parameters_analysis_Nc}
Finally, $N_{\mathcal{C}}$ in Equation~\ref{eq:divide_conf2interval} determines how many sub-intervals the confidence should be divided into during the BaseVoting process for Weighted Majority Voting. We evaluated three BaseVoting-based methods (DIS-Top50, DIS-GMM, and DIS-GMM*) across $N_\mathcal{C} \in [1,20]$, as shown in Figure~\ref{fig:appendix_parameters_analysis_Nc}. The results show that except for DIS-Top50 at $N_\mathcal{C}=1$ which exhibits notably different performance, all methods achieve similar results across different values of $N_\mathcal{C}$. This indicates that BaseVoting demonstrates considerable robustness for $N_\mathcal{C}>1$. For simplicity, we adopt the default setting of $N_\mathcal{C}=10$ throughout all experiments.

Based on our experimental results, we observe an interesting phenomenon: \textbf{stratified voting strategies ($N_\mathcal{C} > 1$) significantly improve performance only for the DIS-Top50 method, while showing limited effectiveness for \textit{GMM Filter} approaches}. The fundamental reason why \textit{GMM Filter} is insensitive to stratification operations lies in its powerful quality filtering capability, which has already addressed the problems that stratification strategies attempt to solve. Specifically, \textit{GMM Filter} uses probabilistic modeling to precisely identify high-quality trajectories, resulting in relatively uniform quality across different confidence intervals after filtering, thereby reducing the necessity for stratification. Unlike DIS-Top50's simple threshold selection, \textit{GMM Filter} has already removed most low-quality trajectories globally, making the voting results of remaining samples inherently stable. When trajectory quality has already reached a high level through \textit{GMM Filter}, further stratification operations provide limited performance gains, explaining why GMM and GMM* maintain relatively stable performance across different $N_\mathcal{C}$ settings. This observation is further supported by the trend shown in Figure~\ref{fig:appendix_parameters_analysis_Nc} (b), where the performance difference between $N_c=1$ and $N_C=2$ decreases as the top-threshold increases. This finding demonstrates that \textbf{high-quality pre-filtering mechanisms can significantly reduce dependence on complex voting strategies}, which is consistent with the ablation analysis of  \textit{HierVoting} discussed in \autoref{tab:appendix_complete_main_results_ablation} of \autoref{sec:appendix_complete_main_results}.

    \begin{figure}[ht]
        \vskip 0.2in
        \begin{center}
            \begin{subfigure}[b]{0.48\textwidth}
                \centering
                \includegraphics[width=\textwidth]{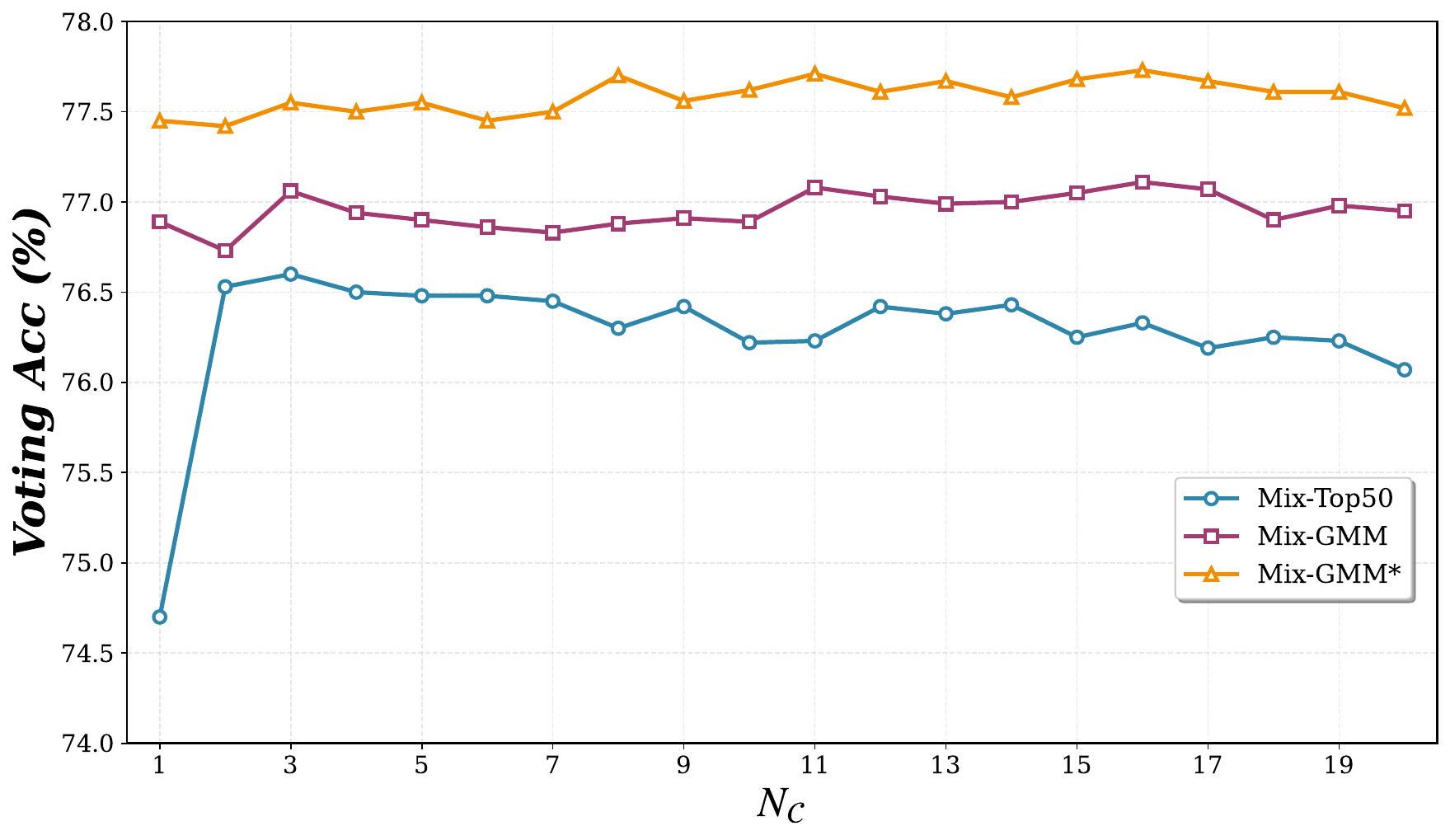}
                \caption{$N_\mathcal{C}$ sensitivity under different \textit{DistriVoting} methods}
            \end{subfigure}
            \hfill
            \begin{subfigure}[b]{0.48\textwidth}
                \centering
                \includegraphics[width=\textwidth]{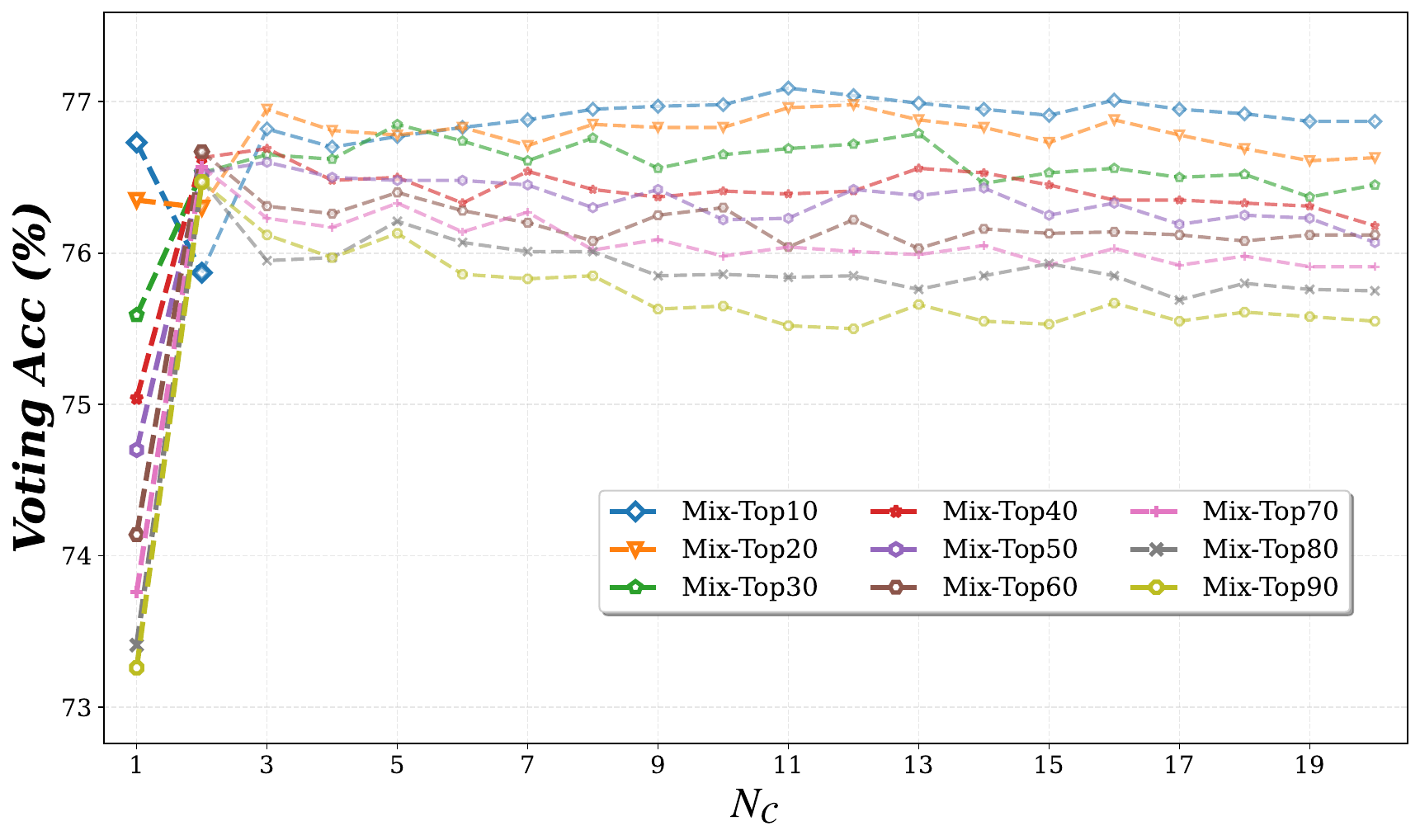}
                \caption{$N_\mathcal{C}$ sensitivity under different Top-Threshold}
            \end{subfigure}
        \caption{
          Parameter sensitivity analysis of $N_\mathcal{C}$ under different voting methods. Using DeepSeek-R1-8B to generate 128 trajectories per query for voting, with experiments repeated 64 times. Complete results are provided in \autoref{tab:appendix_complete_top50_analysis} of \autoref{sec:appendix_complete_top50_analysis} and \autoref{tab:appendix_complete_parameters_Nc_analysis_1}, \autoref{tab:appendix_complete_parameters_Nc_analysis_2} of \autoref{sec:appendix_complete_parameters_Nc_analysis}.
          }
        \label{fig:appendix_parameters_analysis_Nc}
        \vskip -0.5in
        \end{center}
    \end{figure}

\section{Supplementary experiment}
\label{sec:supp_exp}

\subsection{Reflection Information Ablation Study}
\label{sec:supp_exp_reflection_info_ablation}
    For \textit{SelfStepConf}'s operation of injecting additional reflection information at steps where confidence significantly decreases during inference, we primarily use ``wait'' as the critical fork token to trigger model reflection for simplicity. Additionally, we experimented with other reflection information as shown in \autoref{tab:ablation_reflection_token}. It can be observed that under the condition of not introducing additional parameters (i.e., without using other models for judgment or learning how to reflect through training), the performance differences across different tokens are not significant when only injecting reflection prompt tokens.

    \begin{table}[ht]
  \caption{Ablation study results of different reflection tokens. Using DeepSeek-R1-8B to generate 64 trajectories for each question across 5 benchmarks, computing the average pass@1.}
  \label{tab:ablation_reflection_token}
  \begin{center}
    \begin{tabular}{lcccc}
      \toprule
      Benchmark & ``wait'' & ``Wait'' & ``Hmm'' & ``Alternatively'' \\
      \midrule
      HMMT2025  & 59.22 & 58.91 & 59.53 & 58.75	\\
      GPQA-D    & 62.71 & 62.34 & 62.46 & 62.98 \\
      AIME2024  & 82.86 & 85.45 & 84.85 & 83.02 \\
      AIME2025  & 75.07 & 73.80 & 73.00 & 72.97 \\
      BRUMO2025 & 79.84 & 79.27 & 78.23 & 78.18 \\
      \midrule
      Avg.      & 67.06 & 66.87 & 66.78 & 66.85 \\
      \bottomrule
    \end{tabular}
  \end{center}
  \vskip -0.1in
\end{table}

\subsection{Step Split Ablation Study}
\label{sec:supp_exp_step_split_ablation}
    A key challenge in implementing \textit{SelfStepConf} lies in how to reasonably partition steps. Previous works have proposed partitioning by fixed-size windows~\cite{Re-Schedule, DeepConf, DORA}, by ``\textbackslash n''~\cite{GPO}, and by complete Thought-Action-Observation cycles in multi-turn dialogues~\cite{TreeGRPO}. Unlike these approaches, we need to dynamically adjust steps during test-time inference, thus we believe using ``\textbackslash n\textbackslash n'' to maintain paragraph-level logical integrity is more appropriate in reasoning tasks.
    
    Additionally, we refer to~\cite{Beyond80/20} and believe that high entropy tokens in trajectories may be potential fork tokens, thus conducting step checks at these critical positions provides decisive assistance for the overall reasoning direction. Specifically, we adopt the high entropy threshold of 0.672 from~\cite{Beyond80/20}, performing confidence checks when token entropy exceeds this value. However, during experiments, we found that directly splitting steps only according to high-entropy tokens results in excessively short step lengths. Therefore, we further set a lower bound of 200 for step length. As shown in \autoref{tab:ablation_step_split} and \autoref{fig:supp_exp_step_split_ablation}, Similar to the evaluation methods of \autoref{sec:appendix_parameters_analysis_alpha} and \autoref{sec:appendix_parameters_analysis_delta}, we compared six indicators such as Avg Pass@1, Steps Count, Tokens Count, Confidence, Reflections Count and Time. Compared with fixed-size window-level (256/512/1024/2048), dynamic sentence-level (``\textbackslash n''), and entropy-level approaches (HET, High-Entropy Token), paragraph-level partitioning (``\textbackslash n\textbackslash n'') is more suitable for SSC. Furthermore, in terms of efficiency, it can be seen that the influence of different step splits on inference time is relatively obvious. Combining steps count, tokens count, confidence, and reflection count, it can be known that this influence is mainly related to the number of times reflection is triggered, and is also affected by the other several indicators.

    \begin{table}[ht]
  \caption{Ablation study results of different step splitting methods. Using DeepSeek-R1-8B to generate 64 trajectories for each question across 5 benchmarks, computing the \textbf{average pass@1}. HET denotes High-Entropy Token.}
  \label{tab:ablation_step_split}
  \begin{center}
    \begin{tabular}{lccccccc}
      \toprule
      \multirow{2}{*}{Benchmark} & \multirow{2}{*}{``\textbackslash n\textbackslash n''} & \multirow{2}{*}{``\textbackslash n''} & \multirow{2}{*}{HET} & \multicolumn{4}{c}{Fixed Window} \\
      \cmidrule{5-8}
      & & & & 256 & 512 & 1024 & 2048 \\
      \midrule
      HMMT2025  & \textbf{59.22} & 58.85 & 58.91 & 58.39 & 58.75 & 58.06 & 57.76 \\
      GPQA-D    & \textbf{62.71} & 61.43 & 57.37 & 56.69 & 56.32 & 60.11 & 57.07 \\
      AIME2024  & 82.86 & 82.76 & 81.61 & 82.24 & \textbf{82.97} & 82.76 & 81.41 \\
      AIME2025  & \textbf{75.07} & 73.23 & 70.68 & 73.70 & 73.54 & 73.02 & 70.83 \\
      BRUMO2025 & \textbf{79.84} & 78.75 & 78.59 & 79.01 & 77.76 & 78.28 & 77.97 \\
      \midrule
      Avg.      & \textbf{67.06} & 65.95 & 63.06 & 62.97 & 62.71 & 64.99 & 62.70 \\
      \bottomrule
    \end{tabular}
  \end{center}
\end{table}

    \begin{figure}[ht]
        \vskip 0.2in
        \begin{center}
            \centering
            \includegraphics[width=0.95\columnwidth]{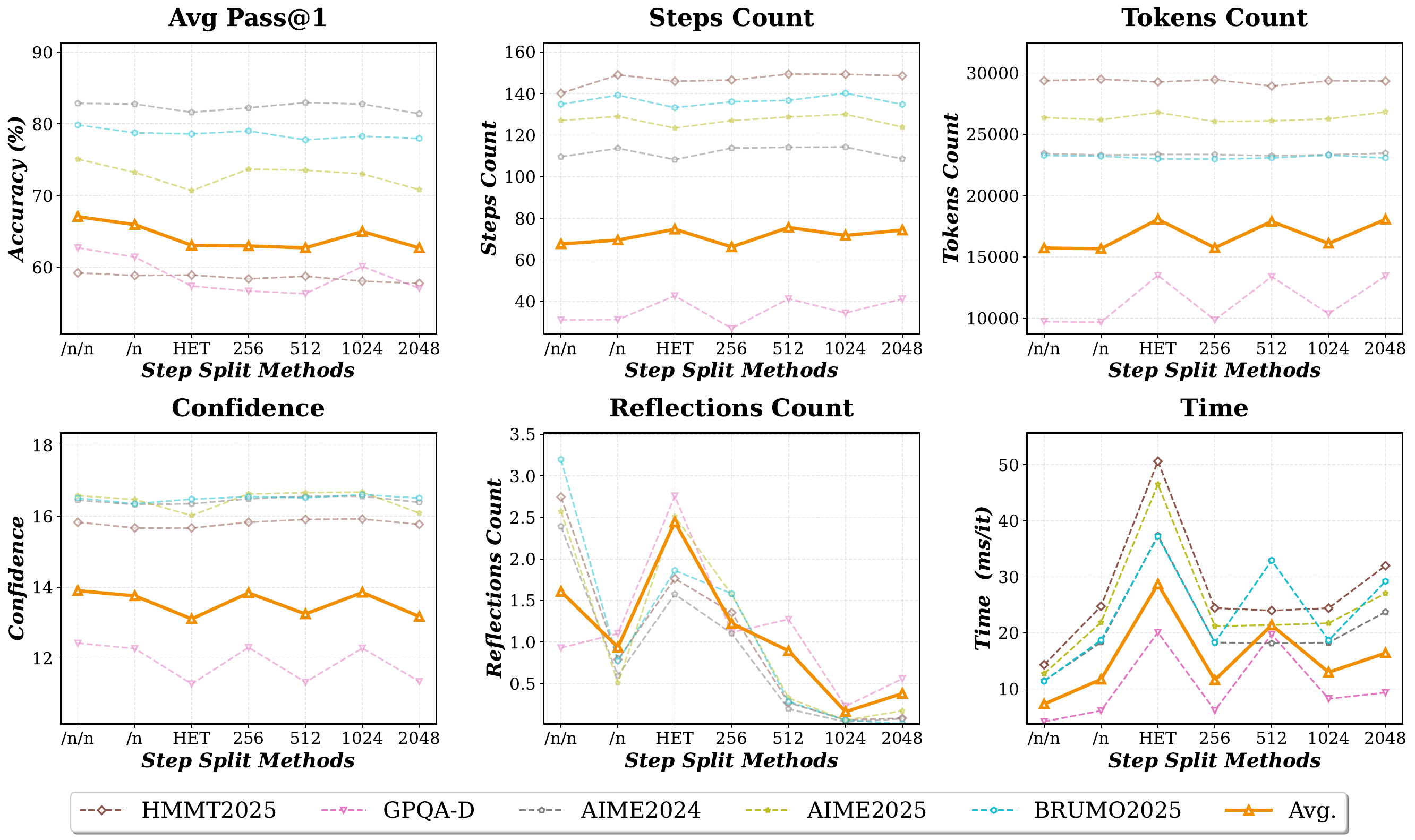}
        \caption{
          Ablation study results of different step splitting methods. Using DeepSeek-R1-8B to generate 64 trajectories for each question across 5 benchmarks. HET denotes High-Entropy Token. Complete results are provided in \autoref{tab:ablation_step_split_metric} of \autoref{sec:appendix_complete_step_split_ablation}.
          }
        \label{fig:supp_exp_step_split_ablation}
        \vskip -0.5in
        \end{center}
    \end{figure}

\newpage
\section{Complete experimental results}
\label{sec:appendix_complete_res}

\subsection{Main Experiments and Ablation Results on More Models}
\label{sec:appendix_complete_main_results}

    \paragraph{More Main Results.} In addition to the main experiments in the text involving the DeepSeek-R1 and Qwen3-32B models (\autoref{tab:main_results} in \autoref{sec:experiment_main_results}), We also tested the voting performance of \textit{SelfStepConf} and \textit{DistriVoting} on Qwen3 models ranging from Qwen3-0.6B to Qwen3-32B, including both thinking mode and nonthinking mode. The results shown in \autoref{tab:appendix_complete_main_results_thinking} and \autoref{tab:appendix_complete_main_results_nonthinking} are consistent with the analysis in the main text, showing that \textit{DistriVoting} consistently outperforms WSCVoting under different filter strategies. Moreover, both \textit{DistriVoting} and WSCVoting demonstrate that the \textit{GMM Filter} significantly outperforms the Top50 filter. The \textit{SelfStepConf} generation method, indicated by *, further amplifies this improvement. Notably, the voting performance improvement of SSC on models with weaker and stronger reasoning abilities is not as pronounced, this observation aligns with our pass@1 analysis results of \autoref{fig:analysis_ssc_pass@k_pass@1} in \autoref{sec:analysis_ssc_pass@k} of the main text.

    \begin{table*}[ht]
  \caption{All main results of \textit{SelfStepConf} (SSC) and \textit{DistriVoting} on various benchmarks using \textbf{Thinking Mode} of Qwen3-Series models. Set $\text{Budget}$ to 128 and repeat 64 times. SC denotes Self-Consistency (i.e., Majority Voting), WSC represents Weighted Self-Consistency (i.e., Weighted Majority Voting), BoN represents Best of N, and \textbf{* indicates answers generated using the SSC approach}.}
  \label{tab:appendix_complete_main_results_thinking}
  \begin{center}
    \resizebox{0.98\textwidth}{!}{
      \begin{tabular}{llccccccccccc}
        \toprule
        \multirow{2}{*}{Model} & \multirow{2}{*}{Benchmark} & \multirow{2}{*}{Pass@1} & \multirow{2}{*}{SC} & \multirow{2}{*}{WSC} & \multirow{2}{*}{BoN} & \multirow{2}{*}{MoB} & \multicolumn{3}{c}{WSC} & \multicolumn{3}{c}{DIS} \\
        \cmidrule(lr){8-10} \cmidrule(lr){11-13}
        & & & & & & & Top50 & GMM & GMM* & Top50 & GMM & GMM* \\
        \midrule
        \multirow{7}{*}{Qwen3-0.6B}
        & HMMT2025 & 8.89 & 18.97 & 20.37 & 17.24 & 17.22 & 17.54 & 14.51 & 16.65 & 17.98 & 20.74 & 14.21 \\
        & GPQA-D & 21.31 & 26.62 & 26.61 & 28.78 & 27.72 & 24.09 & 28.52 & 27.83 & 26.73 & 27.54 & 29.21 \\
        & AIME2024 & 10.73 & 17.08 & 17.55 & 13.44 & 23.95 & 29.58 & 23.84 & 29.74 & 26.20 & 28.45 & 27.22 \\
        & AIME2025 & 16.78 & 30.00 & 30.00 & 29.95 & 33.45 & 33.39 & 33.47 & 33.33 & 30.24 & 33.49 & 35.21 \\
        & BRUMO2025 & 16.93 & 23.13 & 23.33 & 23.39 & 30.25 & 29.79 & 28.83 & 28.96 & 30.45 & 29.66 & 27.11 \\
        \cmidrule(lr){2-13}
        & Avg. & 18.30 & 24.99 & 25.18 & 25.85 & 27.16 & 25.40 & 27.25 & 27.58 & 26.54 & 27.75 & 27.97 \\
        \midrule
        \midrule
        \multirow{7}{*}{Qwen3-1.7B}
        & HMMT2025 & 22.29 & 29.64 & 29.48 & 36.61 & 26.82 & 28.12 & 32.30 & 32.93 & 27.14 & 32.81 & 35.77 \\
        & GPQA-D & 34.91 & 43.22 & 43.94 & 42.91 & 43.51 & 43.60 & 44.74 & 48.36 & 45.51 & 45.85 & 48.26 \\
        & AIME2024 & 47.60 & 73.33 & 73.33 & 70.00 & 76.51 & 76.27 & 75.42 & 76.98 & 76.97 & 77.27 & 78.05 \\
        & AIME2025 & 34.48 & 46.67 & 47.29 & 39.95 & 50.00 & 50.08 & 55.25 & 53.07 & 52.19 & 54.85 & 51.94 \\
        & BRUMO2025 & 47.60 & 60.00 & 60.00 & 56.46 & 63.91 & 63.70 & 66.17 & 67.18 & 67.15 & 61.84 & 70.38 \\
        \cmidrule(lr){2-13}
        & Avg. & 36.07 & 46.69 & 47.18 & 45.87 & 47.59 & 47.73 & 49.48 & 51.82 & 49.42 & 49.94 & 52.33 \\
        \midrule
        \midrule
        \multirow{7}{*}{Qwen3-4B}
        & HMMT2025 & 42.34 & 53.33 & 53.33 & 43.39 & 53.78 & 53.42 & 53.33 & 56.57 & 56.76 & 53.75 & 55.27 \\
        & GPQA-D & 52.23 & 55.13 & 55.56 & 58.06 & 58.94 & 58.84 & 59.08 & 62.36 & 58.69 & 59.47 & 62.82 \\
        & AIME2024 & 71.30 & 80.00 & 80.00 & 76.72 & 83.35 & 83.47 & 83.33 & 82.41 & 84.01 & 83.41 & 83.91 \\
        & AIME2025 & 62.24 & 73.02 & 76.46 & 66.77 & 75.78 & 76.79 & 76.51 & 76.89 & 77.18 & 76.38 & 77.09 \\
        & BRUMO2025 & 62.60 & 67.55 & 70.00 & 73.39 & 76.83 & 75.07 & 76.04 & 77.10 & 77.13 & 76.22 & 76.67 \\
        \cmidrule(lr){2-13}
        & Avg. & 55.02 & 60.17 & 60.99 & 60.70 & 64.03 & 63.88 & 64.07 & 66.47 & 64.38 & 64.36 & 66.75 \\
        \midrule
        \midrule
        \multirow{7}{*}{Qwen3-8B}
        & HMMT2025 & 41.46 & 59.95 & 60.00 & 54.27 & 61.72 & 60.16 & 60.99 & 63.75 & 61.20 & 62.34 & 62.60 \\
        & GPQA-D & 54.70 & 63.83 & 63.80 & 63.22 & 64.11 & 64.35 & 64.48 & 65.37 & 64.29 & 64.43 & 66.31 \\
        & AIME2024 & 73.02 & 80.68 & 81.15 & 84.69 & 85.63 & 84.22 & 87.08 & 86.51 & 87.66 & 87.08 & 86.61 \\
        & AIME2025 & 62.55 & 77.19 & 77.55 & 64.79 & 74.32 & 74.95 & 74.22 & 72.45 & 73.65 & 73.59 & 74.11 \\
        & BRUMO2025 & 67.19 & 80.21 & 80.57 & 76.15 & 81.93 & 82.97 & 82.71 & 81.88 & 82.40 & 82.55 & 82.97 \\
        \cmidrule(lr){2-13}
        & Avg. & 57.10 & 67.86 & 67.96 & 65.77 & 68.56 & 68.59 & 68.92 & 69.44 & 68.79 & 68.94 & 70.18 \\
        \midrule
        \midrule
        \multirow{7}{*}{Qwen3-14B}
        & HMMT2025 & 49.32 & 61.67 & 62.14 & 52.92 & 63.23 & 63.33 & 63.70 & 64.43 & 63.75 & 63.96 & 64.74 \\
        & GPQA-D & 63.68 & 65.25 & 65.52 & 66.75 & 66.83 & 66.60 & 66.99 & 67.32 & 66.86 & 67.08 & 67.26 \\
        & AIME2024 & 78.59 & 85.00 & 85.42 & 85.99 & 86.20 & 84.74 & 87.76 & 85.68 & 88.39 & 88.28 & 86.82 \\
        & AIME2025 & 68.91 & 76.72 & 76.67 & 72.40 & 77.08 & 77.03 & 77.50 & 77.50 & 77.14 & 77.34 & 77.45 \\
        & BRUMO2025 & 75.16 & 80.52 & 80.83 & 83.70 & 86.41 & 85.26 & 86.25 & 85.31 & 86.41 & 86.61 & 85.89 \\
        \cmidrule(lr){2-13}
        & Avg. & 65.31 & 69.30 & 69.58 & 69.39 & 71.13 & 70.74 & 71.44 & 71.45 & 71.41 & 71.59 & 71.60 \\
        \bottomrule
      \end{tabular}
    }
  \end{center}
\end{table*}
    \begin{table*}[ht]
    \caption{All main results of \textit{SelfStepConf} (SSC) and \textit{DistriVoting} on various benchmarks using \textbf{NonThinking Mode} of Qwen3-Series models. Set $\text{Budget}$ to 128 and repeat 64 times. SC denotes Self-Consistency (i.e., Majority Voting), WSC represents Weighted SC (i.e., Weighted Majority Voting), BoN represents Best of N, and \textbf{* indicates answers generated using the SSC approach}.}
    \label{tab:appendix_complete_main_results_nonthinking}
    \begin{center}
    \resizebox{0.95\textwidth}{!}{
      \begin{tabular}{llccccccccccc}
        \toprule
        \multirow{2}{*}{Model} & \multirow{2}{*}{Benchmark} & \multirow{2}{*}{Pass@1} & \multirow{2}{*}{SC} & \multirow{2}{*}{WSC} & \multirow{2}{*}{BoN} & \multirow{2}{*}{MoB} & \multicolumn{3}{c}{WSC} & \multicolumn{3}{c}{DIS} \\
        \cmidrule(lr){8-10} \cmidrule(lr){11-13}
        & & & & & & & Top50 & GMM & GMM* & Top50 & GMM & GMM* \\
        \midrule
        \multirow{7}{*}{Qwen3-0.6B}
        & HMMT2025 & 0.97 & 3.34 & 3.83 & 3.50 & 6.91 & 6.95 & 4.94 & 6.94 & 5.27 & 6.90 & 5.70 \\
        & GPQA-D & 23.14 & 23.39 & 23.17 & 23.79 & 22.12 & 20.45 & 23.92 & 24.09 & 24.24 & 24.32 & 24.94 \\
        & AIME2024 & 2.66 & 3.33 & 3.33 & 3.39 & 3.34 & 3.41 & 4.67 & 6.77 & 6.62 & 3.33 & 5.36 \\
        & AIME2025 & 2.45 & 6.67 & 6.67 & 3.44 & 7.83 & 6.61 & 7.51 & 9.96 & 10.03 & 9.69 & 8.40 \\
        & BRUMO2025 & 8.12 & 16.51 & 16.67 & 16.46 & 16.59 & 17.32 & 20.53 & 21.34 & 16.96 & 19.74 & 20.06 \\
        \cmidrule(lr){2-13}
        & Avg. & 15.75 & 17.38 & 17.30 & 17.34 & 17.05 & 15.97 & 18.45 & 19.25 & 18.76 & 18.88 & 19.26 \\
        \midrule
        \midrule
        \multirow{7}{*}{Qwen3-1.7B}
        & HMMT2025 & 5.68 & 6.67 & 6.67 & 6.56 & 10.40 & 11.34 & 10.08 & 13.13 & 11.80 & 10.23 & 11.92 \\
        & GPQA-D & 28.23 & 37.19 & 36.34 & 31.83 & 37.78 & 33.06 & 37.56 & 37.82 & 32.63 & 37.95 & 38.73 \\
        & AIME2024 & 12.92 & 23.49 & 23.65 & 10.10 & 25.73 & 26.84 & 29.87 & 32.97 & 26.85 & 30.26 & 32.97 \\
        & AIME2025 & 9.32 & 16.67 & 16.67 & 23.13 & 16.81 & 17.75 & 16.80 & 16.67 & 20.30 & 16.69 & 16.80 \\
        & BRUMO2025 & 17.40 & 23.33 & 23.33 & 19.95 & 25.26 & 30.53 & 25.62 & 26.41 & 31.09 & 26.61 & 25.36 \\
        \cmidrule(lr){2-13}
        & Avg. & 21.85 & 29.78 & 29.26 & 25.45 & 30.90 & 28.74 & 31.15 & 31.96 & 28.81 & 31.53 & 32.33 \\
        \midrule
        \midrule
        \multirow{7}{*}{Qwen3-4B}
        & HMMT2025 & 11.67 & 16.67 & 16.67 & 19.84 & 16.67 & 16.80 & 16.91 & 17.23 & 16.69 & 16.80 & 16.47 \\
        & GPQA-D & 41.26 & 44.79 & 46.69 & 46.48 & 47.17 & 45.35 & 47.85 & 48.36 & 46.97 & 48.13 & 48.46 \\
        & AIME2024 & 22.45 & 36.41 & 37.03 & 30.10 & 40.26 & 46.78 & 45.17 & 37.61 & 43.41 & 47.03 & 41.43 \\
        & AIME2025 & 18.39 & 20.36 & 23.18 & 23.28 & 33.28 & 33.39 & 33.14 & 31.20 & 33.23 & 32.44 & 35.40 \\
        & BRUMO2025 & 27.60 & 36.72 & 37.19 & 29.95 & 39.06 & 41.95 & 40.60 & 51.05 & 39.70 & 41.73 & 51.29 \\
        \cmidrule(lr){2-13}
        & Avg. & 33.25 & 38.28 & 39.83 & 38.67 & 41.57 & 41.34 & 42.60 & 43.04 & 41.80 & 42.98 & 43.81 \\
        \midrule
        \midrule
        \multirow{7}{*}{Qwen3-8B}
        & HMMT2025 & 10.36 & 13.33 & 13.33 & 13.44 & 16.85 & 16.75 & 14.64 & 19.70 & 13.62 & 22.60 & 16.48 \\
        & GPQA-D & 46.32 & 56.00 & 55.97 & 45.53 & 54.71 & 54.53 & 55.29 & 54.64 & 55.43 & 55.59 & 57.08 \\
        & AIME2024 & 28.23 & 46.51 & 46.30 & 23.44 & 43.68 & 40.23 & 40.16 & 44.70 & 43.51 & 48.51 & 44.25 \\
        & AIME2025 & 20.16 & 30.00 & 30.00 & 23.39 & 27.77 & 30.67 & 33.49 & 35.35 & 32.92 & 34.61 & 33.39 \\
        & BRUMO2025 & 28.02 & 41.35 & 43.13 & 33.39 & 45.15 & 46.37 & 50.00 & 43.54 & 49.93 & 39.49 & 44.99 \\
        \cmidrule(lr){2-13}
        & Avg. & 37.03 & 47.24 & 47.38 & 37.18 & 46.65 & 46.60 & 47.47 & 47.54 & 47.72 & 48.31 & 48.66 \\
        \midrule
        \midrule
        \multirow{7}{*}{Qwen3-14B}
        & HMMT2025 & 9.81 & 21.82 & 20.26 & 20.53 & 20.69 & 25.61 & 22.16 & 25.54 & 25.84 & 26.51 & 21.94 \\
        & GPQA-D & 52.38 & 55.47 & 58.51 & 54.06 & 58.07 & 56.04 & 59.45 & 56.28 & 55.98 & 56.45 & 60.71 \\
        & AIME2024 & 27.45 & 47.34 & 40.00 & 29.95 & 40.10 & 48.91 & 40.31 & 50.68 & 48.96 & 50.05 & 43.63 \\
        & AIME2025 & 21.09 & 35.51 & 33.07 & 30.10 & 33.33 & 38.58 & 34.17 & 39.60 & 37.82 & 38.31 & 35.49 \\
        & BRUMO2025 & 31.67 & 49.79 & 43.70 & 36.51 & 49.58 & 49.32 & 46.51 & 52.10 & 50.26 & 51.89 & 46.94 \\
        \cmidrule(lr){2-13}
        & Avg. & 41.11 & 49.11 & 49.36 & 44.71 & 49.71 & 50.22 & 50.52 & 50.88 & 50.22 & 50.88 & 51.77 \\
        \midrule
        \midrule
        \multirow{7}{*}{Qwen3-32B}
        & HMMT2025 & 12.18 & 19.34 & 24.14 & 21.37 & 25.50 & 20.72 & 25.39 & 20.18 & 24.28 & 17.17 & 25.84 \\
        & GPQA-D & 51.55 & 58.74 & 57.07 & 52.92 & 56.57 & 59.07 & 56.45 & 60.36 & 58.93 & 60.84 & 56.92 \\
        & AIME2024 & 30.05 & 40.00 & 46.67 & 34.22 & 50.16 & 40.07 & 49.64 & 44.63 & 40.91 & 49.63 & 50.89 \\
        & AIME2025 & 24.95 & 31.04 & 36.30 & 35.94 & 36.64 & 33.39 & 39.12 & 33.51 & 34.98 & 33.09 & 46.61 \\
        & BRUMO2025 & 35.73 & 43.49 & 50.00 & 34.74 & 49.64 & 49.51 & 50.63 & 43.68 & 46.22 & 45.80 & 52.24 \\
        \cmidrule(lr){2-13}
        & Avg. & 41.81 & 49.21 & 50.36 & 44.86 & 50.50 & 50.34 & 50.69 & 50.98 & 50.50 & 51.62 & 52.00 \\
        \bottomrule
      \end{tabular}
    }
    \end{center}
    \vskip -0.1in
\end{table*}
    
    \newpage
    \paragraph{Methods Ablation.} We conducted detailed ablation experiments on the key components of the proposed methods: \textit{SelfStepConf} and \textit{DistriVoting} (\textit{GMM Filter} + \textit{Reject Filter} + \textit{ \textit{HierVoting}}). As shown in ~\autoref{tab:appendix_complete_main_results_ablation} and ~\autoref{tab:appendix_complete_main_results_ablation_qwen3_32B}, \textit{SelfStepConf} is marked with \ding{55} for basic inference and \ding{51} for SSC inference. \textit{GMM Filter} is marked with \ding{51} for using the \textit{GMM Filter} and \ding{55} for using the naive top50 filter, and ``--'' indicates no filtering, using all trajectories for voting. \textit{Reject Filter} is marked with \ding{55} and \ding{51} for performing and not performing the \textit{Reject Filter}, respectively.  \textit{HierVoting} is marked with \ding{55} and \ding{51} for using Weighted-SC Voting and \textit{HierVoting}, respectively. From the results, it is evident that SSC inference consistently outperforms basic inference under any setting, indicating that SSC provides a consistent improvement in voting performance. Additionally, among the three components of \textit{DistriVoting}, the \textit{GMM Filter} is crucial in affecting performance, as its presence significantly impacts voting results. The benefit of the \textit{Reject Filter} on voting performance relies on the foundation of the \textit{GMM Filter}, meaning that only after a good distribution split can the positive and negative results of the distribution be reliably rejected. In contrast,  \textit{HierVoting} improves voting performance even without the \textit{GMM Filter}, but its impact is not significant with the addition of the \textit{GMM Filter} and \textit{Reject Filter}. This is primarily due to the diminishing marginal returns of performance enhancement, where the advantages of more complex voting strategies are mainly evident on a weaker filtering basis, aligning with our analysis of \autoref{fig:appendix_complete_top50_analysis} in \autoref{sec:appendix_parameters_analysis_Nc}.
    
    \begin{table}[ht]
    \vskip 0.3in
    \caption{Complete ablation experiments on the method components in the main experiments using \textbf{DeepSeek-R1-8B}, sampling 128 trajectories/query (64 repeats). The naive version of the \textit{GMM Filter} is the Top50 Filter, the naive version of \textit{HierVoting} is Weighted-SC, and the naive version of the \textit{Reject Filter} is no filtering. -- indicates no filtering at all, i.e., directly using all trajectories for voting.}
    \label{tab:appendix_complete_main_results_ablation}
    \begin{center}
    \resizebox{0.98\textwidth}{!}{
    \begin{tabular}{cccccccccc}
        \toprule
        \multirow{2}{*}{SelfStepConf} & \multicolumn{3}{c}{DistriVoting} & \multicolumn{6}{c}{Voting Acc (\%)}  \\
        \cmidrule(lr){2-4} \cmidrule(lr){5-10}
        & \textit{GMM Filter} & \textit{Reject Filter} & \textit{HierVoting} & HMMT2025 & GPQA-D & AIME2024 & AIME2025 & BRUMO2025 & Avg.\\
        \midrule
        \multirow{11}{*}{\ding{55}} 
        & {\ding{55}} & {\ding{55}} & {\ding{55}} & 73.80$\pm$0.44 & 68.58$\pm$0.04 & 90.00$\pm$0.18 & 82.55$\pm$0.34 & 93.33$\pm$0.00 & 74.75$\pm$0.07 \\
        & {\ding{51}} & {\ding{55}} & {\ding{55}} & 82.50$\pm$0.49 & 69.58$\pm$0.13 & 93.13$\pm$0.15 & 83.91$\pm$0.99 & 93.59$\pm$0.08 & 76.64$\pm$0.11 \\
        & {\ding{55}} & {\ding{51}} & {\ding{55}} & 73.91$\pm$0.39 & 68.47$\pm$0.15 & 89.90$\pm$0.08 & 82.86$\pm$0.31 & 93.33$\pm$0.00 & 74.71$\pm$0.11 \\
        & {\ding{55}} & {\ding{55}} & {\ding{51}} & 79.95$\pm$0.23 & 69.36$\pm$0.07 & 93.07$\pm$0.05 & 84.38$\pm$0.16 & 93.44$\pm$0.03 & 76.28$\pm$0.08 \\
        & {\ding{51}} & {\ding{51}} & {\ding{55}} & 83.07$\pm$0.26 & 69.31$\pm$0.23 & 93.28$\pm$0.05 & 86.77$\pm$0.29 & 93.75$\pm$0.13 & 76.82$\pm$0.14 \\
        & {\ding{51}} & {\ding{55}} & {\ding{51}} & 82.03$\pm$0.68 & 69.67$\pm$0.13 & 93.28$\pm$0.03 & 85.89$\pm$0.31 & 93.85$\pm$0.08 & 76.87$\pm$0.18 \\
        & {\ding{55}} & {\ding{51}} & {\ding{51}} & 79.27$\pm$0.65 & 69.52$\pm$0.19 & 93.18$\pm$0.13 & 84.43$\pm$0.23 & 93.28$\pm$0.18 & 76.32$\pm$0.05 \\
        & {\ding{51}} & {\ding{51}} & {\ding{51}} & 82.55$\pm$0.31 & 69.82$\pm$0.09 & 93.13$\pm$0.18 & 85.52$\pm$0.60 & 93.70$\pm$0.10 & 76.95$\pm$0.10 \\
        \cmidrule{2-10}
        & --          & --          & {\ding{55}} & 69.69$\pm$0.18 & 67.65$\pm$0.04 & 86.67$\pm$0.00 & 80.78$\pm$0.13 & 93.33$\pm$0.00 & 73.30$\pm$0.03 \\
        & --          & --          & {\ding{51}} & 76.67$\pm$0.24 & 68.26$\pm$0.07 & 92.92$\pm$0.13 & 84.58$\pm$0.45 & 93.33$\pm$0.05 & 75.28$\pm$0.05 \\
        \midrule
        \midrule
        \multirow{11}{*}{\ding{51}} 
        & {\ding{55}} & {\ding{55}} & {\ding{55}} & 75.31$\pm$0.00 & 69.11$\pm$0.28 & 89.69$\pm$0.19 & 80.00$\pm$0.10 & 93.33$\pm$0.00 & 74.95$\pm$0.16 \\
        & {\ding{51}} & {\ding{55}} & {\ding{55}} & 84.17$\pm$0.54 & 70.11$\pm$0.07 & 93.33$\pm$0.36 & 84.38$\pm$0.59 & 94.17$\pm$0.44 & 77.24$\pm$0.14 \\
        & {\ding{55}} & {\ding{51}} & {\ding{55}} & 75.36$\pm$0.28 & 69.07$\pm$0.15 & 89.69$\pm$0.08 & 80.05$\pm$0.18 & 93.33$\pm$0.00 & 74.94$\pm$0.13 \\
        & {\ding{55}} & {\ding{55}} & {\ding{51}} & 81.30$\pm$0.50 & 70.08$\pm$0.12 & 93.13$\pm$0.24 & 82.76$\pm$0.21 & 93.39$\pm$0.03 & 76.71$\pm$0.11 \\
        & {\ding{51}} & {\ding{51}} & {\ding{55}} & 85.21$\pm$0.52 & 70.16$\pm$0.04 & 93.18$\pm$0.05 & 84.90$\pm$0.23 & 94.74$\pm$0.16 & 77.46$\pm$0.07 \\
        & {\ding{51}} & {\ding{55}} & {\ding{51}} & 83.75$\pm$0.60 & 70.11$\pm$0.48 & 93.13$\pm$0.03 & 84.74$\pm$0.31 & 94.64$\pm$0.11 & 77.26$\pm$0.29 \\
        & {\ding{55}} & {\ding{51}} & {\ding{51}} & 81.20$\pm$0.52 & 70.13$\pm$0.17 & 93.13$\pm$0.24 & 82.71$\pm$0.21 & 93.33$\pm$0.03 & 76.72$\pm$0.13 \\
        & {\ding{51}} & {\ding{51}} & {\ding{51}} & 84.95$\pm$0.86 & 70.63$\pm$0.17 & 93.23$\pm$0.05 & 86.46$\pm$0.55 & 94.27$\pm$0.17 & 77.84$\pm$0.28 \\
        \cmidrule{2-10}
        & --          & --          & {\ding{55}} & 70.57$\pm$0.13 & 67.76$\pm$0.02 & 86.67$\pm$0.00 & 79.32$\pm$0.21 & 93.18$\pm$0.16 & 73.30$\pm$0.03 \\
        & --          & --          & {\ding{51}} & 78.96$\pm$0.52 & 68.81$\pm$0.11 & 92.76$\pm$0.16 & 82.86$\pm$0.21 & 93.23$\pm$0.03 & 75.66$\pm$0.08 \\
      \bottomrule
    \end{tabular}
    }
    \end{center}
    \vskip 0.1in
\end{table}

\begin{table}[H]
  \caption{Complete ablation experiments on the method components in the main experiments using \textbf{Qwen3-32B}, sampling 128 trajectories/query (64 repeats). The naive version of the \textit{GMM Filter} is the Top50 Filter, the naive version of \textit{HierVoting} is Weighted-SC, and the naive version of the \textit{Reject Filter} is no filtering. -- indicates no filtering at all, i.e., directly using all trajectories for voting.}
  \label{tab:appendix_complete_main_results_ablation_qwen3_32B}
  \begin{center}
    \resizebox{0.98\textwidth}{!}{
    \begin{tabular}{cccccccccc}
    \toprule
        \multirow{2}{*}{SelfStepConf} & \multicolumn{3}{c}{DistriVoting} & \multicolumn{6}{c}{Voting Acc (\%)}  \\
        \cmidrule(lr){2-4} \cmidrule(lr){5-10}
        & \textit{GMM Filter} & \textit{Reject Filter} & \textit{HierVoting} & HMMT2025 & GPQA-D & AIME2024 & AIME2025 & BRUMO2025 & Avg.\\
        \midrule
        \multirow{11}{*}{\ding{55}} 
        & {\ding{55}} & {\ding{55}} & {\ding{55}} & 63.96$\pm$0.29 & 72.03$\pm$0.17 & 88.28$\pm$0.31 & 76.93$\pm$0.26 & 92.71$\pm$0.00 & 75.22$\pm$0.12 \\
        & {\ding{51}} & {\ding{55}} & {\ding{55}} & 64.84$\pm$0.23 & 72.42$\pm$0.11 & 88.65$\pm$0.31 & 79.22$\pm$0.34 & 92.71$\pm$0.03 & 75.79$\pm$0.10 \\
        & {\ding{55}} & {\ding{51}} & {\ding{55}} & 63.85$\pm$0.23 & 72.01$\pm$0.15 & 88.28$\pm$0.36 & 76.98$\pm$0.29 & 93.33$\pm$0.00 & 75.26$\pm$0.10 \\
        & {\ding{55}} & {\ding{55}} & {\ding{51}} & 64.79$\pm$0.36 & 72.29$\pm$0.25 & 88.70$\pm$0.49 & 79.11$\pm$0.31 & 93.28$\pm$0.05 & 75.76$\pm$0.18 \\
        & {\ding{51}} & {\ding{51}} & {\ding{55}} & 65.00$\pm$0.21 & 72.51$\pm$0.16 & 88.54$\pm$0.44 & 78.96$\pm$0.55 & 93.33$\pm$0.00 & 75.88$\pm$0.09 \\
        & {\ding{51}} & {\ding{55}} & {\ding{51}} & 64.48$\pm$0.26 & 72.75$\pm$0.27 & 89.11$\pm$0.36 & 78.75$\pm$0.18 & 93.23$\pm$0.03 & 76.01$\pm$0.15 \\
        & {\ding{55}} & {\ding{51}} & {\ding{51}} & 64.95$\pm$0.24 & 72.33$\pm$0.33 & 88.70$\pm$0.52 & 79.06$\pm$0.34 & 93.28$\pm$0.03 & 75.79$\pm$0.20 \\
        & {\ding{51}} & {\ding{51}} & {\ding{51}} & 64.43$\pm$0.21 & 72.74$\pm$0.25 & 89.01$\pm$0.39 & 78.70$\pm$0.23 & 93.23$\pm$0.08 & 75.99$\pm$0.09 \\
        \cmidrule{2-10}
        & --          & --          & {\ding{55}} & 62.24$\pm$0.39 & 70.55$\pm$0.14 & 86.88$\pm$0.67 & 77.08$\pm$0.10 & 93.33$\pm$0.00 & 74.07$\pm$0.16 \\
        & --          & --          & {\ding{51}} & 62.66$\pm$0.42 & 70.34$\pm$0.15 & 89.32$\pm$0.16 & 80.42$\pm$0.18 & 93.18$\pm$0.08 & 74.51$\pm$0.12 \\
        \midrule
        \midrule
        \multirow{11}{*}{\ding{51}} 
        & {\ding{55}} & {\ding{55}} & {\ding{55}} & 65.63$\pm$0.49 & 72.53$\pm$0.12 & 90.42$\pm$0.36 & 77.86$\pm$0.33 & 93.33$\pm$0.00 & 76.03$\pm$0.12 \\
        & {\ding{51}} & {\ding{55}} & {\ding{55}} & 65.21$\pm$0.57 & 72.43$\pm$0.25 & 90.16$\pm$0.21 & 80.00$\pm$0.42 & 93.28$\pm$0.18 & 76.10$\pm$0.18 \\
        & {\ding{55}} & {\ding{51}} & {\ding{55}} & 65.63$\pm$0.55 & 72.53$\pm$0.09 & 90.42$\pm$0.34 & 77.76$\pm$0.21 & 93.33$\pm$0.00 & 76.02$\pm$0.09 \\
        & {\ding{55}} & {\ding{55}} & {\ding{51}} & 64.74$\pm$0.42 & 72.99$\pm$0.16 & 89.27$\pm$0.47 & 79.38$\pm$0.55 & 93.33$\pm$0.15 & 76.27$\pm$0.16 \\
        & {\ding{51}} & {\ding{51}} & {\ding{55}} & 65.16$\pm$0.55 & 72.48$\pm$0.28 & 90.05$\pm$0.26 & 79.90$\pm$0.60 & 92.71$\pm$0.15 & 76.05$\pm$0.21 \\
        & {\ding{51}} & {\ding{55}} & {\ding{51}} & 63.23$\pm$0.47 & 73.13$\pm$0.06 & 89.22$\pm$0.57 & 80.31$\pm$0.60 & 92.66$\pm$0.06 & 76.23$\pm$0.09 \\
        & {\ding{55}} & {\ding{51}} & {\ding{51}} & 64.64$\pm$0.29 & 73.13$\pm$0.15 & 89.38$\pm$0.52 & 79.17$\pm$0.44 & 93.28$\pm$0.05 & 76.33$\pm$0.17 \\
        & {\ding{51}} & {\ding{51}} & {\ding{51}} & 65.73$\pm$0.05 & 73.18$\pm$0.02 & 89.11$\pm$0.50 & 80.05$\pm$0.44 & 93.33$\pm$0.08 & 76.53$\pm$0.08 \\
        \cmidrule{2-10}
        & --          & --          & {\ding{55}} & 64.11$\pm$0.18 & 71.43$\pm$0.09 & 87.45$\pm$0.52 & 77.55$\pm$0.52 & 93.33$\pm$0.00 & 74.90$\pm$0.03 \\
        & --          & --          & {\ding{51}} & 65.00$\pm$0.31 & 71.97$\pm$0.19 & 88.39$\pm$0.39 & 79.38$\pm$0.47 & 93.18$\pm$0.08 & 75.56$\pm$0.09 \\
      \bottomrule
    \end{tabular}
    }
  \end{center}
\end{table}

\newpage
\subsection{Detailed Ablation Results of Distribution Splitting Methods}
\label{sec:appendix_complete_gmm_ablation}
    For the average ablation experiment results of different distribution splitting methods on five benchmarks in \autoref{tab:ablation_gmm} of \autoref{sec:experiment_ablation}, we supplement the detailed results for each benchmark here. As shown in \autoref{tab:appendix_complete_gmm_ablation}, Acc and WAcc are defined in \autoref{sec:experiment_ablation}, and the closer the AUROC is to 1, the higher the quality of confidence in the positive interval after distribution splitting. Voting Acc is calculated based on \textit{DistriVoting}, Predict Acc represents the binary classification accuracy of the distribution splitting method, and Predict Time indicates the average runtime of different clustering methods per trajectory.

    \begin{table}[ht]
    \vskip 0.2in
    \caption{Complete ablation study results of different correctness prediction methods. Using DeepSeek-R1-8B to generate 128 trajectories for each question across 5 benchmarks with \textit{DistriVoting}, repeated 64 times.}
    \label{tab:appendix_complete_gmm_ablation}
    \begin{center}
    \resizebox{0.9\textwidth}{!}{
    \begin{tabular}{llcccccc}
      \toprule
      Method    & Metric & HMMT2025 & GPQA-D & AIME2024 & AIME2025 & BRUMO2025 & Avg. \\
      \midrule
      \multirow{7}{*}{Top50}
      & Acc (\%)                & 71.94  & 68.65  & 91.24  & 84.40  & 88.63  & 74.47  \\
      & WAcc (\%)               & 72.39  & 68.71  & 91.41  & 84.57  & 89.02  & 74.61  \\
      & AUROC ($\uparrow$)      & 0.5234 & 0.4921 & 0.5358 & 0.5129 & 0.6037 & 0.5117 \\
      \cmidrule{2-8}
      & Voting Acc (\%)         & 78.28  & 67.85  & 93.18  & 83.44  & 93.33  & 75.10  \\
      & Predict Acc (\%)        & 56.77  & 52.66  & 53.62  & 53.98  & 56.68  & 53.64  \\
      & Predict Time (ms/it)    & 0.0300 & 0.1110 & 0.0961 & 0.0884 & 0.0443 & 0.0935 \\
      \midrule
      \midrule
      \multirow{7}{*}{K-Means}
      & Acc (\%)                & 76.59  & 68.07  & 94.00  & 87.50  & 90.76  & 75.29  \\
      & WAcc (\%)               & 76.64  & 68.07  & 94.04  & 87.49  & 90.88  & 75.32  \\
      & AUROC ($\uparrow$)      & 0.5306 & 0.5047 & 0.5478 & 0.4827	& 0.6242 & 0.5204 \\
      \cmidrule{2-8}
      & Voting Acc (\%)         & 82.03  & 67.12  & 93.33  & 85.16  & 93.54  & 75.19  \\
      & Predict Acc (\%)        & 71.12  & 51.00  & 62.79  & 62.51  & 62.66  & 56.19  \\
      & Predict Time (ms/it)    & 0.6490 & 0.6080 & 0.5560 & 0.5690 & 0.5880 & 0.6014 \\
      \midrule
      \midrule
      \multirow{7}{*}{MeanShift} 
      & Acc (\%)                & 76.93  & 69.82  & 94.40  & 88.03  & 91.33  & 76.55  \\
      & WAcc (\%)               & 76.95  & 69.82  & 94.41  & 88.00  & 91.43  & 76.57  \\
      & AUROC ($\uparrow$)      & 0.5125 & 0.5047 & 0.5383 & 0.4806 & 0.6051 & 0.5158 \\
      \cmidrule{2-8}
      & Voting Acc (\%)         & 82.71  & 67.40  & 93.28  & 85.47  & 93.96  & 75.50  \\
      & Predict Acc (\%)        & 76.08  & 53.01  & 68.53  & 68.44  & 66.08  & 59.34  \\
      & Predict Time (ms/it)    & 1.8210 & 1.8950 & 1.8030 & 1.7700 & 1.7000 & 1.8492 \\ 
      \midrule
      \midrule
      \multirow{7}{*}{GMM}
      & Acc (\%)                & 76.71  & 71.69  & 93.90  & 87.79  & 91.06  & 77.60  \\
      & WAcc (\%)               & 76.92  & 71.72  & 94.00  & 87.86  & 91.24  & 77.68  \\
      & AUROC ($\uparrow$)      & 0.6032 & 0.5605 & 0.6441 & 0.5569 & 0.6776 & 0.5831 \\
      \cmidrule{2-8}
      & Voting Acc (\%)         & 82.55  & 69.82 & 93.13  & 85.52  & 93.70   & 76.95  \\
      & Predict Acc (\%)        & 72.33  & 56.38 & 66.20  & 65.11  & 65.19   & 60.46  \\
      & Predict Time (ms/it)    & 0.2730 & 0.354 & 0.3300 & 0.3350 & 0.2970  & 0.3369 \\ 
      \bottomrule
    \end{tabular}
    }
    \end{center}
\end{table}

\clearpage
\newpage
\subsection{Detailed Ablation Results of Voting Budget}
\label{sec:appendix_complete_budget_ablation}
    For the average ablation results under different budgets on five benchmarks in \autoref{tab:ablation_budget} of \autoref{sec:experiment_ablation} in the main text, we provide detailed results for each benchmark here, as shown in \autoref{tab:appendix_complete_budget_ablation}.
    \begin{table}[ht]
    \vskip 0.2in
    \caption{Complete ablation study results of different Budget. Using DeepSeek-R1-8B to generate $\text{Budget}$ trajectories for each question across 5 benchmarks with \textit{DistriVoting}, repeated 64 times. * indicates answers generated using the \textit{SelfStepConf} approach.}
    \label{tab:appendix_complete_budget_ablation}
    \begin{center}
    \resizebox{0.88\textwidth}{!}{
    \begin{tabular}{llcccccc}
    \toprule
      Budget & Method & HMMT2025 & GPQA-D & AIME2024 & AIME2025 & BRUMO2025 & Avg. \\
      \midrule
      \multirow{7}{*}{8}
      & WSC-Top50 & 71.09$\pm$1.30 & 67.37$\pm$0.87 & 89.06$\pm$0.34 & 80.52$\pm$0.16 & 90.31$\pm$0.34 & 73.17$\pm$0.51\\
      & WSC-GMM   & 72.71$\pm$1.98 & 66.99$\pm$0.26 & 89.53$\pm$0.52 & 81.25$\pm$0.49 & 90.10$\pm$0.52 & 73.18$\pm$0.25\\
      & WSC-GMM*  & 73.07$\pm$0.65 & 67.18$\pm$0.89 & 89.38$\pm$0.42 & 80.31$\pm$0.42 & 90.00$\pm$0.49 & 73.22$\pm$0.56\\
      \cmidrule{2-8}
      & DIS-Top50 & 70.78$\pm$1.30 & 67.14$\pm$0.69 & 88.75$\pm$0.39 & 80.47$\pm$0.00 & 90.16$\pm$0.34 & 72.95$\pm$0.31\\
      & DIS-GMM   & 72.71$\pm$2.11 & 66.94$\pm$0.41 & 89.27$\pm$0.57 & 81.20$\pm$0.37 & 90.05$\pm$0.42 & 73.12$\pm$0.24\\
      & DIS-GMM*  & 72.97$\pm$0.60 & 67.16$\pm$1.22 & 89.32$\pm$0.44 & 80.21$\pm$0.34 & 89.90$\pm$0.50 & 73.18$\pm$0.78\\
      \midrule
      \midrule
      \multirow{7}{*}{16}
      & WSC-Top50 & 72.55$\pm$0.70 & 67.86$\pm$0.10 & 89.48$\pm$0.13 & 81.25$\pm$0.76 & 91.77$\pm$0.13 & 73.86$\pm$0.09\\
      & WSC-GMM   & 75.78$\pm$1.20 & 67.89$\pm$0.20 & 90.42$\pm$0.03 & 83.23$\pm$0.86 & 92.55$\pm$0.44 & 74.53$\pm$0.17\\
      & WSC-GMM*  & 76.72$\pm$0.86 & 68.05$\pm$0.01 & 91.25$\pm$0.52 & 82.14$\pm$0.24 & 92.34$\pm$0.15 & 74.68$\pm$0.07\\
      \cmidrule{2-8}
      & DIS-Top50 & 72.08$\pm$1.22 & 67.83$\pm$0.26 & 89.22$\pm$0.08 & 80.94$\pm$0.96 & 91.51$\pm$0.42 & 73.72$\pm$0.16\\
      & DIS-GMM   & 74.58$\pm$0.68 & 68.35$\pm$0.67 & 91.09$\pm$0.28 & 82.97$\pm$0.83 & 92.40$\pm$0.55 & 74.73$\pm$0.39\\
      & DIS-GMM*  & 75.94$\pm$0.52 & 68.19$\pm$0.04 & 90.78$\pm$0.08 & 82.45$\pm$0.47 & 93.07$\pm$1.32 & 74.74$\pm$0.13\\
      \midrule
      \midrule
      \multirow{7}{*}{32}
      & WSC-Top50 & 74.01$\pm$0.83 & 68.49$\pm$0.28 & 89.69$\pm$0.55 & 81.51$\pm$0.34 & 93.07$\pm$0.23 & 74.56$\pm$0.24\\
      & WSC-GMM   & 79.95$\pm$1.56 & 68.68$\pm$0.28 & 92.40$\pm$0.55 & 84.53$\pm$0.68 & 93.65$\pm$0.29 & 75.83$\pm$1.94\\
      & WSC-GMM*  & 78.23$\pm$0.26 & 69.13$\pm$0.31 & 92.03$\pm$0.21 & 83.49$\pm$1.09 & 93.85$\pm$0.13 & 75.84$\pm$2.22\\
      \cmidrule{2-8}
      & DIS-Top50 & 73.39$\pm$0.89 & 68.47$\pm$0.23 & 90.89$\pm$0.42 & 82.03$\pm$1.04 & 92.92$\pm$0.34 & 74.64$\pm$0.16\\
      & DIS-GMM   & 79.43$\pm$1.59 & 68.69$\pm$0.47 & 92.34$\pm$0.49 & 83.80$\pm$0.47 & 93.59$\pm$0.34 & 75.71$\pm$0.53\\
      & DIS-GMM*  & 80.10$\pm$1.43 & 68.96$\pm$0.27 & 91.88$\pm$0.11 & 83.65$\pm$1.31 & 93.85$\pm$0.38 & 75.90$\pm$0.23\\
      \midrule
      \midrule
      \multirow{7}{*}{64}
      & WSC-Top50 & 74.48$\pm$0.63 & 68.43$\pm$0.16 & 89.64$\pm$0.26 & 81.98$\pm$0.42 & 93.28$\pm$0.08 & 74.62$\pm$0.13\\
      & WSC-GMM   & 81.20$\pm$0.52 & 69.02$\pm$0.15 & 92.97$\pm$0.21 & 85.31$\pm$0.68 & 93.75$\pm$0.26 & 76.30$\pm$0.15\\
      & WSC-GMM*  & 82.03$\pm$0.65 & 69.82$\pm$0.12 & 92.76$\pm$0.39 & 83.96$\pm$0.26 & 94.32$\pm$0.36 & 76.78$\pm$0.05\\
      \cmidrule{2-8}
      & DIS-Top50 & 77.45$\pm$0.78 & 69.06$\pm$0.09 & 92.50$\pm$0.44 & 83.54$\pm$0.44 & 93.23$\pm$0.03 & 75.71$\pm$0.16\\
      & DIS-GMM   & 80.94$\pm$1.46 & 69.14$\pm$0.44 & 92.97$\pm$0.16 & 85.21$\pm$0.94 & 93.75$\pm$0.36 & 76.34$\pm$0.39\\
      & DIS-GMM*  & 81.30$\pm$0.49 & 70.13$\pm$0.50 & 92.76$\pm$0.21 & 83.70$\pm$0.13 & 94.17$\pm$0.16 & 76.87$\pm$0.18\\
      \midrule
      \midrule
      \multirow{7}{*}{128}
      & WSC-Top50 & 73.80$\pm$0.44 & 68.58$\pm$0.04 & 90.00$\pm$0.18 & 82.55$\pm$0.34 & 93.33$\pm$0.00 & 74.75$\pm$0.07 \\
      & WSC-GMM   & 82.50$\pm$0.49 & 69.58$\pm$0.13 & 93.13$\pm$0.15 & 83.91$\pm$0.99 & 93.59$\pm$0.08 & 76.64$\pm$0.11 \\
      & WSC-GMM*  & 84.17$\pm$0.54 & 70.11$\pm$0.07 & 93.33$\pm$0.36 & 84.38$\pm$0.59 & 94.17$\pm$0.44 & 77.24$\pm$0.14 \\
      \cmidrule{2-8}
      & DIS-Top50 & 79.27$\pm$0.65 & 69.52$\pm$0.19 & 93.18$\pm$0.13 & 84.43$\pm$0.23 & 93.28$\pm$0.18 & 76.32$\pm$0.05 \\
      & DIS-GMM   & 82.55$\pm$0.31 & 69.82$\pm$0.09 & 93.13$\pm$0.18 & 85.52$\pm$0.60 & 93.70$\pm$0.10 & 76.95$\pm$0.10 \\
      & DIS-GMM*  & 84.95$\pm$0.86 & 70.63$\pm$0.17 & 93.23$\pm$0.05 & 86.46$\pm$0.55 & 94.27$\pm$0.17 & 77.84$\pm$0.28 \\
      \midrule
      \midrule
      \multirow{7}{*}{256}
      & WSC-Top50 & 73.96$\pm$0.34 & 68.49$\pm$0.11 & 89.79$\pm$0.00 & 83.65$\pm$0.78 & 93.33$\pm$0.00 & 74.79$\pm$0.04\\
      & WSC-GMM   & 83.65$\pm$0.34 & 69.59$\pm$0.25 & 93.33$\pm$0.00 & 86.82$\pm$0.57 & 93.54$\pm$0.21 & 77.04$\pm$0.06\\
      & WSC-GMM*  & 86.25$\pm$0.23 & 70.24$\pm$0.02 & 93.23$\pm$0.16 & 84.32$\pm$0.36 & 94.74$\pm$0.78 & 77.56$\pm$0.00\\
      \cmidrule{2-8}
      & DIS-Top50 & 81.30$\pm$0.73 & 69.80$\pm$0.19 & 93.33$\pm$0.00 & 84.79$\pm$0.39 & 93.49$\pm$0.03 & 76.75$\pm$0.22\\
      & DIS-GMM   & 83.91$\pm$0.26 & 70.21$\pm$0.20 & 93.33$\pm$0.00 & 86.77$\pm$0.16 & 94.43$\pm$0.55 & 77.53$\pm$0.11\\
      & DIS-GMM*  & 85.42$\pm$0.70 & 71.15$\pm$0.08 & 93.33$\pm$0.00 & 84.90$\pm$0.40 & 95.47$\pm$0.50 & 78.18$\pm$0.05\\
      \bottomrule
    \end{tabular}
    }
    \end{center}
    \vskip -0.1in
\end{table}

\newpage
\subsection{Detailed Results on the Impact of GMM Filter and Reject Filter on Voting Information Quality}
\label{sec:appendix_complete_filter_analysis}
    For the analysis of the impact of key components \textit{GMM Filter} and \textit{Reject Filter} of \textit{DistriVoting} on the confidence quality of trajectories involved in voting in the main text \autoref{tab:analysis_distribution_voting} of \autoref{sec:analysis_distribution_voting}, we provide detailed results for each benchmark here. In addition to DeepSeek-R1-8B, we also provide results for Qwen3-8B, Qwen3-14B, Qwen3-14B-NonThinking, and Qwen3-32B, as shown in \autoref{tab:appendix_complete_filter_analysis}.
    \begin{table}[ht]
    \vskip 0.2in
    \caption{Complete results of the proportion of correct samples across different stages of the voting process. Different models generate Budget=128 responses for each query during voting, with experiments repeated 64 times and results averaged.}
    \label{tab:appendix_complete_filter_analysis}
    \begin{center}
    \small
    \resizebox{0.95\textwidth}{!}{
    \begin{tabular}{llccccccc}
      \toprule
      Model & Metric & Stage & HMMT2025 & GPQA-D & AIME2024 & AIME2025 & BRUMO2025 & Avg. \\
      \midrule
      \multirow{7}{*}{DeepSeek-R1-8B}
      & \multirow{3}{*}{Acc}
        & \uppercase\expandafter{\romannumeral1} & 60.36 & 64.54 & 86.83 & 79.92 & 81.22 & 69.27 \\
      & & \uppercase\expandafter{\romannumeral2} & 76.71 & 71.69 & 93.90 & 87.79 & 91.06 & 77.60 \\
      & & \uppercase\expandafter{\romannumeral3} & 77.43 & 75.86 & 94.08 & 88.75 & 91.34 & 80.41 \\
      \cmidrule{2-9}
      & \multirow{3}{*}{WAcc}
        & \uppercase\expandafter{\romannumeral1} & 61.59 & 64.73 & 87.40 & 80.70 & 82.31 & 69.74 \\
      & & \uppercase\expandafter{\romannumeral2} & 76.92 & 71.72 & 94.00 & 87.86 & 91.24 & 77.68 \\
      & & \uppercase\expandafter{\romannumeral3} & 77.63 & 75.89 & 94.17 & 88.83 & 91.52 & 80.48 \\
      \midrule
      \midrule
      \multirow{7}{*}{Qwen3-8B}
      & \multirow{3}{*}{Acc}
        & \uppercase\expandafter{\romannumeral1} & 58.75 & 65.16 & 82.66 & 73.89 & 73.85 & 67.85 \\
      & & \uppercase\expandafter{\romannumeral2} & 73.25 & 71.78 & 88.69 & 80.94 & 81.19 & 75.26 \\
      & & \uppercase\expandafter{\romannumeral3} & 76.10 & 78.89 & 89.00 & 82.69 & 82.16 & 80.25 \\
      \cmidrule{2-9}
      & \multirow{3}{*}{WAcc}
        & \uppercase\expandafter{\romannumeral1} & 59.66 & 65.32 & 83.08 & 74.43 & 74.35 & 68.17 \\
      & & \uppercase\expandafter{\romannumeral2} & 73.34 & 71.80 & 88.74 & 80.93 & 81.26 & 75.30 \\
      & & \uppercase\expandafter{\romannumeral3} & 76.18 & 78.91 & 89.05 & 82.69 & 82.22 & 80.27 \\
      \midrule
      \midrule
      \multirow{7}{*}{Qwen3-14B}
      & \multirow{3}{*}{Acc}
        & \uppercase\expandafter{\romannumeral1} & 60.98 & 68.56 & 85.70 & 76.73 & 79.70 & 71.28 \\
      & & \uppercase\expandafter{\romannumeral2} & 76.48 & 75.40 & 90.38 & 86.40 & 85.46 & 78.90 \\
      & & \uppercase\expandafter{\romannumeral3} & 78.32 & 82.20 & 90.76 & 87.40 & 86.73 & 83.56 \\
      \cmidrule{2-9}
      & \multirow{3}{*}{WAcc}
        & \uppercase\expandafter{\romannumeral1} & 61.70 & 68.70 & 86.02 & 77.32 & 80.16 & 71.57 \\
      & & \uppercase\expandafter{\romannumeral2} & 76.49 & 75.43 & 90.44 & 86.42 & 85.55 & 78.94 \\
      & & \uppercase\expandafter{\romannumeral3} & 78.33 & 82.22 & 90.82 & 87.42 & 86.82 & 83.59 \\
      \midrule
      \midrule
      \multirow{7}{*}{Qwen3-14B*}
      & \multirow{3}{*}{Acc}
        & \uppercase\expandafter{\romannumeral1} & 24.42 & 56.21 & 37.93 & 37.95 & 43.06 & 48.52 \\
      & & \uppercase\expandafter{\romannumeral2} & 28.67 & 62.13 & 42.14 & 42.42 & 45.52 & 53.66 \\
      & & \uppercase\expandafter{\romannumeral3} & 29.81 & 69.34 & 43.90 & 43.29 & 47.13 & 58.66 \\
      \cmidrule{2-9}
      & \multirow{3}{*}{WAcc}
        & \uppercase\expandafter{\romannumeral1} & 25.20 & 56.56 & 38.97 & 38.55 & 43.55 & 49.01 \\
      & & \uppercase\expandafter{\romannumeral2} & 28.93 & 62.22 & 42.53 & 42.60 & 45.69 & 53.81 \\
      & & \uppercase\expandafter{\romannumeral3} & 30.08 & 69.41 & 44.28 & 43.45 & 47.28 & 58.79 \\
      \midrule
      \midrule
      \multirow{7}{*}{Qwen3-32B}
      & \multirow{3}{*}{Acc}
        & \uppercase\expandafter{\romannumeral1} & 59.29 & 72.72 & 87.65 & 73.33 & 81.82 & 73.78 \\
      & & \uppercase\expandafter{\romannumeral2} & 61.27 & 79.30 & 89.27 & 73.95 & 84.00 & 78.48 \\
      & & \uppercase\expandafter{\romannumeral3} & 63.44 & 85.51 & 89.45 & 76.60 & 85.07 & 82.91 \\
      \cmidrule{2-9}
      & \multirow{3}{*}{WAcc}
        & \uppercase\expandafter{\romannumeral1} & 59.94 & 72.91 & 87.98 & 73.57 & 82.48 & 74.07 \\
      & & \uppercase\expandafter{\romannumeral2} & 61.51 & 79.28 & 89.36 & 74.03 & 84.20 & 78.52 \\
      & & \uppercase\expandafter{\romannumeral3} & 63.67 & 85.48 & 89.53 & 76.68 & 85.25 & 82.96 \\   
      \bottomrule
    \end{tabular}
    }
    \end{center}
    \vskip 0.3in
\end{table}

\newpage
\subsection{Detailed Results on the Impact of \textit{SelfStepConf} on Confidence Distribution}
\label{sec:appendix_complete_ssc_distribution}
    For the impact of \textit{SelfStepConf} on confidence distribution under DeepSeek-R1-8B shown in \autoref{fig:analysis_ssc_separation_comparison} of \autoref{sec:analysis_ssc_separation_comparison} in the main text, we provide additional quantitative experiments under more models (Qwen3-8B, Qwen3-14B, Qwen3-14B-NonThinking, and Qwen3-32B) here. As shown in \autoref{tab:appendix_complete_ssc_distribution}, we primarily use two evaluation metrics: AUROC and WAcc. It can be observed that under different models, the confidence quality (AUROC) of SSC's generated results is higher compared to the BasicInference, and the average correct confidence (WAcc) is also higher.

\begin{table}[ht]
    \vskip 0.1in
    \caption{Complete results of the confidence separation of the content generated by BasicInference and \textit{SelfStepConf} on different models.}
    \label{tab:appendix_complete_ssc_distribution}
    \begin{center}
    \small
    \begin{tabular}{llccccccc}
      \toprule
      \multirow{2}{*}{Model} & \multirow{2}{*}{Benchmark} & \multicolumn{2}{c}{BasicInference} & \multicolumn{2}{c}{SelfStepConf} & \multicolumn{2}{c}{$\Delta$ (\%)} \\
      \cmidrule(lr){3-4} \cmidrule(lr){5-6} \cmidrule(lr){7-8}
      & & AUROC & WAcc & AUROC & WAcc & AUROC & WAcc \\
      \midrule
      \multirow{6}{*}{DeepSeek-R1-8B}
      & HMMT2025 & 0.9002 & 67.28 & 0.8957 & 67.72 & -0.50 & +0.65 \\
      & GPQA-D & 0.7167 & 60.59 & 0.7699 & 61.29 & +7.42 & +1.16 \\
      & AIME2024 & 0.8833 & 86.56 & 0.8881 & 86.96 & +0.54 & +0.46 \\
      & AIME2025 & 0.8294 & 82.62 & 0.8478 & 82.87 & +2.22 & +0.30 \\
      & BRUMo2025 & 0.8198 & 84.88 & 0.8304 & 85.40 & +1.29 & +0.61 \\
      \cmidrule{2-8}
      & Avg. & 0.7701 & 68.04 & 0.8060 & 68.63 & +4.66 & +0.87 \\
      \midrule
      \midrule
      \multirow{6}{*}{Qwen3-8B}
      & HMMT2025 & 0.8424 & 48.39 & 0.9001 & 50.24 & +6.85 & +3.83 \\
      & GPQA-D & 0.6518 & 56.06 & 0.7505 & 57.60 & +15.14 & +2.75 \\
      & AIME2024 & 0.8012 & 80.06 & 0.9099 & 80.25 & +13.57 & +0.24 \\
      & AIME2025 & 0.7390 & 66.43 & 0.8942 & 67.09 & +21.00 & +0.99 \\
      & BRUMo2025 & 0.7600 & 75.99 & 0.7992 & 75.82 & +5.16 & -0.22 \\
      \cmidrule{2-8}
      & Avg. & 0.7023 & 60.46 & 0.7978 & 61.66 & +13.60 & +1.98 \\
      \midrule
      \midrule
      \multirow{6}{*}{Qwen3-14B}
      & HMMT2025 & 0.8071 & 53.46 & 0.8979 & 53.90 & +11.25 & +0.82 \\
      & GPQA-D & 0.6660 & 58.29 & 0.7447 & 60.12 & +11.82 & +3.14 \\
      & AIME2024 & 0.7745 & 83.93 & 0.9235 & 83.15 & +19.24 & -0.93 \\
      & AIME2025 & 0.6932 & 71.08 & 0.8849 & 70.48 & +27.65 & -0.84 \\
      & BRUMo2025 & 0.7566 & 79.38 & 0.7961 & 79.38 & +5.22 & +0.00 \\
      \cmidrule{2-8}
      & Avg. & 0.7007 & 63.45 & 0.7941 & 64.50 & +13.33 & +1.65 \\
      \midrule
      \midrule
      \multirow{6}{*}{Qwen3-14B*}
      & HMMT2025 & 0.7499 & 19.40 & 0.8471 & 20.40 & +12.96 & +5.15 \\
      & GPQA-D & 0.6229 & 25.96 & 0.6570 & 26.49 & +5.47 & +2.04 \\
      & AIME2024 & 0.7724 & 38.53 & 0.7945 & 43.42 & +2.87 & +12.69 \\
      & AIME2025 & 0.8535 & 32.53 & 0.9144 & 35.67 & +7.14 & +9.65 \\
      & BRUMo2025 & 0.6821 & 41.79 & 0.7131 & 45.99 & +4.54 & +10.05 \\
      \cmidrule{2-8}
      & Avg. & 0.6763 & 28.64 & 0.7175 & 30.22 & +6.09 & +5.52 \\
      \midrule
      \midrule
      \multirow{6}{*}{Qwen3-32B}
      & HMMT2025 & 0.7976 & 56.84 & 0.8868 & 57.76 & +11.18 & +1.62 \\
      & GPQA-D & 0.6673 & 63.27 & 0.7763 & 66.04 & +16.33 & +4.38 \\
      & AIME2024 & 0.9538 & 83.93 & 0.8713 & 82.83 & -8.65 & -1.31 \\
      & AIME2025 & 0.6743 & 71.67 & 0.8371 & 71.88 & +24.14 & +0.29 \\
      & BRUMo2025 & 0.7273 & 83.29 & 0.7692 & 82.82 & +5.76 & -0.56 \\
      \cmidrule{2-8}
      & Avg. & 0.7129 & 67.29 & 0.8008 & 68.98 & +12.33 & +2.51 \\
      \bottomrule
    \end{tabular}
    \end{center}
\end{table}

\newpage
\subsection{Top-Threshold Sensitivity Analysis on More Models and Benchmarks}
\label{sec:appendix_complete_top50_analysis}
    For the results shown in \autoref{fig:analysis_top50_selection} of \autoref{sec:analysis_top50_selection} in the main text, which display the optimal benchmark-level top-threshold applicable under different benchmarks using Qwen3-8B, we present here from both benchmark and model perspectives. As shown in \autoref{fig:appendix_complete_top50_analysis}, we separately demonstrate that the applicable top-threshold varies for different models on the same benchmark, and also varies for different benchmarks on the same model. This further supports the rationale for using top50 to represent the fixed threshold filter method.
    
    \begin{figure}[ht]
    \vskip 0.1in
        \begin{center}
            \begin{subfigure}[b]{0.48\textwidth}
                \centering
                \includegraphics[width=\textwidth]{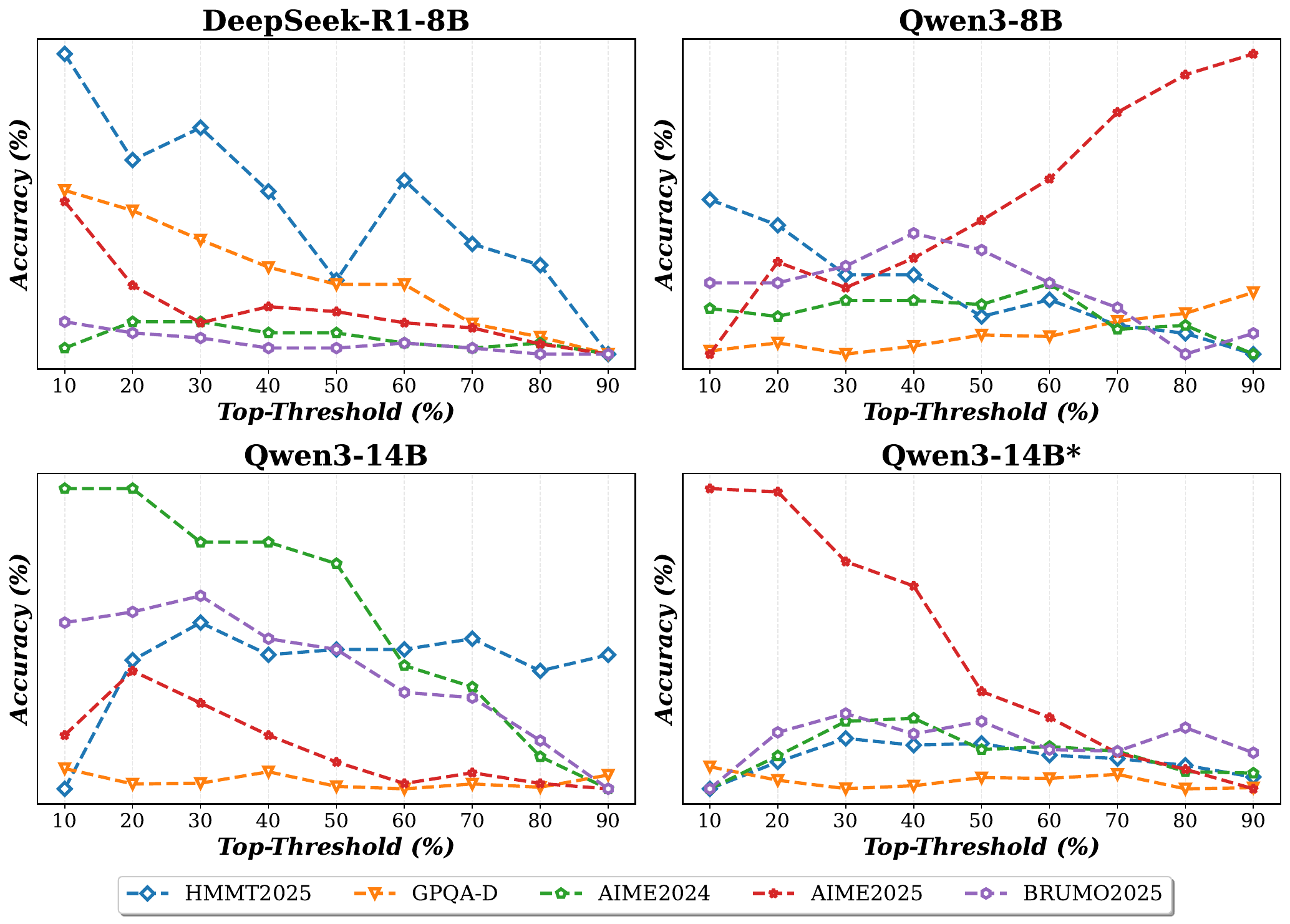}
                \caption{Top-Threshold sensitivity on different models}
            \end{subfigure}
            \hfill
            \begin{subfigure}[b]{0.48\textwidth}
                \centering
                \includegraphics[width=\textwidth]{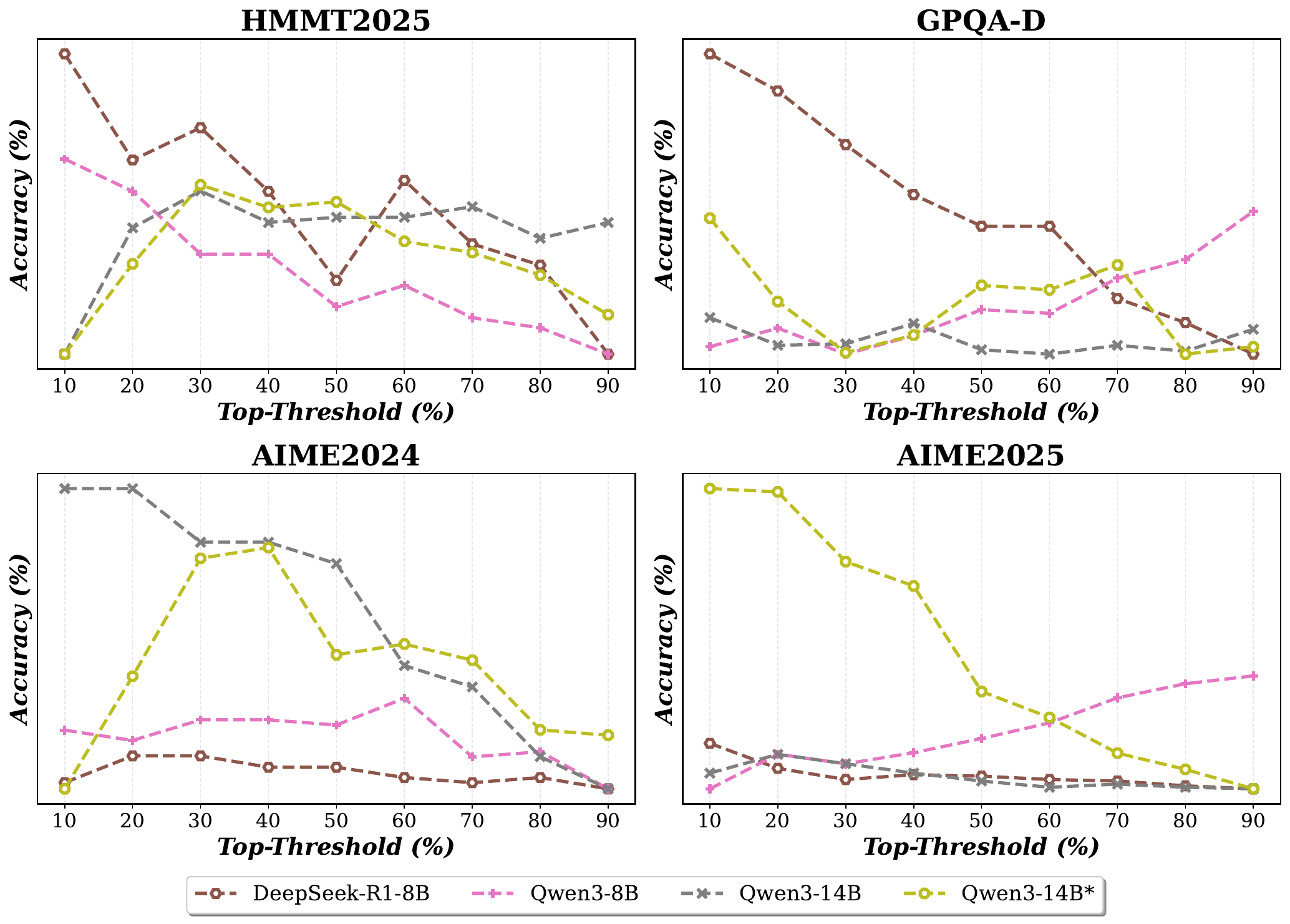}
                \caption{Top-Threshold sensitivity on different benchmark}
            \end{subfigure}
        \caption{
          Optimal fixed top-threshold traversal results for filter voting across different benchmarks. Using different models to generate 256 responses for each query, repeated 64 times. \textbf{Accuracy denotes WSC-TopK.} Complete results are provided in \autoref{tab:appendix_complete_top50_analysis}.
          }
        \label{fig:appendix_complete_top50_analysis}
        \end{center}
    \vskip 0.1in
    \end{figure}

    \begin{table}[ht]
  \vskip 0.1in
  \caption{Complete WSC-TopK results of optimal fixed top-threshold traversal results for filter voting across different benchmarks. Using different models to generate 256 responses for each query, repeated 64 times.}
  \label{tab:appendix_complete_top50_analysis}
  \small
  \begin{center}
  \resizebox{0.95\textwidth}{!}{
    \begin{tabular}{llcccccccccc}
      \toprule
      \multirow{2}{*}{Model} & \multirow{2}{*}{Benchmark} & \multicolumn{9}{c}{Top-Threshold} \\
      \cmidrule(lr){3-11}
      & & 10 & 20 & 30 & 40 & 50 & 60 & 70 & 80 & 90 & \\
      \midrule
        \multirow{5}{*}{DeepSeek-R1-8B}
        & HMMT2025 & 81.88 & 80.83 & 81.15 & 80.52 & 79.64 & 80.63 & 80.00 & 79.79 & 78.91 \\
        & GPQA-D & 70.24 & 70.04 & 69.75 & 69.48 & 69.31 & 69.31 & 68.92 & 68.79 & 68.62 \\
        & AIME2024 & 92.92 & 93.18 & 93.18 & 93.07 & 93.07 & 92.97 & 92.92 & 92.97 & 92.86 \\
        & AIME2025 & 85.47 & 84.64 & 84.27 & 84.43 & 84.38 & 84.27 & 84.22 & 84.06 & 83.96 \\
        & BRUMO2025 & 93.65 & 93.54 & 93.49 & 93.39 & 93.39 & 93.44 & 93.39 & 93.33 & 93.33 \\
        \midrule
        \multirow{5}{*}{Qwen3-8B}
        & HMMT2025 & 62.66 & 62.34 & 61.72 & 61.72 & 61.20 & 61.41 & 61.09 & 60.99 & 60.73 \\
        & GPQA-D & 64.13 & 64.23 & 64.09 & 64.19 & 64.33 & 64.31 & 64.50 & 64.60 & 64.86 \\
        & AIME2024 & 86.98 & 86.88 & 87.08 & 87.08 & 87.03 & 87.29 & 86.72 & 86.77 & 86.41 \\
        & AIME2025 & 72.03 & 73.18 & 72.86 & 73.23 & 73.70 & 74.22 & 75.05 & 75.52 & 75.78 \\
        & BRUMO2025 & 82.45 & 82.45 & 82.66 & 83.07 & 82.86 & 82.45 & 82.14 & 81.56 & 81.82 \\
        \midrule
        \multirow{5}{*}{Qwen3-14B}
        & HMMT2025 & 62.71 & 63.96 & 64.32 & 64.01 & 64.06 & 64.06 & 64.17 & 63.85 & 64.01 \\
        & GPQA-D & 67.10 & 66.95 & 66.96 & 67.07 & 66.93 & 66.90 & 66.95 & 66.92 & 67.04 \\
        & AIME2024 & 89.06 & 89.06 & 88.54 & 88.54 & 88.33 & 87.34 & 87.14 & 86.46 & 86.15 \\
        & AIME2025 & 77.55 & 78.18 & 77.86 & 77.55 & 77.29 & 77.08 & 77.19 & 77.08 & 77.03 \\
        & BRUMO2025 & 86.61 & 86.72 & 86.88 & 86.46 & 86.35 & 85.94 & 85.89 & 85.47 & 85.00 \\
        \midrule
        \multirow{5}{*}{Qwen3-14B*}
        & HMMT2025 & 24.22 & 25.11 & 25.89 & 25.67 & 25.73 & 25.33 & 25.22 & 25.00 & 24.61 \\
        & GPQA-D & 56.96 & 56.51 & 56.23 & 56.33 & 56.60 & 56.57 & 56.71 & 56.23 & 56.27 \\
        & AIME2024 & 48.39 & 49.48 & 50.63 & 50.73 & 49.69 & 49.79 & 49.64 & 48.96 & 48.91 \\
        & AIME2025 & 45.91 & 45.80 & 43.48 & 42.67 & 39.17 & 38.31 & 37.12 & 36.58 & 35.94 \\
        & BRUMO2025 & 48.91 & 50.78 & 51.41 & 50.73 & 51.15 & 50.21 & 50.16 & 50.94 & 50.10 \\
        \bottomrule
    \end{tabular}
  }
  \end{center}
  \vskip -0.5in
\end{table}

\newpage
\subsection{Detailed Clustering Results of Answer Categories via GMM}
\label{sec:appendix_complete_gmm_components_analysis}
    In \autoref{sec:analysis_gmm_components} (\autoref{fig:analysis_gmm_components}) of the main text, we presented aggregated visualization results of the GMM distributions, where the number of answer types was used to determine the optimal number of GMM components. Here, we provide separate visualizations for the frequency distributions of selected top-frequent answers. As shown in \autoref{fig:appendix_complete_gmm_components_analysis}, each answer's distribution approximately follows a Gaussian pattern, which aligns consistently with our theoretical analysis in main text.

    \begin{figure}[H]
        \begin{center}
        \centerline{\includegraphics[width=0.98\columnwidth]{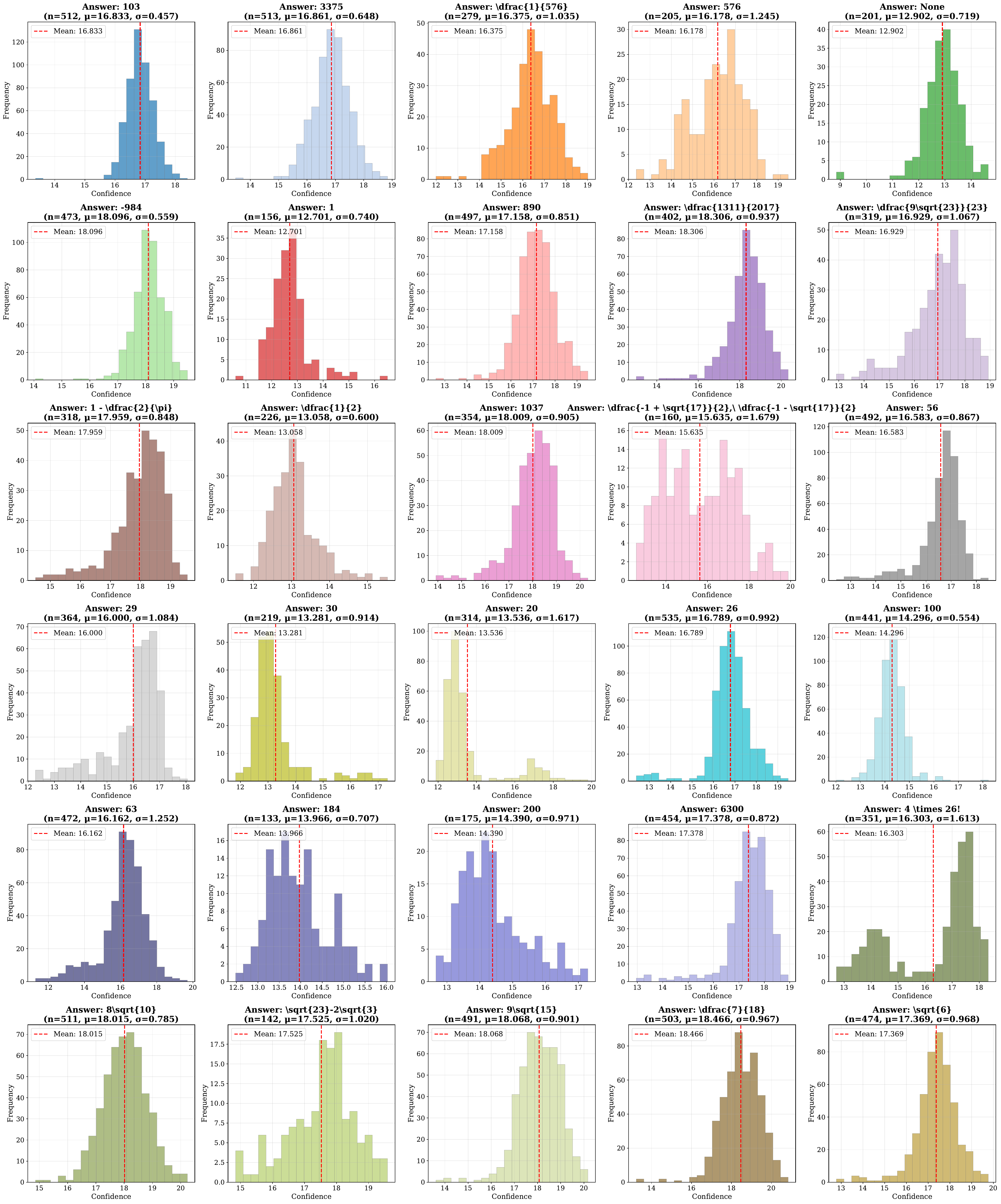}}
        \caption{
            Distribution visualization of answers as Gaussian components in GMM using DeepSeek-R1-8B with 512 repeated sampling for each query on HMMT2025. Individual visualization of the top 30 most frequent answers.
        }
        \label{fig:appendix_complete_gmm_components_analysis}
        \end{center}
    \end{figure}

\subsection{Additional Results on the Impact of SelfStepConf at the Trajectory-Level}
\label{sec:appendix_complete_trajlevel_conf_analysis}
    For the visualization results in \autoref{fig:analysis_ssc_pass@k_confidence} of \autoref{sec:analysis_ssc_pass@k} in the main text, which show the impact of \textit{SelfStepConf} on trajectory-level confidence compared to the BasicInference using DeepSeek-R1-8B, we additionally provide results for Qwen3-8B, Qwen3-14B, Qwen3-14B-NonThinking, and Qwen3-32B here. As shown in \autoref{fig:appendix_complete_trajlevel_conf_analysis}, it can be intuitively observed that our analysis conclusions remain consistent across different models.
    
    \begin{figure}[ht]
    \vskip 0.2in
        \centering
        \begin{subfigure}[b]{0.48\textwidth}
            \centering
            \includegraphics[width=\textwidth]{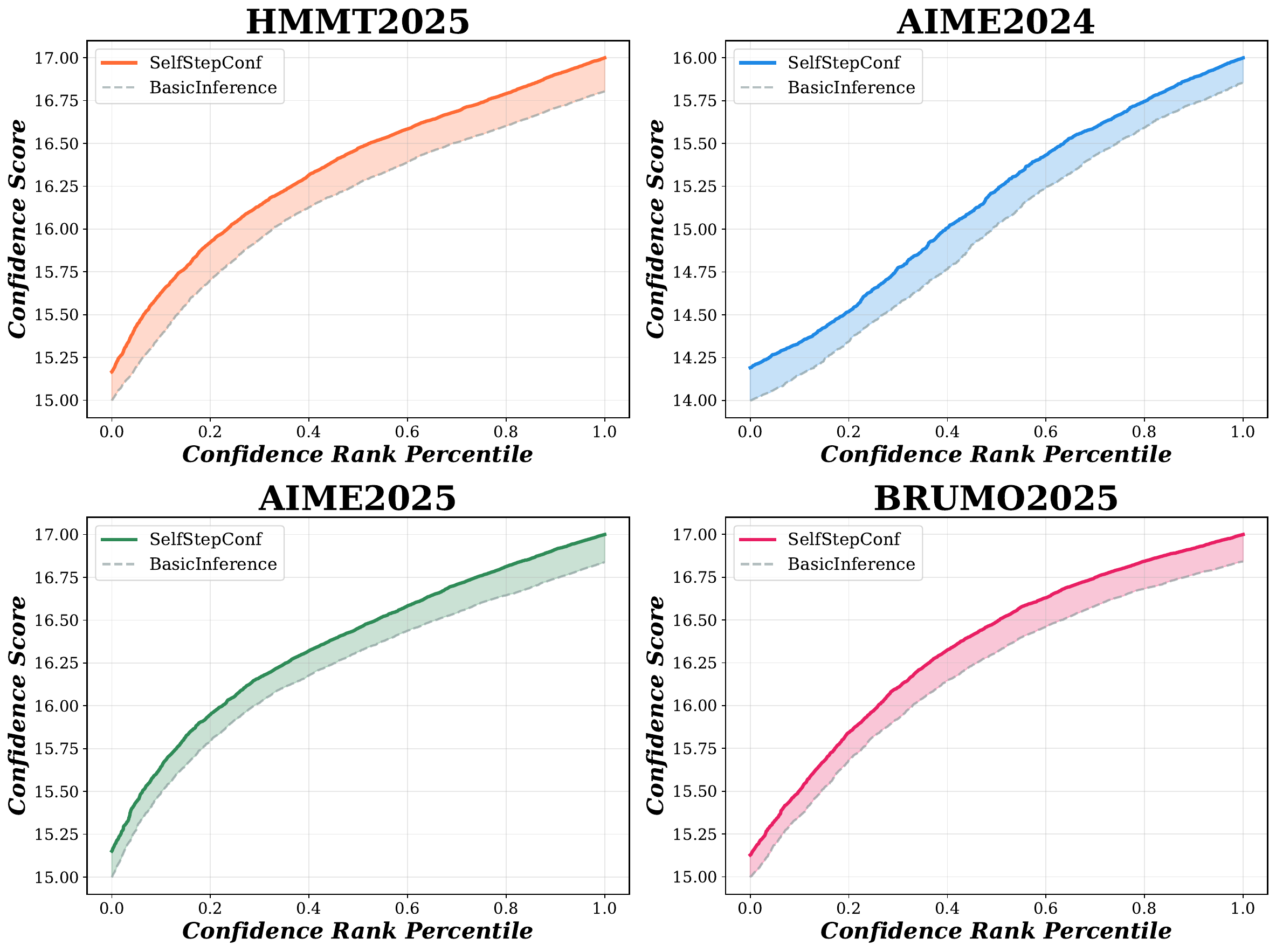}
            \caption{DeepSeek-R1-8B}
        \end{subfigure}
        \hfill
        \begin{subfigure}[b]{0.48\textwidth}
            \centering
            \includegraphics[width=\textwidth]{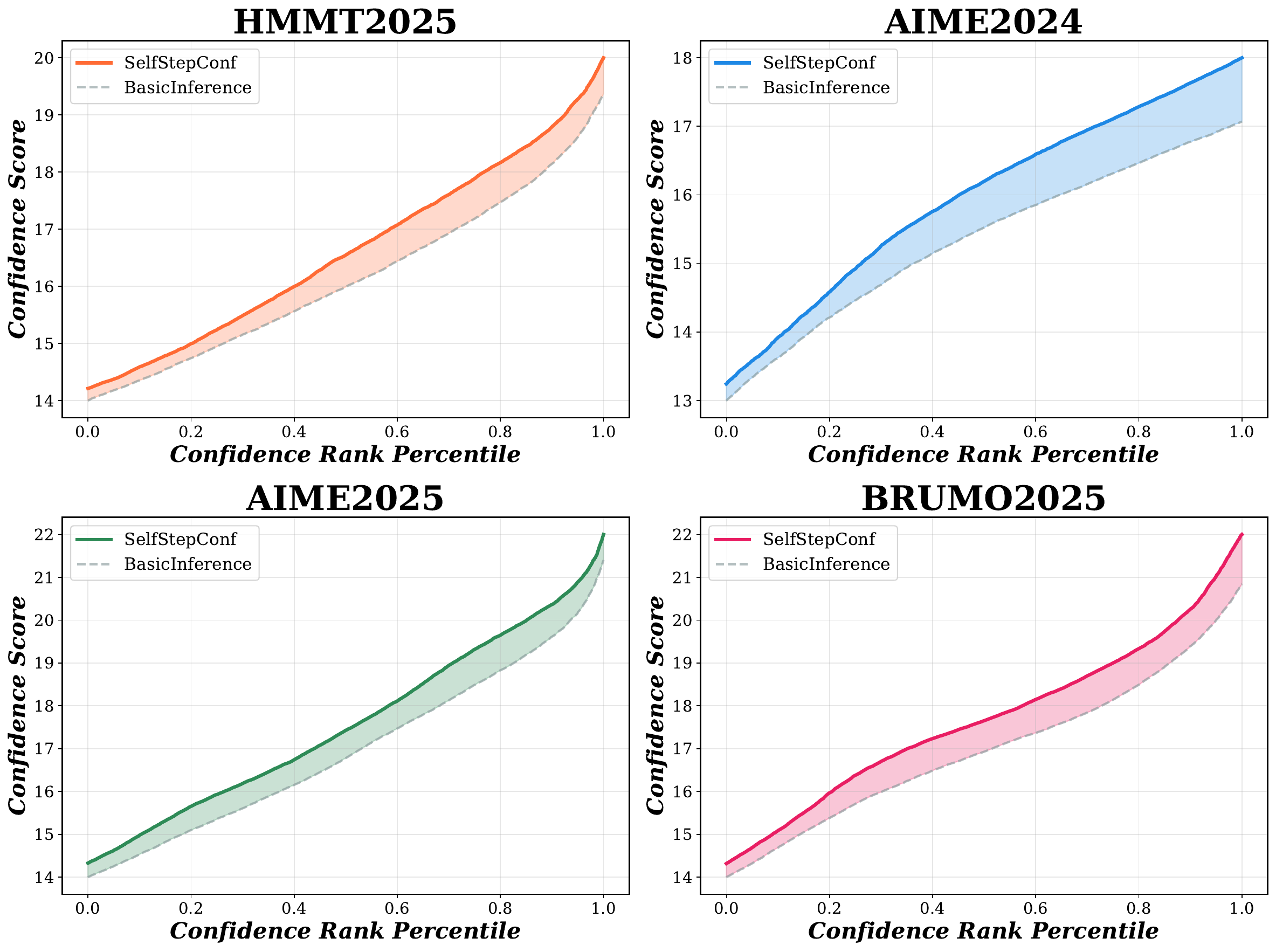}
            \caption{Qwen3-8B}
        \end{subfigure}
        
        \vspace{0.5cm}
        
        \begin{subfigure}[b]{0.48\textwidth}
            \centering
            \includegraphics[width=\textwidth]{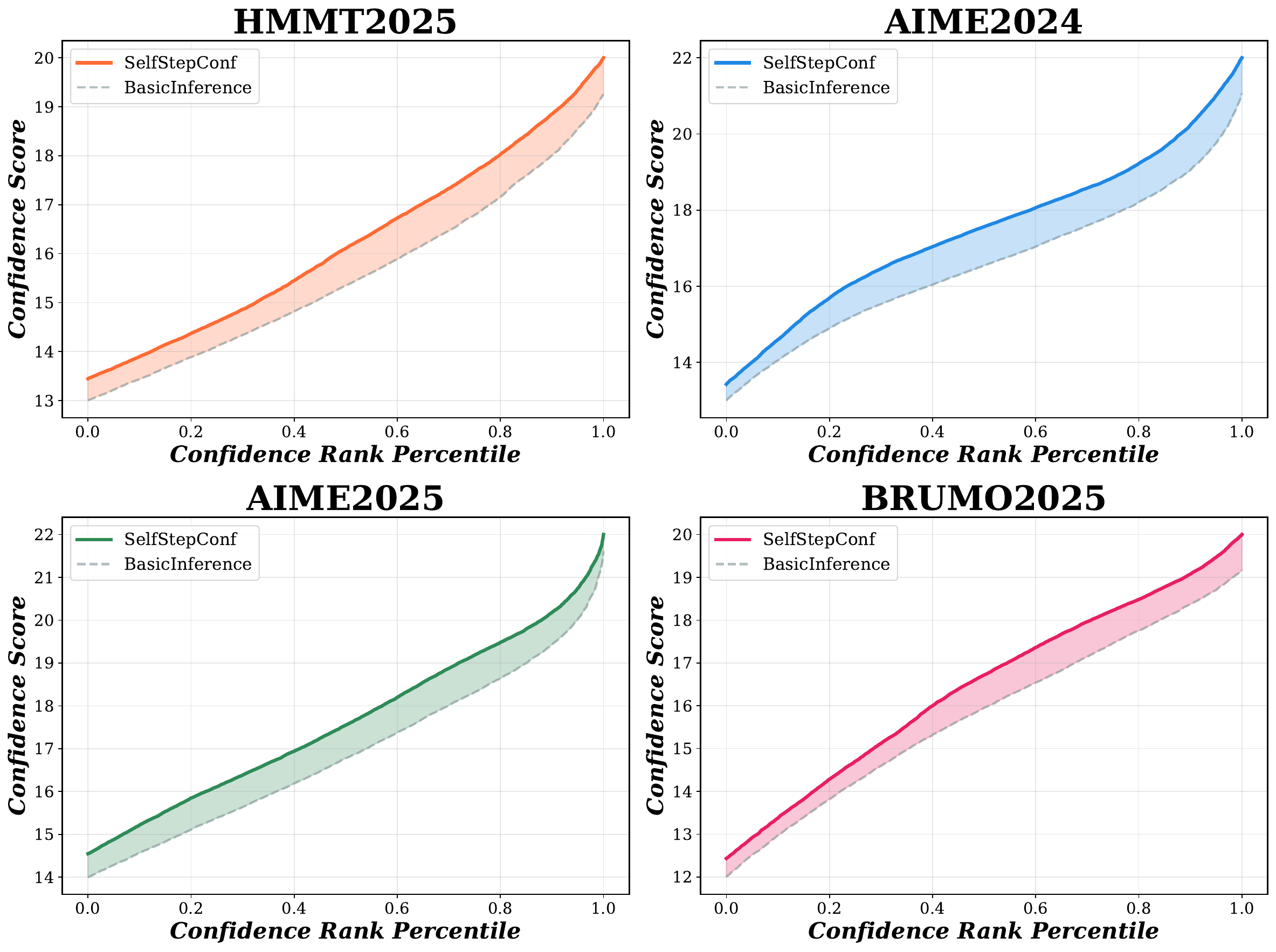}
            \caption{Qwen3-14B}
        \end{subfigure}
        \hfill
        \begin{subfigure}[b]{0.48\textwidth}
            \centering
            \includegraphics[width=\textwidth]{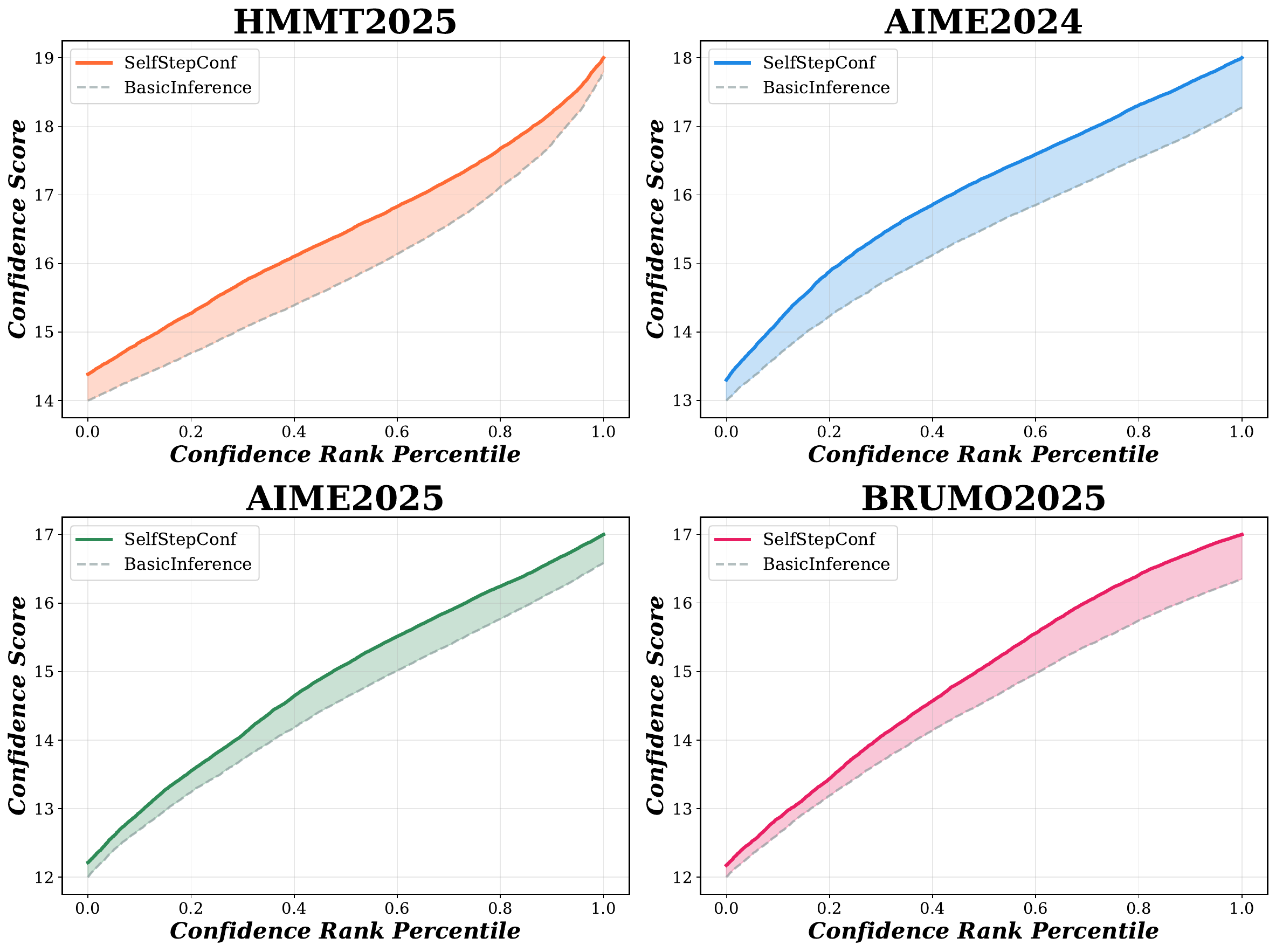}
            \caption{Qwen3-32B}
        \end{subfigure}
        
        \caption{Complete results of trajectory-level confidence comparison between SSC and BasicInference across different benchmarks. Using different models to generate 512 responses for each query, and all responses across the entire benchmark are sorted and visualized by confidence scores.}
        \label{fig:appendix_complete_trajlevel_conf_analysis}
    \end{figure}

\newpage
\subsection{Pass@K Analysis of SelfStepConf on More Benchmarks and Models}
\label{sec:appendix_complete_pass@k}
    For the analysis results shown in \autoref{fig:analysis_ssc_pass@k_pass@k} of \autoref{sec:analysis_ssc_pass@k} in the main text, which compare the pass@K performance of \textit{SelfStepConf} to the BasicInference on GPQA-D under different models, we additionally provide results for four benchmarks: HMMT2025, AIME2024, AIME2025, and BRUMO2025. As shown in \autoref{fig:appendix_complete_pass@k}, it can be intuitively observed that our analysis remains consistent across different benchmarks.

    \begin{figure}[ht]
        \centering
        \begin{subfigure}[b]{0.48\textwidth}
            \centering
            \includegraphics[width=\textwidth]{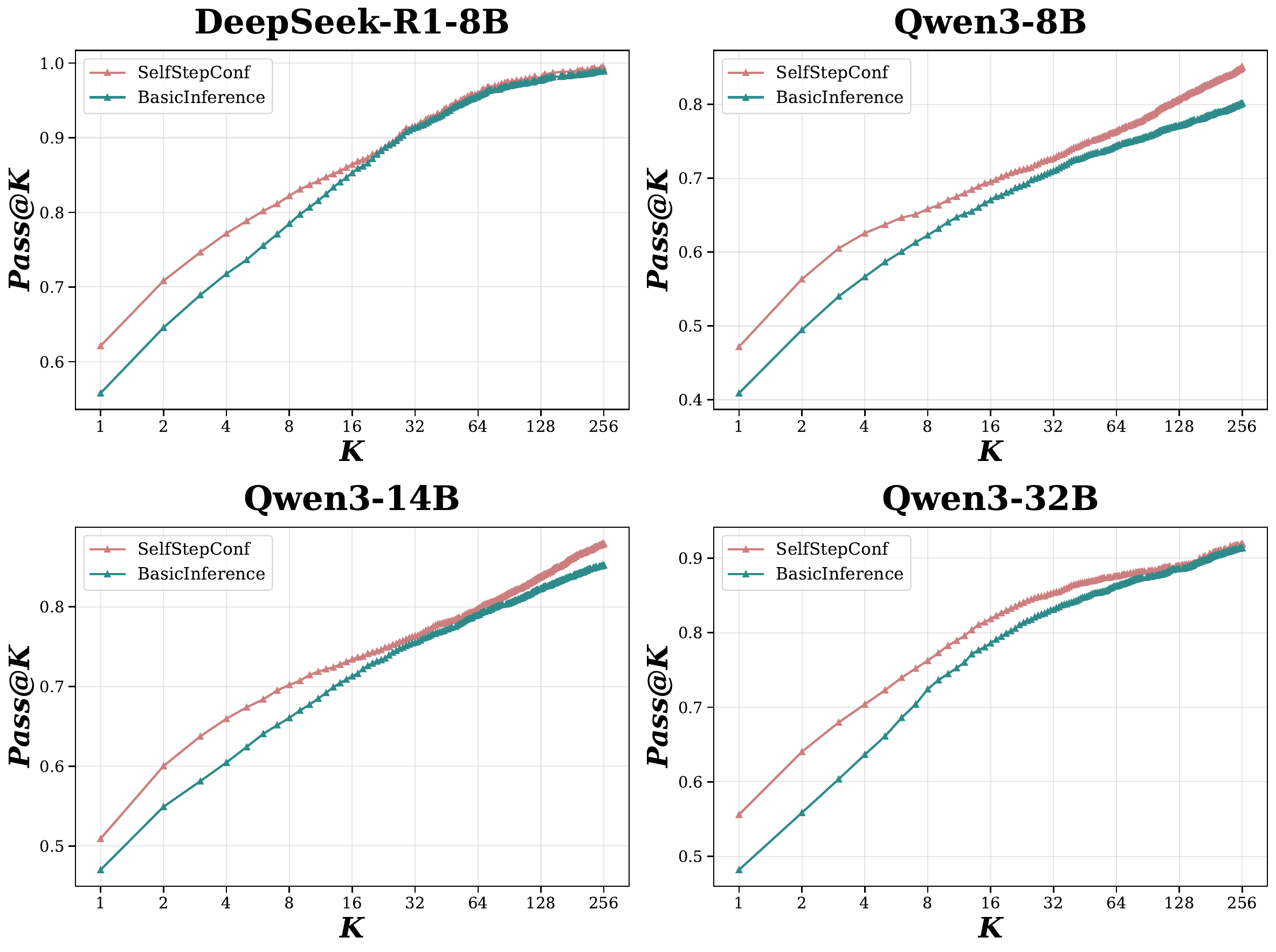}
            \caption{HMMT2025}
        \end{subfigure}
        \hfill
        \begin{subfigure}[b]{0.48\textwidth}
            \centering
            \includegraphics[width=\textwidth]{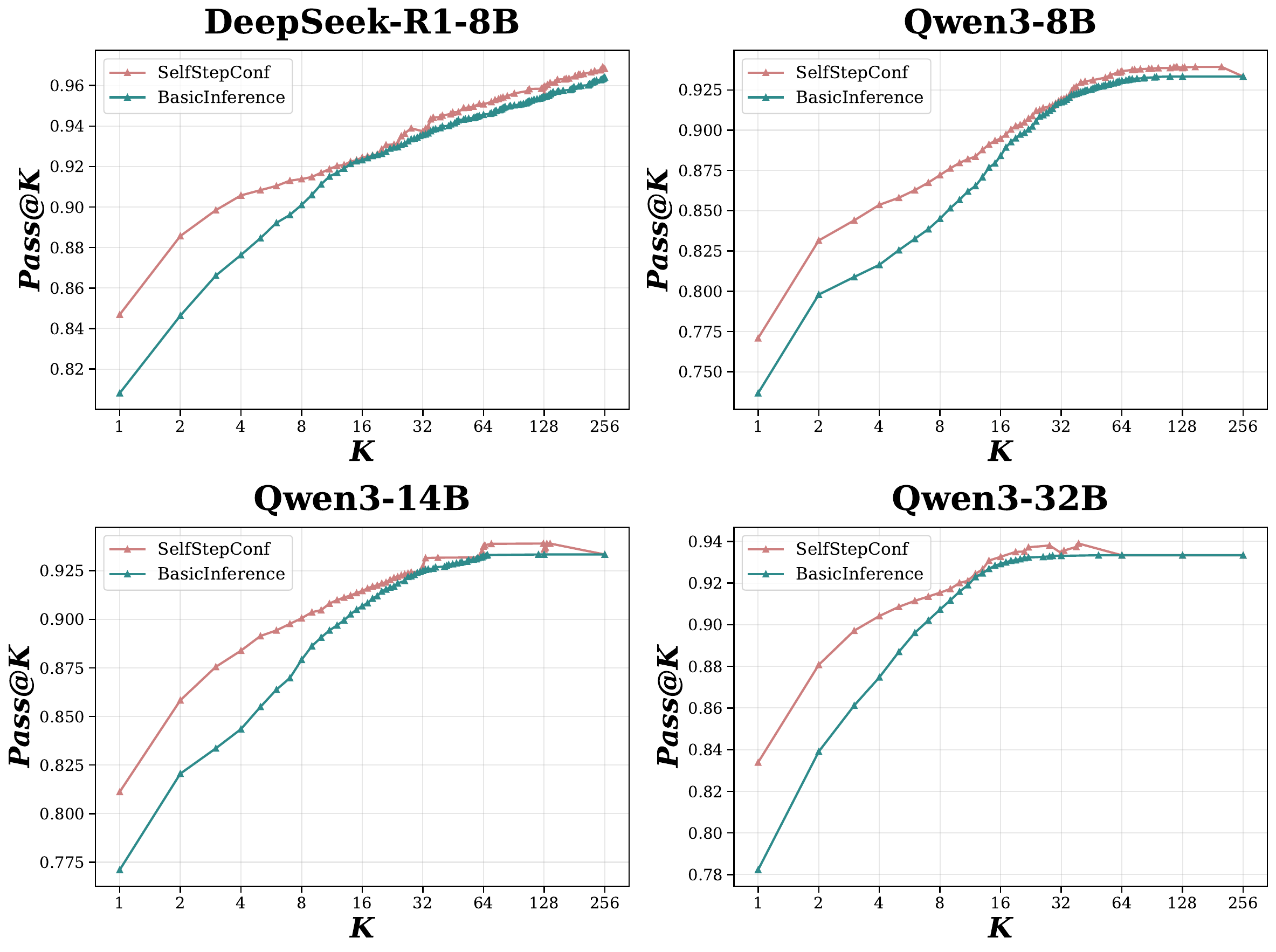}
            \caption{AIME2024}
        \end{subfigure}
        
        \vspace{0.5cm}
        
        \begin{subfigure}[b]{0.48\textwidth}
            \centering
            \includegraphics[width=\textwidth]{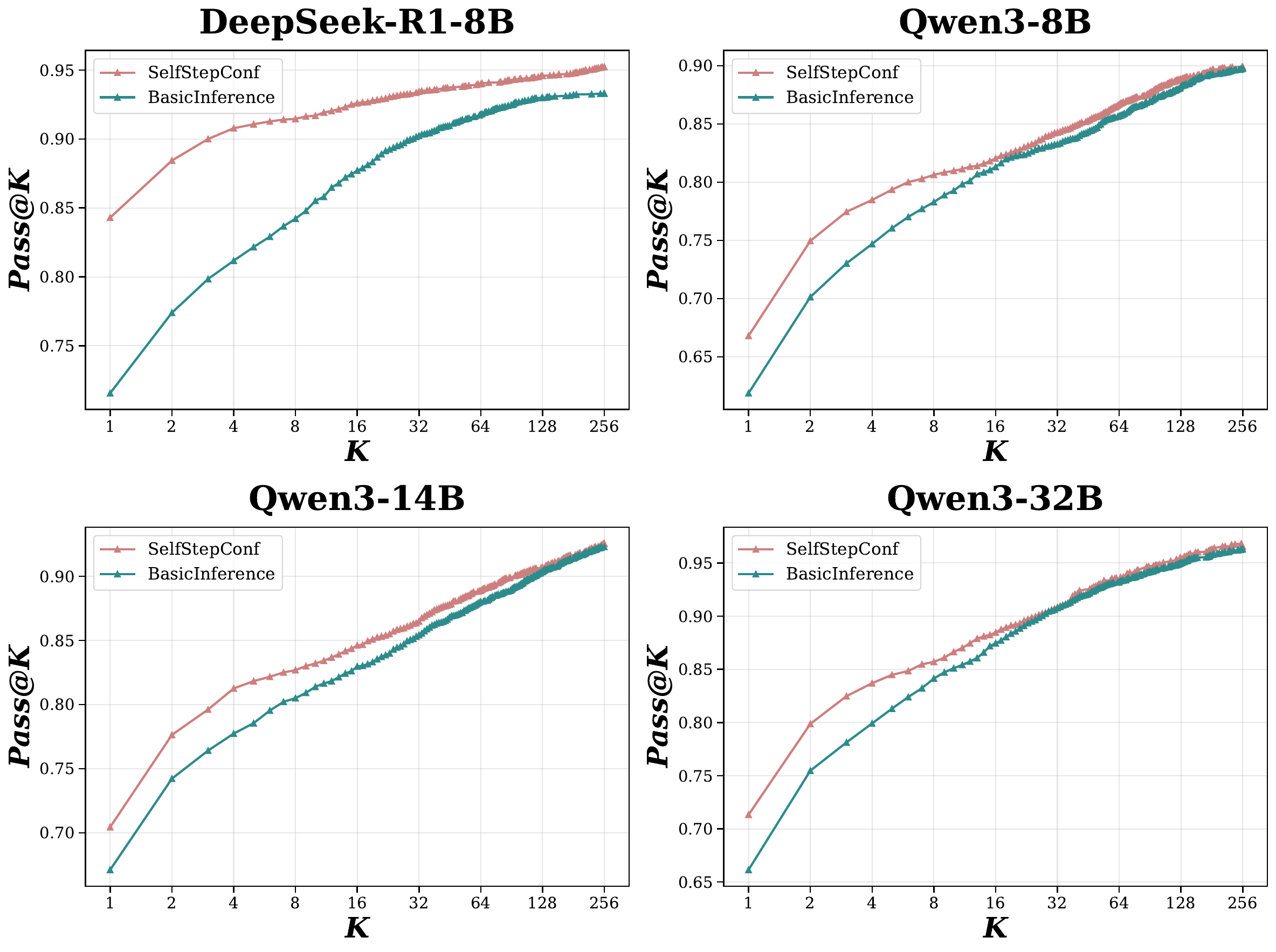}
            \caption{AIME2025}
        \end{subfigure}
        \hfill
        \begin{subfigure}[b]{0.48\textwidth}
            \centering
            \includegraphics[width=\textwidth]{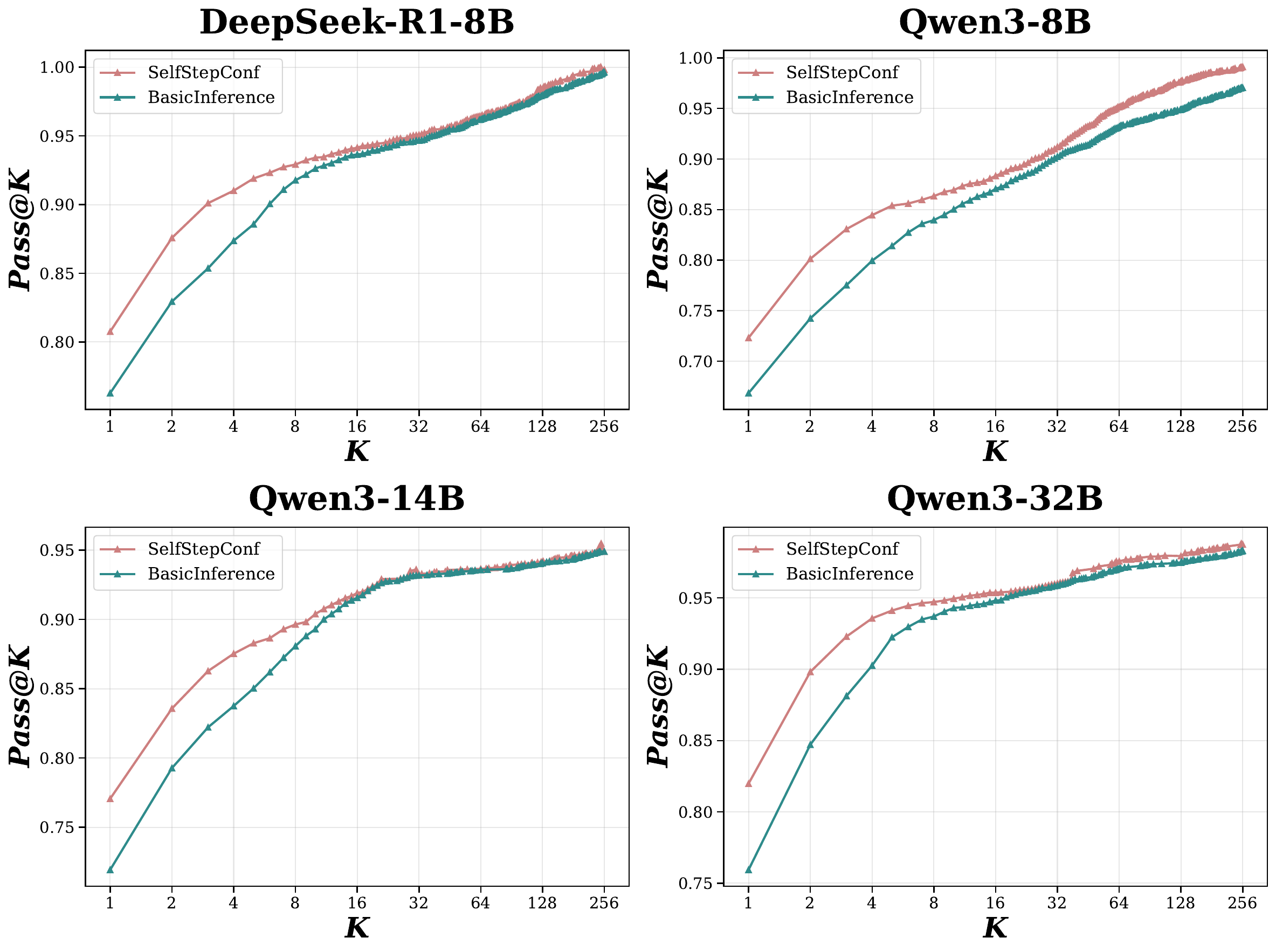}
            \caption{BRUMO2025}
        \end{subfigure}
        
        \caption{Complete results of Pass@K comparison between SSC and BasicInference under different sampling numbers. Different models generate 256 responses on all Benchmark, repeated 64 times.}
        \label{fig:appendix_complete_pass@k}
    \end{figure}

\newpage
\subsection{Detailed Pass@1 Results of SelfStepConf and BasicInference}
\label{sec:appendix_complete_pass@1}
    For the average pass@1 results of \textit{SelfStepConf} and BasicInference under different models shown in \autoref{fig:analysis_ssc_pass@k_pass@1} of \autoref{sec:analysis_ssc_pass@k} in the main text, we provide detailed data specific to each benchmark here, as shown in \autoref{tab:appendix_complete_pass@1}.

    \begin{table}[ht]
  \caption{Complete results of Pass@1 comparison between SSC and BasicInference across different models and benchmarks. Results averaged over 64 independent sampling per query.}
  \label{tab:appendix_complete_pass@1}
  \small
  \begin{center}
  \resizebox{0.95\textwidth}{!}{
    \begin{tabular}{llccccccc}
      \toprule
      Method & Model & HMMT2025 & GPQA-D & AIME2024 & AIME2025 & BRUMO2025 & Avg.\\
      \midrule
        \multirow{16}{*}{BaseModel}
        & Qwen2.5-Math-7B          & 0.00  & 12.89 & 1.56  & 0.47  & 2.03  & 8.41  \\
        & Qwen3-0.6B*              & 0.97  & 23.14 & 2.66  & 2.45  & 8.12  & 15.75 \\
        & Qwen3-0.6B               & 8.89  & 21.31 & 10.73 & 16.78 & 16.93 & 18.30 \\
        & Llama3.1-8B-Instruct     & 0.10  & 28.54 & 5.31  & 0.73  & 3.44  & 18.67 \\
        & DeepSeek-R1-7B           & 20.31 & 12.47 & 41.04 & 30.01 & 40.07 & 20.16 \\
        & Qwen3-1.7B*              & 5.68  & 28.23 & 12.92 & 9.32  & 17.40 & 21.85 \\
        & Qwen3-4B*                & 11.67 & 41.26 & 22.45 & 18.39 & 27.60 & 33.25 \\
        & Qwen3-1.7B               & 22.29 & 34.91 & 47.60 & 34.48 & 47.60 & 36.07 \\
        & Qwen3-8B*                & 10.36 & 46.32 & 28.23 & 20.16 & 28.02 & 37.03 \\
        & Qwen3-14B*               & 9.81  & 52.38 & 27.45 & 21.09 & 31.67 & 41.11 \\
        & Qwen3-32B*               & 12.18 & 51.55 & 30.05 & 24.95 & 35.73 & 41.81 \\
        & Qwen3-4B                 & 42.34 & 52.23 & 71.30 & 62.24 & 62.60 & 55.02 \\
        & Qwen3-8B                 & 41.46 & 54.70 & 73.02 & 62.55 & 67.19 & 57.10 \\
        & Qwen3-14B                & 49.32 & 63.68 & 78.59 & 68.91 & 75.16 & 65.31 \\
        & DeepSeek-R1-8B           & 57.71 & 61.79 & 82.76 & 73.28 & 77.71 & 65.97 \\
        & Qwen3-32B                & 52.34 & 69.11 & 81.20 & 69.64 & 79.53 & 69.70 \\
        \midrule
        \multirow{16}{*}{SelfStepConf}
        & Qwen2.5-Math-7B          & 0.00  & 12.97 & 2.11  & 0.22  & 1.33  & 8.42  \\
        & Qwen3-0.6B*              & 1.19  & 22.92 & 2.81  & 2.66  & 9.53  & 15.80 \\
        & Qwen3-0.6B               & 8.81  & 21.25 & 11.00 & 16.42 & 18.68 & 18.41 \\
        & Llama3.1-8B-Instruct     & 0.10  & 28.78 & 4.89  & 0.59  & 3.39  & 18.77 \\
        & DeepSeek-R1-7B           & 19.31 & 12.39 & 42.12 & 30.63 & 44.26 & 20.57 \\
        & Qwen3-1.7B*              & 6.46  & 28.13 & 15.05 & 11.20 & 18.00 & 22.30 \\
        & Qwen3-4B*                & 11.93 & 42.20 & 24.17 & 19.06 & 28.33 & 34.15 \\
        & Qwen3-1.7B               & 23.17 & 35.82 & 48.47 & 34.39 & 49.21 & 36.95 \\
        & Qwen3-8B*                & 11.41 & 47.44 & 28.44 & 20.42 & 30.94 & 38.14 \\
        & Qwen3-14B*               & 12.76 & 54.24 & 30.05 & 21.64 & 34.02 & 42.14 \\
        & Qwen3-32B*               & 9.80  & 54.24 & 30.05 & 21.64 & 34.02 & 42.78 \\
        & Qwen3-4B                 & 43.83 & 53.19 & 75.25 & 63.44 & 66.28 & 56.59 \\
        & Qwen3-8B                 & 44.53 & 57.34 & 74.90 & 62.71 & 70.73 & 59.56 \\
        & Qwen3-14B                & 50.41 & 64.86 & 81.07 & 70.63 & 76.50 & 66.67 \\
        & DeepSeek-R1-8B           & 59.22 & 62.71 & 82.86 & 75.07 & 79.84 & 67.06 \\
        & Qwen3-32B                & 53.15 & 69.62 & 81.97 & 70.15 & 79.56 & 70.22 \\
      \bottomrule
    \end{tabular}
  }
  \end{center}
\end{table}

\subsection{Detailed Results on the Dynamic Impact of SelfStepConf on the Inference Process}
\label{sec:appendix_complete_token_analysis}
    In the main text \autoref{tab:analysis_ssc_token} of \autoref{sec:analysis_ssc_token}, we briefly analyzed the impact of \textit{SelfStepConf} on the inference process in terms of the number of steps, number of tokens, confidence, and generate time. Here, we additionally provide time analysis for each benchmark. In addition to DeepSeek-R1-8B, we also provide results for Qwen3-8B, as shown in \autoref{tab:appendix_complete_token_analysis}.
    \begin{table}[H]
  \caption{Complete results of response length changes between SSC and BasicInference. Results are obtained at \textbf{temperature=0}, generating one response for each query in every benchmark, with values averaged across all queries within each benchmark (Time: ms/it).}
  \label{tab:appendix_complete_token_analysis}
  \begin{center}
    \resizebox{0.9\textwidth}{!}{
    \begin{tabular}{llcccccccc}
      \toprule
      \multirow{2}{*}{Model} & \multirow{2}{*}{Benchmark} & \multicolumn{4}{c}{BasicInference} & \multicolumn{4}{c}{SelfStepConf} \\
      \cmidrule(lr){3-6} \cmidrule(lr){7-10}
                &           & Step & Token & Confs. & Time & Step & Token & Confs. & Time \\
      \midrule
      \multirow{7}{*}{DeepSeek-R1-8B}
      & HMMT2025  & 154.00 & 28266.73 & 17.03 & 396.07 & 154.40 & 28604.80 & 17.20 & 417.11 \\
      & GPQA-D    & 30.80  & 9560.65  & 13.31 & 124.27 & 29.27  & 9411.71  & 13.33 & 127.17 \\
      & AIME2024  & 88.63  & 21239.20 & 17.96 & 289.66 & 83.33  & 20733.97 & 17.92 & 295.24 \\
      & AIME2025  & 123.23 & 26673.87 & 17.87 & 374.35 & 128.10 & 26280.73 & 17.91 & 382.20 \\
      & BRUMO2025 & 135.40 & 23137.50 & 17.44 & 321.32 & 124.50 & 22291.60 & 17.45 & 318.73 \\
      \cmidrule{2-10}
      & Avg.      & 66.47  & 15322.42 & 14.92 & 207.70 & 64.48  & 15097.02 & 14.95 & 212.51 \\
      \midrule
      \midrule
      \multirow{7}{*}{Qwen3-8B}
      & HMMT2025  & 50.03  & 22114.77 & 15.13 & 302.55 & 622.63 & 24983.03 & 15.08 & 357.86 \\
      & GPQA-D    & 18.68  & 9757.86  & 13.97 & 130.35 & 60.47  & 5243.37  & 14.89 & 76.72 \\
      & AIME2024  & 39.60  & 15649.53 & 16.59 & 230.83 & 78.73  & 19586.87 & 16.24 & 279.85 \\
      & AIME2025  & 58.60  & 19786.97 & 16.08 & 270.24 & 97.63  & 20736.83 & 15.93 & 294.75 \\
      & BRUMO2025 & 54.77  & 16939.90 & 16.50 & 233.05 & 183.00 & 20881.97 & 15.90 & 297.39 \\
      \cmidrule{2-10}
      & Avg.      & 30.78  & 13103.12 & 14.76 & 178.96 & 130.29 & 11395.75 & 15.23 & 163.79 \\
      \bottomrule
    \end{tabular}
    }
  \end{center}
  \vskip -0.3in
\end{table}

\newpage
\subsection{Detailed Results of Sensitivity Analysis for Parameter $\alpha$}
\label{sec:appendix_complete_parameters_alpha_analysis}
    For the sensitivity analysis results of the parameter $\alpha$ shown in  of \autoref{sec:appendix_parameters_analysis_alpha}, we provide the complete experimental data here. This mainly includes the voting results of Pass@1, WSC-GMM* and DIS-GMM* using DeepSeek-R1-8B with $\alpha$ ranging from 0.1 to 0.9, as shown in \autoref{tab:appendix_complete_parameters_alpha_analysis}.
    \begin{table}[H]
    \caption{Complete of sensitivity analysis of $\alpha$ using DeepSeek-R1-8B, sampling 128 trajectories/query (64 repeats), setting $\delta=0.8$.}
    \label{tab:appendix_complete_parameters_alpha_analysis}
    \begin{center}
    \resizebox{0.98\textwidth}{!}{
    \begin{tabular}{llccccccccc}
      \toprule
        \multirow{2}{*}{Metric} & \multirow{2}{*}{Benchmark} & \multicolumn{9}{c}{$\delta$} \\
        \cmidrule(lr){3-11}
        & & 0.1 & 0.2 & 0.3 & 0.4 & 0.5 & 0.6 & 0.7 & 0.8 & 0.9 \\
        \midrule
        \midrule
        \multicolumn{11}{c}{\textbf{\textit{Accuracy Metric}}} \\
        \midrule
        \midrule
        \multirow{7}{*}{Pass@1}
        & HMMT2025   & 58.65 & 58.44 & 58.39 & 58.49 & 58.53 & 58.32 & 58.69 & 59.22 & 57.92 \\
        & GPQA-D    & 61.87 & 61.90 & 61.96 & 62.00 & 62.03 & 61.79 & 60.54 & 62.71 & 61.70 \\
        & AIME2024    & 83.65 & 83.39 & 83.07 & 83.47 & 83.45 & 82.66 & 83.29 & 82.86 & 82.34 \\
        & AIME2025    & 73.13 & 73.54 & 73.49 & 73.39 & 71.65 & 74.61 & 74.23 & 75.07 & 72.86 \\
        & BRUMO2025    & 78.54 & 78.80 & 78.80 & 78.70 & 79.12 & 78.12 & 79.12 & 79.84 & 78.70 \\
        \cmidrule{2-11}
        & Avg.    & 66.25 & 66.29 & 66.29 & 66.35 & 66.24 & 66.18 & 65.58 & 67.06 & 65.95 \\
        \midrule
        \multirow{7}{*}{WSC-GMM*}
        & HMMT2025    & 79.90 & 77.34 & 80.31 & 80.05 & 84.53 & 78.47 & 79.08 & 84.17 & 82.40 \\
        & GPQA-D    & 69.38 & 69.55 & 69.18 & 68.28 & 70.00 & 67.26 & 70.75 & 70.11 & 69.52 \\
        & AIME2024    & 93.33 & 93.33 & 93.33 & 93.33 & 93.32 & 93.32 & 93.32 & 93.33 & 93.28 \\
        & AIME2025    & 80.00 & 79.90 & 80.94 & 80.00 & 78.57 & 80.39 & 83.64 & 84.38 & 80.00 \\
        & BRUMO2025    & 93.33 & 93.33 & 93.33 & 93.33 & 96.29 & 91.05 & 93.07 & 94.17 & 93.33 \\
        \cmidrule{2-11}
        & Avg.    & 75.89 & 75.75 & 75.90 & 75.22 & 76.84 & 74.27 & 77.00 & 77.24 & 76.21 \\
        \midrule
        \multirow{7}{*}{DIS-GMM*}
        & HMMT2025    & 82.97 & 76.77 & 82.60 & 80.05 & 86.09 & 77.85 & 79.61 & 84.95 & 79.74 \\
        & GPQA-D    & 69.40 & 68.70 & 69.60 & 68.81 & 70.35 & 67.95 & 69.85 & 70.63 & 69.14 \\
        & AIME2024    & 93.33 & 93.33 & 93.33 & 93.33 & 93.30 & 93.30 & 93.30 & 93.23 & 93.23 \\
        & AIME2025    & 80.00 & 79.90 & 83.02 & 80.00 & 79.99 & 81.95 & 82.75 & 86.46 & 80.00 \\
        & BRUMO2025    & 93.33 & 93.33 & 93.33 & 93.33 & 95.16 & 92.63 & 92.68 & 94.27 & 93.33 \\
        \cmidrule{2-11}
        & Avg.    & 76.19 & 75.17 & 76.57 & 75.55 & 77.23 & 74.94 & 76.36 & 77.84 & 75.72 \\
        \midrule
        \midrule
        \multicolumn{11}{c}{\textbf{\textit{Analysis Metric}}} \\
        \midrule
        \midrule
        \multirow{7}{*}{Steps Count}
        & HMMT2025    & 151.20 & 150.11 & 148.40 & 149.17 & 147.18 & 150.51 & 146.20 & 140.14 & 148.77 \\
        & GPQA-D    & 31.33 & 31.38 & 31.46 & 31.21 & 31.65 & 31.14 & 30.98 & 31.13 & 31.04 \\
        & AIME2024    & 109.74 & 109.63 & 108.43 & 108.38 & 107.26 & 110.16 & 109.68 & 109.60 & 108.42 \\
        & AIME2025    & 131.01 & 130.16 & 129.28 & 129.89 & 132.58 & 123.98 & 132.33 & 127.08 & 130.38 \\
        & BRUMO2025    & 139.42 & 139.39 & 137.71 & 136.89 & 137.66 & 137.69 & 138.50 & 134.88 & 139.41 \\
        \cmidrule{2-11}
        & Avg.    & 69.64 & 69.47 & 69.01 & 68.90 & 69.24 & 68.70 & 68.92 & 67.66 & 69.04 \\
        \midrule
        \multirow{7}{*}{Tokens Count}
        & HMMT2025    & 29442.61 & 29314.99 & 29366.06 & 29318.86 & 29319.90 & 29541.56 & 29137.55 & 29368.05 & 29187.45 \\
        & GPQA-D    & 9745.97 & 9751.10 & 9741.49 & 9735.64 & 9722.62 & 9815.61 & 9689.64 & 9721.65 & 9759.88 \\
        & AIME2024    & 23005.22 & 23026.91 & 22870.45 & 22939.62 & 22745.89 & 23280.70 & 23119.20 & 23438.78 & 23010.60 \\
        & AIME2025    & 26444.03 & 26225.54 & 26356.69 & 26201.99 & 26843.64 & 25884.50 & 26175.51 & 26367.72 & 26211.33 \\
        & BRUMO2025    & 23355.66 & 23281.21 & 23119.51 & 23116.11 & 23081.49 & 23330.81 & 23262.30 & 23270.82 & 23205.88 \\
        \cmidrule{2-11}
        & Avg.    & 15714.24 & 15679.80 & 15660.99 & 15644.51 & 15671.50 & 15745.95 & 15622.83 & 15717.76 & 15663.25 \\
        \midrule
        \multirow{7}{*}{Confidence}
        & HMMT2025    & 15.77 & 15.79 & 15.76 & 15.77 & 15.78 & 15.76 & 15.82 & 15.83 & 15.78 \\
        & GPQA-D    & 12.43 & 12.43 & 12.43 & 12.44 & 12.43 & 12.42 & 12.43 & 12.42 & 12.43 \\
        & AIME2024    & 16.45 & 16.46 & 16.45 & 16.49 & 16.52 & 16.43 & 16.43 & 16.45 & 16.45 \\
        & AIME2025    & 16.49 & 16.51 & 16.50 & 16.53 & 16.46 & 16.62 & 16.49 & 16.58 & 16.54 \\
        & BRUMO2025    & 16.42 & 16.45 & 16.48 & 16.48 & 16.50 & 16.43 & 16.47 & 16.51 & 16.45 \\
        \cmidrule{2-11}
        & Avg.    & 13.88 & 13.89 & 13.89 & 13.90 & 13.90 & 13.89 & 13.89 & 13.90 & 13.89 \\
        \midrule
        \multirow{7}{*}{Time (ms/it)}
        & HMMT2025    & 20.15 & 20.27 & 20.11 & 19.93 & 19.03 & 19.50 & 18.81 & 14.31 & 19.90 \\
        & GPQA-D    & 5.87 & 5.89 & 5.87 & 5.79 & 5.54 & 5.66 & 5.53 & 4.17 & 5.86 \\
        & AIME2024    & 15.31 & 15.53 & 15.36 & 15.40 & 14.34 & 15.15 & 14.75 & 11.47 & 15.40 \\
        & AIME2025    & 18.01 & 17.94 & 17.83 & 17.42 & 17.37 & 16.52 & 16.82 & 12.73 & 17.50 \\
        & BRUMO2025    & 15.61 & 15.64 & 15.51 & 15.24 & 14.65 & 14.95 & 14.78 & 11.39 & 15.38 \\
        \cmidrule{2-11}
        & Avg.    & 10.17 & 10.21 & 10.15 & 10.02 & 9.62 & 9.76 & 9.59 & 7.30 & 10.08 \\
      \bottomrule
    \end{tabular}
    }
    \end{center}
\end{table}

\newpage
\subsection{Detailed Results of Sensitivity Analysis for Parameter $\delta$}
\label{sec:appendix_complete_parameters_delta_analysis}
    For the sensitivity analysis results of the parameter $\delta$ shown in  of \autoref{sec:appendix_parameters_analysis_delta}, we provide the complete experimental data here. This mainly includes the voting results of Pass@1, WSC-GMM* and DIS-GMM* using DeepSeek-R1-8B with $\delta$ ranging from 0.1 to 0.9, as shown in \autoref{tab:appendix_complete_parameters_delta_analysis}.
    
    \begin{table}[H]
    \caption{Complete of sensitivity analysis of $\delta$ using DeepSeek-R1-8B, sampling 128 trajectories/query(64 repeats), setting $\alpha=0.8$.}
    \label{tab:appendix_complete_parameters_delta_analysis}
    \begin{center}
    \resizebox{0.95\textwidth}{!}{
\begin{tabular}{llccccccccc}
      \toprule
        \multirow{2}{*}{Metric} & \multirow{2}{*}{Benchmark} & \multicolumn{9}{c}{$\delta$} \\
        \cmidrule(lr){3-11}
        & & 0.1 & 0.2 & 0.3 & 0.4 & 0.5 & 0.6 & 0.7 & 0.8 & 0.9 \\
        \midrule
        \midrule
        \multicolumn{11}{c}{\textbf{\textit{Accuracy Metric}}} \\
        \midrule
        \midrule
        \multirow{7}{*}{Pass@1}
        & HMMT2025   & 58.96 & 58.96 & 58.96 & 59.13 & 59.17 & 57.92 & 57.86 & 59.22 & 58.28 \\
        & GPQA-D    & 61.77 & 61.64 & 61.67 & 61.73 & 61.65 & 61.56 & 61.71 & 62.71 & 61.88 \\
        & AIME2024    & 82.66 & 82.66 & 82.66 & 82.71 & 82.45 & 82.71 & 82.86 & 82.86 & 82.97 \\
        & AIME2025    & 73.02 & 73.02 & 73.02 & 73.02 & 73.13 & 73.80 & 72.24 & 75.07 & 73.80 \\
        & BRUMO2025    & 81.11 & 81.11 & 81.11 & 81.28 & 81.28 & 78.59 & 78.07 & 79.84 & 79.38 \\
        \cmidrule{2-11}
        & Avg.    & 66.36 & 66.28 & 66.30 & 66.38 & 66.31 & 65.97 & 65.88 & 67.06 & 66.31 \\
        \midrule
        \multirow{7}{*}{WSC-GMM*}
        & HMMT2025    & 82.03 & 82.03 & 82.03 & 78.70 & 80.42 & 77.24 & 79.90 & 84.17 & 80.16 \\
        & GPQA-D    & 69.32 & 68.59 & 68.59 & 68.57 & 69.86 & 68.84 & 69.83 & 70.11 & 71.09 \\
        & AIME2024    & 93.07 & 93.07 & 93.07 & 93.07 & 93.07 & 93.33 & 93.33 & 93.33 & 93.33 \\
        & AIME2025    & 86.67 & 86.67 & 86.67 & 86.67 & 86.67 & 80.05 & 80.00 & 84.38 & 82.29 \\
        & BRUMO2025    & 96.61 & 96.61 & 96.61 & 96.61 & 96.61 & 96.04 & 93.33 & 94.17 & 93.33 \\
        \cmidrule{2-11}
        & Avg.    & 76.97 & 76.52 & 76.52 & 76.19 & 77.15 & 75.57 & 76.17 & 77.24 & 77.20 \\
        \midrule
        \multirow{7}{*}{DIS-GMM*}
        & HMMT2025    & 82.08 & 82.08 & 82.08 & 78.75 & 80.21 & 77.45 & 79.95 & 84.95 & 80.78 \\
        & GPQA-D    & 68.92 & 69.71 & 69.71 & 69.77 & 70.39 & 70.19 & 70.79 & 70.63 & 71.27 \\
        & AIME2024    & 93.23 & 93.23 & 93.23 & 93.23 & 93.23 & 93.33 & 93.28 & 93.23 & 93.33 \\
        & AIME2025    & 86.67 & 86.67 & 86.67 & 86.72 & 86.56 & 80.05 & 80.00 & 86.46 & 82.19 \\
        & BRUMO2025    & 96.61 & 96.61 & 96.61 & 96.61 & 96.61 & 95.42 & 93.33 & 94.27 & 93.33 \\
        \cmidrule{2-11}
        & Avg.    & 76.74 & 77.24 & 77.24 & 76.96 & 77.47 & 76.37 & 76.77 & 77.84 & 77.36 \\
        \midrule
        \midrule
        \multicolumn{11}{c}{\textbf{\textit{Analysis Metric}}} \\
        \midrule
        \midrule
        \multirow{7}{*}{Steps Count}
        & HMMT2025    & 149.39 & 149.39 & 149.33 & 153.73 & 149.22 & 150.52 & 149.19 & 140.14 & 159.92 \\
        & GPQA-D    & 31.04 & 32.94 & 33.05 & 32.83 & 33.05 & 32.42 & 31.93 & 31.13 & 33.24 \\
        & AIME2024    & 114.58 & 114.58 & 114.58 & 114.57 & 115.12 & 108.98 & 108.58 & 109.60 & 113.88 \\
        & AIME2025    & 128.66 & 128.66 & 128.66 & 129.03 & 128.94 & 132.75 & 133.61 & 127.08 & 135.34 \\
        & BRUMO2025    & 140.20 & 140.20 & 140.20 & 139.84 & 139.77 & 138.18 & 137.54 & 134.88 & 142.81 \\
        \cmidrule{2-11}
        & Avg.    & 69.59 & 70.78 & 70.84 & 71.12 & 70.87 & 70.23 & 69.78 & 67.66 & 72.77 \\
        \midrule
        \multirow{7}{*}{Tokens Count}
        & HMMT2025    & 29410.30 & 29410.30 & 29407.64 & 29480.44 & 29403.09 & 29389.67 & 29401.74 & 29368.05 & 30271.32 \\
        & GPQA-D    & 9640.65 & 9710.41 & 9719.70 & 9698.47 & 9715.40 & 9704.60 & 9718.37 & 9721.65 & 9972.02 \\
        & AIME2024    & 23414.50 & 23414.50 & 23414.50 & 23406.22 & 23467.06 & 22798.41 & 22774.72 & 23438.78 & 23779.66 \\
        & AIME2025    & 26129.52 & 26129.52 & 26129.52 & 26151.06 & 26183.19 & 26215.12 & 26387.48 & 26367.72 & 27142.56 \\
        & BRUMO2025    & 23290.87 & 23290.87 & 23290.87 & 23260.09 & 23226.01 & 23043.78 & 23038.24 & 23270.82 & 24036.34 \\
        \cmidrule{2-11}
        & Avg.    & 15648.44 & 15691.88 & 15697.41 & 15689.41 & 15698.21 & 15612.96 & 15636.17 & 15717.76 & 16136.34 \\
        \midrule
        \multirow{7}{*}{Confidence}
        & HMMT2025    & 15.91 & 15.91 & 15.91 & 15.95 & 15.90 & 15.91 & 15.90 & 15.83 & 15.02 \\
        & GPQA-D    & 12.48 & 12.48 & 12.48 & 12.48 & 12.48 & 12.48 & 12.47 & 12.42 & 12.29 \\
        & AIME2024    & 16.54 & 16.54 & 16.54 & 16.54 & 16.53 & 16.59 & 16.58 & 16.45 & 15.80 \\
        & AIME2025    & 16.70 & 16.70 & 16.70 & 16.70 & 16.69 & 16.70 & 16.65 & 16.58 & 15.85 \\
        & BRUMO2025    & 16.60 & 16.60 & 16.60 & 16.61 & 16.60 & 16.64 & 16.63 & 16.51 & 15.78 \\
        \cmidrule{2-11}
        & Avg.    & 13.97 & 13.97 & 13.97 & 13.98 & 13.97 & 13.98 & 13.97 & 13.90 & 13.54 \\
        \midrule
        \multirow{7}{*}{Time (ms/it)}
        & HMMT2025    & 24.77 & 24.66 & 24.76 & 24.93 & 24.78 & 20.46 & 20.22 & 14.31 & 21.02 \\
        & GPQA-D    & 6.03 & 5.86 & 5.80 & 5.87 & 5.77 & 5.85 & 5.82 & 4.17 & 5.98 \\
        & AIME2024    & 18.61 & 18.58 & 18.63 & 18.62 & 18.69 & 15.14 & 15.24 & 11.47 & 16.18 \\
        & AIME2025    & 21.74 & 21.69 & 21.68 & 21.78 & 21.82 & 17.67 & 17.90 & 12.73 & 18.46 \\
        & BRUMO2025    & 18.79 & 18.78 & 18.78 & 18.80 & 18.79 & 15.60 & 15.51 & 11.39 & 16.18 \\
        \cmidrule{2-11}
        & Avg.    & 11.67 & 11.55 & 11.52 & 11.59 & 11.52 & 10.14 & 10.12 & 7.30 & 10.50 \\
      \bottomrule
    \end{tabular}
    }
    \end{center}
    \vskip -0.2in
\end{table}

\newpage
\subsection{Detailed Results of Sensitivity Analysis for Parameter $\alpha \times \delta$}
\label{sec:appendix_complete_parameters_alphaDelta_analysis}
    For the sensitivity analysis results of the parameter $\alpha \times \delta$ shown in  of \autoref{sec:appendix_parameters_analysis_delta}, we provide the complete experimental data here. This mainly includes the voting results of Pass@1, WSC-GMM* and DIS-GMM* using DeepSeek-R1-8B on HMMT2025 with $\alpha$ and $\delta$ ranging from 0.1 to 0.9, as shown in \autoref{tab:appendix_complete_parameters_alphaDelta_analysis_1} and \autoref{tab:appendix_complete_parameters_alphaDelta_analysis_2}.

    \begin{table}[H]
  \vskip -0.1in
  \caption{Complete results of parameter sensitivity analysis of $\alpha \times \delta$ on HMMT2025 using DeepSeek-R1-8B, sampling 128 trajectories/query (64 repeats). \textbf{Evaluate using Analysis Metric.}}
  \label{tab:appendix_complete_parameters_alphaDelta_analysis_1}
  \begin{center}
    \resizebox{0.98\textwidth}{!}{
    \begin{tabular}{llccccccccc}
      \toprule
        \multirow{2}{*}{Metric} & \multirow{2}{*}{$\alpha$} & \multicolumn{9}{c}{$\delta$} \\
        \cmidrule(lr){3-11}
        & & 0.1 & 0.2 & 0.3 & 0.4 & 0.5 & 0.6 & 0.7 & 0.8 & 0.9 \\ 
        \midrule
        \multirow{9}{*}{Steps Count}
        & 0.1 & 149.39 & 149.39 & 149.39 & 149.34 & 150.66 & 150.36 & 148.73 & 151.20 & 160.66 \\
        & 0.2 & 149.39 & 149.39 & 149.39 & 149.39 & 150.15 & 150.25 & 149.21 & 150.11 & 158.18 \\
        & 0.3 & 149.39 & 149.39 & 149.39 & 149.39 & 150.21 & 150.61 & 149.34 & 148.40 & 157.68 \\
        & 0.4 & 149.39 & 149.39 & 149.39 & 149.65 & 150.21 & 150.67 & 149.60 & 149.17 & 158.09 \\
        & 0.5 & 149.39 & 149.39 & 149.39 & 149.65 & 150.21 & 151.46 & 151.98 & 148.97 & 158.97 \\
        & 0.6 & 149.39 & 149.39 & 149.39 & 149.65 & 149.73 & 163.15 & 147.19 & 152.34 & 157.30 \\
        & 0.7 & 149.39 & 149.39 & 149.39 & 149.65 & 149.22 & 148.59 & 150.93 & 147.97 & 158.29 \\
        & 0.8 & 149.39 & 149.39 & 149.33 & 149.65 & 149.69 & 150.52 & 149.19 & 140.14 & 159.92 \\
        & 0.9 & 149.39 & 148.07 & 149.39 & 153.73 & 149.69 & 148.81 & 149.64 & 148.77 & 156.52 \\
        \midrule
        \multirow{9}{*}{Tokens Count}
        & 0.1 & 29410.30 & 29410.30 & 29410.30 & 29410.69 & 29496.48 & 29434.60 & 29210.27 & 29442.61 & 30491.65 \\
        & 0.2 & 29410.30 & 29410.30 & 29410.30 & 29410.30 & 29441.36 & 29435.32 & 29317.60 & 29314.99 & 30306.47 \\
        & 0.3 & 29410.30 & 29410.30 & 29410.30 & 29410.30 & 29447.65 & 29462.27 & 29310.82 & 29366.06 & 30401.81 \\
        & 0.4 & 29410.30 & 29410.30 & 29410.30 & 29436.85 & 29447.65 & 29475.31 & 29319.77 & 29318.86 & 30287.35 \\
        & 0.5 & 29410.30 & 29410.30 & 29410.30 & 29436.85 & 29447.65 & 29520.59 & 29497.41 & 29637.60 & 30471.49 \\
        & 0.6 & 29410.30 & 29410.30 & 29410.30 & 29436.85 & 29355.62 & 29693.73 & 29325.25 & 29389.67 & 31158.71 \\
        & 0.7 & 29410.30 & 29410.30 & 29410.30 & 29436.85 & 29403.09 & 29325.25 & 29500.52 & 29453.27 & 30498.60 \\
        & 0.8 & 29410.30 & 29410.30 & 29407.64 & 29480.44 & 29423.84 & 29389.67 & 29401.74 & 29368.05 & 30271.32 \\
        & 0.9 & 29410.30 & 29481.11 & 29410.30 & 29436.85 & 29423.84 & 29395.96 & 29435.42 & 29187.45 & 30187.49 \\
        \midrule
        \multirow{9}{*}{Confidence}
        & 0.1 & 15.91 & 15.91 & 15.91 & 15.91 & 15.89 & 15.93 & 15.96 & 15.77 & 14.96 \\
        & 0.2 & 15.91 & 15.91 & 15.91 & 15.91 & 15.92 & 15.93 & 15.93 & 15.79 & 15.02 \\
        & 0.3 & 15.91 & 15.91 & 15.91 & 15.91 & 15.91 & 15.92 & 15.93 & 15.76 & 14.97 \\
        & 0.4 & 15.91 & 15.91 & 15.91 & 15.90 & 15.91 & 15.93 & 15.93 & 15.77 & 14.98 \\
        & 0.5 & 15.91 & 15.91 & 15.91 & 15.90 & 15.91 & 15.91 & 15.89 & 15.61 & 14.96 \\
        & 0.6 & 15.91 & 15.91 & 15.91 & 15.90 & 15.90 & 15.91 & 15.89 & 15.59 & 14.85 \\
        & 0.7 & 15.91 & 15.91 & 15.91 & 15.90 & 15.90 & 15.95 & 15.86 & 15.65 & 14.98 \\
        & 0.8 & 15.91 & 15.91 & 15.91 & 15.95 & 15.90 & 15.91 & 15.90 & 15.83 & 15.02 \\
        & 0.9 & 15.91 & 15.91 & 15.91 & 15.90 & 15.90 & 15.89 & 15.90 & 15.78 & 14.99 \\
        \midrule
        \multirow{9}{*}{Reflections Count}
        & 0.1 & 0 & 0 & 0.000529 & 0.002092 & 0.010502 & 0.034449 & 0.321665 & 4.01803 & 28.00328 \\
        & 0.2 & 0 & 0 & 0.000529 & 0.002092 & 0.007813 & 0.034449 & 0.302839 & 3.686321 & 26.48138 \\
        & 0.3 & 0 & 0 & 0.000529 & 0.002092 & 0.007292 & 0.033929 & 0.271697 & 3.403557 & 24.61753 \\
        & 0.4 & 0 & 0 & 0.000529 & 0.002092 & 0.007292 & 0.031845 & 0.25119 & 3.227719 & 22.9203 \\
        & 0.5 & 0 & 0 & 0.000529 & 0.002092 & 0.007292 & 0.031845 & 0.225324 & 2.37307 & 24.35578 \\
        & 0.6 & 0 & 0 & 0.000529 & 0.002092 & 0.007292 & 0.035053 & 0.184679 & 2.428287 & 24.59583 \\
        & 0.7 & 0 & 0 & 0.000529 & 0.002612 & 0.008854 & 0.040833 & 0.163306 & 2.306881 & 24.03886 \\
        & 0.8 & 0 & 0 & 0.000529 & 0.002092 & 0.010475 & 0.028646 & 0.195833 & 2.746875 & 26.40324 \\
        & 0.9 & 0 & 0 & 0.000529 & 0.002092 & 0.010995 & 0.039674 & 0.185857 & 3.375634 & 26.08073 \\

        \midrule
        \multirow{9}{*}{Time (ms/it)}
        & 0.1 & 24.67 & 24.73 & 24.67 & 24.70 & 24.79 & 18.61 & 18.45 & 20.15 & 19.45 \\
        & 0.2 & 24.65 & 24.68 & 24.71 & 24.69 & 18.60 & 18.66 & 18.59 & 20.27 & 19.12 \\
        & 0.3 & 24.65 & 24.70 & 24.66 & 24.68 & 18.60 & 18.56 & 18.41 & 20.11 & 25.58 \\
        & 0.4 & 24.70 & 24.74 & 24.72 & 24.70 & 18.56 & 18.61 & 18.57 & 19.93 & 28.05 \\
        & 0.5 & 24.68 & 24.68 & 24.68 & 24.65 & 18.61 & 18.55 & 21.54 & 51.62 & 25.64 \\
        & 0.6 & 24.73 & 24.60 & 24.66 & 24.75 & 18.62 & 25.07 & 16.74 & 52.89 & 30.53 \\
        & 0.7 & 24.63 & 24.67 & 24.61 & 24.76 & 20.12 & 17.78 & 24.72 & 51.04 & 25.59 \\
        & 0.8 & 24.77 & 24.66 & 24.76 & 24.93 & 24.78 & 20.46 & 20.22 & 14.31 & 21.02 \\
        & 0.9 & 24.94 & 27.79 & 24.85 & 24.93 & 24.69 & 19.90 & 16.94 & 19.90 & 25.21 \\
      \bottomrule
    \end{tabular}
    }
  \end{center}
\end{table}

\begin{table}[ht]
  \caption{Complete results of parameter sensitivity analysis of $\alpha \times \delta$ on HMMT2025 using DeepSeek-R1-8B, sampling 128 trajectories/query (64 repeats). \textbf{Evaluate using Accuracy Metric.}}
  \label{tab:appendix_complete_parameters_alphaDelta_analysis_2}
  \begin{center}
    \resizebox{0.9\textwidth}{!}{
    \begin{tabular}{llccccccccc}
      \toprule
        \multirow{2}{*}{Metric} & \multirow{2}{*}{$\alpha$} & \multicolumn{9}{c}{$\delta$} \\
        \cmidrule(lr){3-11}
        & & 0.1 & 0.2 & 0.3 & 0.4 & 0.5 & 0.6 & 0.7 & 0.8 & 0.9 \\ 
        \midrule
        \multirow{9}{*}{Pass@1}
        & 0.1 & 58.96 & 58.96 & 58.96 & 58.96 & 59.17 & 58.54 & 59.22 & 58.65 & 58.13 \\
        & 0.2 & 58.96 & 58.96 & 58.96 & 58.96 & 58.49 & 58.54 & 59.06 & 58.44 & 58.96 \\
        & 0.3 & 58.96 & 58.96 & 58.96 & 58.96 & 58.49 & 58.59 & 59.22 & 58.39 & 57.76 \\
        & 0.4 & 58.96 & 58.96 & 58.96 & 58.91 & 58.49 & 58.59 & 58.18 & 58.49 & 57.24 \\
        & 0.5 & 58.96 & 58.96 & 58.96 & 58.91 & 58.49 & 58.70 & 58.28 & 57.97 & 57.55 \\
        & 0.6 & 58.96 & 58.96 & 58.96 & 58.91 & 58.49 & 58.54 & 58.02 & 57.76 & 56.98 \\
        & 0.7 & 58.96 & 58.96 & 58.96 & 58.91 & 58.49 & 58.54 & 58.75 & 58.13 & 57.34 \\
        & 0.8 & 58.96 & 58.96 & 58.96 & 59.13 & 59.17 & 57.92 & 57.86 & 59.22 & 58.28 \\
        & 0.9 & 58.96 & 58.96 & 58.96 & 58.91 & 58.49 & 58.54 & 58.02 & 57.92 & 58.28 \\
        \midrule
        \multirow{9}{*}{WSC-GMM*}
        & 0.1 & 82.03 & 82.03 & 82.03 & 82.03 & 82.29 & 79.74 & 84.95 & 79.90 & 79.84 \\
        & 0.2 & 82.03 & 82.03 & 82.03 & 82.03 & 79.58 & 79.74 & 79.27 & 77.34 & 79.84 \\
        & 0.3 & 82.03 & 82.03 & 82.03 & 82.03 & 79.58 & 79.74 & 80.05 & 80.31 & 79.90 \\
        & 0.4 & 82.03 & 82.03 & 82.03 & 82.76 & 79.58 & 79.74 & 79.90 & 80.05 & 78.28 \\
        & 0.5 & 82.03 & 82.03 & 82.03 & 82.76 & 79.58 & 79.64 & 76.20 & 79.90 & 79.27 \\
        & 0.6 & 82.03 & 82.03 & 82.03 & 82.76 & 79.58 & 82.19 & 85.05 & 74.17 & 73.28 \\
        & 0.7 & 82.03 & 82.03 & 82.03 & 82.76 & 79.79 & 76.77 & 83.28 & 74.74 & 84.32 \\
        & 0.8 & 82.03 & 82.03 & 82.03 & 78.70 & 80.42 & 77.24 & 79.90 & 84.17 & 80.16 \\
        & 0.9 & 82.03 & 82.03 & 82.03 & 82.03 & 72.92 & 78.88 & 80.42 & 82.40 & 78.91 \\
        \midrule
        \multirow{9}{*}{DIS-GMM*}
        & 0.1 & 82.08 & 82.08 & 82.08 & 82.08 & 81.98 & 79.22 & 84.01 & 82.97 & 76.88 \\
        & 0.2 & 82.08 & 82.08 & 82.08 & 82.08 & 78.85 & 79.22 & 79.01 & 76.77 & 80.73 \\
        & 0.3 & 82.08 & 82.08 & 82.08 & 82.08 & 78.85 & 79.22 & 80.05 & 82.60 & 75.94 \\
        & 0.4 & 82.08 & 82.08 & 82.08 & 82.70 & 78.85 & 79.22 & 80.05 & 80.05 & 79.74 \\
        & 0.5 & 82.08 & 82.08 & 82.08 & 82.70 & 78.85 & 79.22 & 78.49 & 81.67 & 80.21 \\
        & 0.6 & 82.08 & 82.08 & 82.08 & 82.70 & 79.79 & 83.28 & 85.26 & 73.85 & 75.47 \\
        & 0.7 & 82.08 & 82.08 & 82.08 & 82.70 & 80.21 & 74.64 & 83.23 & 75.52 & 83.96 \\
        & 0.8 & 82.08 & 82.08 & 82.08 & 78.75 & 80.21 & 77.45 & 79.95 & 84.95 & 80.78 \\
        & 0.9 & 82.08 & 82.08 & 82.08 & 82.08 & 77.08 & 78.29 & 80.31 & 79.74 & 79.27 \\
      \bottomrule
    \end{tabular}
    }
  \end{center}
  \vskip -0.1in
\end{table}

\newpage
\subsection{Detailed Results of Sensitivity Analysis for Parameter $N_\mathcal{C}$}
\label{sec:appendix_complete_parameters_Nc_analysis}
    For the sensitivity analysis results of the parameter $N_{\mathcal{C}}$ shown in \autoref{fig:appendix_parameters_analysis_Nc} of \autoref{sec:appendix_parameters_analysis_Nc}, we provide the complete experimental data here. This mainly includes the voting results of \textit{DistriVoting} under different filter methods on DeepSeek-R1-8B with $N_{\mathcal{C}}$ ranging from 1 to 20, as shown in \autoref{tab:appendix_complete_parameters_Nc_analysis_1} and \autoref{tab:appendix_complete_parameters_Nc_analysis_2}.

\begin{table}[H]
  \vskip -0.1in
  \caption{Complete results of parameter sensitivity analysis of $N_\mathcal{C}$ under different voting methods. Using DeepSeek-R1-8B to generate 128 trajectories per query for voting, with experiments repeated 64 times ($N_\mathcal{C} \in [1,10]$).}
  \label{tab:appendix_complete_parameters_Nc_analysis_1}
  \begin{center}
    \resizebox{0.9\textwidth}{!}{
    \begin{tabular}{llcccccccccc}
      \toprule
        Method & Benchmark & 1 & 2 & 3 & 4 & 5 & 6 & 7 & 8 & 9 & 10 \\
        \midrule
        \multirow{6}{*}{DIS-Top50}
        & HMMT2025 & 74.17 & 78.39 & 80.63 & 80.31 & 80.36 & 79.95 & 81.09 & 79.53 & 80.89 & 79.64 \\
        & GPQA-D & 68.55 & 70.08 & 69.74 & 69.65 & 69.55 & 69.59 & 69.44 & 69.48 & 69.41 & 69.31 \\
        & AIME2024 & 89.38 & 92.08 & 93.18 & 93.18 & 93.18 & 93.07 & 93.13 & 93.07 & 93.13 & 93.07 \\
        & AIME2025 & 82.45 & 84.27 & 84.27 & 84.06 & 84.43 & 84.79 & 84.38 & 84.22 & 84.38 & 84.38 \\
        & BRUMO2025 & 93.33 & 93.96 & 93.65 & 93.70 & 93.70 & 93.59 & 93.44 & 93.44 & 93.49 & 93.39 \\
        \cmidrule{2-12}
        & Avg. & 74.70 & 76.53 & 76.60 & 76.50 & 76.48 & 76.48 & 76.45 & 76.30 & 76.42 & 76.22 \\
        \midrule
        \midrule
        \multirow{6}{*}{DIS-GMM}
        & HMMT2025 & 83.02 & 80.42 & 81.61 & 81.88 & 82.08 & 82.08 & 82.81 & 82.55 & 82.66 & 82.55 \\
        & GPQA-D & 69.44 & 69.84 & 70.06 & 69.79 & 69.67 & 69.66 & 69.48 & 69.57 & 69.58 & 69.63 \\
        & AIME2024 & 93.28 & 93.23 & 93.23 & 93.33 & 93.28 & 93.28 & 93.28 & 93.28 & 93.33 & 93.28 \\
        & AIME2025 & 86.72 & 84.74 & 85.63 & 85.78 & 86.09 & 85.78 & 85.68 & 85.83 & 86.04 & 85.73 \\
        & BRUMO2025 & 93.70 & 94.06 & 94.01 & 93.96 & 93.91 & 93.85 & 94.01 & 94.11 & 93.96 & 93.85 \\
        \cmidrule{2-12}
        & Avg. & 76.89 & 76.73 & 77.06 & 76.94 & 76.90 & 76.86 & 76.83 & 76.88 & 76.91 & 76.89 \\
        \midrule
        \midrule
        \multirow{6}{*}{DIS-GMM*}
        & HMMT2025 & 85.05 & 82.45 & 82.50 & 82.86 & 83.07 & 83.02 & 83.23 & 83.54 & 83.75 & 83.75 \\
        & GPQA-D & 70.19 & 70.56 & 70.68 & 70.63 & 70.64 & 70.50 & 70.54 & 70.74 & 70.57 & 70.69 \\
        & AIME2024 & 92.86 & 92.97 & 93.02 & 93.07 & 93.13 & 93.23 & 93.18 & 93.18 & 93.18 & 93.18 \\
        & AIME2025 & 85.00 & 83.91 & 84.48 & 84.22 & 84.74 & 84.32 & 84.69 & 85.31 & 84.79 & 84.69 \\
        & BRUMO2025 & 94.74 & 95.68 & 95.57 & 95.21 & 94.84 & 95.05 & 94.79 & 94.74 & 94.69 & 94.64 \\
        \cmidrule{2-12}
        & Avg. & 77.45 & 77.42 & 77.55 & 77.50 & 77.55 & 77.45 & 77.50 & 77.70 & 77.56 & 77.62 \\
      \bottomrule
    \end{tabular}
    }
  \end{center}
  \vskip -0.2in
\end{table}

\begin{table}[ht]
  \caption{Complete results of parameter sensitivity analysis of $N_\mathcal{C}$ under different voting methods. Using DeepSeek-R1-8B to generate 128 trajectories per query for voting, with experiments repeated 64 times ($N_\mathcal{C} \in [11,20]$).}
  \label{tab:appendix_complete_parameters_Nc_analysis_2}
  \begin{center}
    \resizebox{0.9\textwidth}{!}{
    \begin{tabular}{llcccccccccc}
      \toprule
        Method & Benchmark & 11 & 12 & 13 & 14 & 15 & 16 & 17 & 18 & 19 & 20 \\
        \midrule
        \multirow{6}{*}{DIS-Top50}
        & HMMT2025 & 80.05 & 80.42 & 80.31 & 80.63 & 80.36 & 80.31 & 79.64 & 80.31 & 80.05 & 79.64 \\
        & GPQA-D & 69.28 & 69.55 & 69.47 & 69.54 & 69.33 & 69.53 & 69.36 & 69.35 & 69.44 & 69.28 \\
        & AIME2024 & 93.02 & 93.07 & 93.02 & 93.13 & 93.07 & 92.97 & 92.97 & 92.92 & 92.81 & 92.76 \\
        & AIME2025 & 84.32 & 84.22 & 84.43 & 84.11 & 83.91 & 83.59 & 83.96 & 84.01 & 83.54 & 83.39 \\
        & BRUMO2025 & 93.39 & 93.39 & 93.33 & 93.39 & 93.39 & 93.28 & 93.33 & 93.28 & 93.33 & 93.33 \\
        \cmidrule{2-12}
        & Avg. & 76.23 & 76.42 & 76.38 & 76.43 & 76.25 & 76.33 & 76.19 & 76.25 & 76.23 & 76.07 \\
        \midrule
        \midrule
        \multirow{6}{*}{DIS-GMM}
        & HMMT2025 & 83.59 & 83.02 & 83.02 & 83.28 & 83.39 & 83.18 & 83.49 & 83.13 & 83.23 & 82.86 \\
        & GPQA-D & 69.66 & 69.70 & 69.67 & 69.65 & 69.76 & 69.87 & 69.74 & 69.71 & 69.67 & 69.67 \\
        & AIME2024 & 93.28 & 93.28 & 93.33 & 93.28 & 93.28 & 93.28 & 93.28 & 93.28 & 93.28 & 93.28 \\
        & AIME2025 & 86.51 & 86.30 & 85.94 & 86.15 & 85.83 & 85.94 & 85.99 & 85.78 & 85.89 & 85.94 \\
        & BRUMO2025 & 93.91 & 93.91 & 93.91 & 93.80 & 93.80 & 93.80 & 93.85 & 93.80 & 93.80 & 93.70 \\
        \cmidrule{2-12}
        & Avg. & 77.08 & 77.03 & 76.99 & 77.00 & 77.05 & 77.11 & 77.07 & 76.90 & 76.98 & 76.95 \\
        \midrule
        \midrule
        \multirow{6}{*}{DIS-GMM*}
        & HMMT2025 & 84.32 & 84.22 & 84.53 & 84.74 & 84.64 & 84.84 & 84.79 & 84.69 & 84.64 & 84.53 \\
        & GPQA-D & 70.73 & 70.63 & 70.69 & 70.56 & 70.66 & 70.79 & 70.63 & 70.62 & 70.58 & 70.45 \\
        & AIME2024 & 93.18 & 93.13 & 93.18 & 93.18 & 93.13 & 93.13 & 93.18 & 93.18 & 93.18 & 93.18 \\
        & AIME2025 & 84.69 & 84.43 & 84.38 & 84.17 & 84.69 & 84.17 & 84.58 & 84.27 & 84.58 & 84.43 \\
        & BRUMO2025 & 94.69 & 94.74 & 94.69 & 94.58 & 94.58 & 94.58 & 94.58 & 94.48 & 94.48 & 94.58 \\
        \cmidrule{2-12}
        & Avg. & 77.71 & 77.61 & 77.67 & 77.58 & 77.68 & 77.73 & 77.67 & 77.61 & 77.61 & 77.52 \\
      \bottomrule
    \end{tabular}
    }
  \end{center}
  \vskip -0.1in
\end{table}

\subsection{Detailed Results of Step Split Ablation}
\label{sec:appendix_complete_step_split_ablation}
    For the metrics Avg Pass@1, Steps Count, Tokens Count, Confidence, Time, and Reflections Count under different step split methods shown in \autoref{fig:supp_exp_step_split_ablation} of \autoref{sec:supp_exp_step_split_ablation}, we have supplemented the complete experimental data here, as shown in \autoref{tab:ablation_step_split_metric}.

    \begin{table}[ht]
  \caption{Complete ablation study results of different step splitting methods. Using DeepSeek-R1-8B to generate 64 trajectories for each question across 5 benchmarks, computing the \textbf{Analysis Metrics}. HET denotes High-Entropy Token.}
  \label{tab:ablation_step_split_metric}
  \begin{center}
    \resizebox{0.9\textwidth}{!}{
    \begin{tabular}{llcccccccc}
      \toprule
      \multirow{2}{*}{Metric} & \multirow{2}{*}{Benchmark} & \multirow{2}{*}{``\textbackslash n\textbackslash n''} & \multirow{2}{*}{``\textbackslash n''} & \multirow{2}{*}{HET} & \multicolumn{4}{c}{Fixed Window} \\
      \cmidrule{6-9}
      & & & & & 256 & 512 & 1024 & 2048 \\
      \midrule
      \multirow{7}{*}{Steps Count}
      & HMMT2025  & 140.14 & 148.97 & 145.93 & 146.54 & 149.36 & 149.27 & 148.54 \\
      & GPQA-D    & 31.13 & 31.27 & 42.70 & 27.04 & 41.31 & 34.45 & 41.23 \\
      & AIME2024  & 109.60 & 113.68 & 108.24 & 113.79 & 114.13 & 114.28 & 108.59 \\
      & AIME2025  & 127.08 & 129.03 & 123.41 & 127.02 & 128.76 & 130.05 & 123.92 \\
      & BRUMO2025 & 134.88 & 139.24 & 133.23 & 136.11 & 136.73 & 140.19 & 134.80 \\
      \cmidrule(lr){2-9}
      & Avg.      & 67.66 & 69.56 & 74.78 & 66.22 & 75.63 & 71.81 & 74.34 \\
      \midrule
      \multirow{7}{*}{Tokens Count}
      & HMMT2025  & 29368.05 & 29487.89 & 29272.19 & 29448.24 & 28928.73 & 29365.49 & 29333.62 \\
      & GPQA-D    & 9721.65 & 9675.06 & 13493.27 & 9855.01 & 13386.14 & 10356.28 & 13442.01 \\
      & AIME2024  & 23438.78 & 23302.43 & 23356.76 & 23353.05 & 23244.43 & 23333.48 & 23458.67 \\
      & AIME2025  & 26367.72 & 26182.84 & 26782.66 & 26036.69 & 26090.98 & 26258.83 & 26818.51 \\
      & BRUMO2025 & 23270.82 & 23201.30 & 22994.43 & 22975.56 & 23063.84 & 23293.29 & 23064.78 \\
      \cmidrule(lr){2-9}
      & Avg.      & 15717.76 & 15663.19 & 18062.42 & 15741.19 & 17894.01 & 16094.58 & 18055.93 \\
      \midrule
      \multirow{7}{*}{Confidence}
      & HMMT2025  & 15.83 & 15.67 & 15.67 & 15.83 & 15.91 & 15.92 & 15.77 \\
      & GPQA-D    & 12.42 & 12.27 & 11.27 & 12.30 & 11.32 & 12.28 & 11.34 \\
      & AIME2024  & 16.45 & 16.33 & 16.35 & 16.49 & 16.56 & 16.56 & 16.39 \\
      & AIME2025  & 16.58 & 16.47 & 16.02 & 16.63 & 16.66 & 16.68 & 16.09 \\
      & BRUMO2025 & 16.51 & 16.35 & 16.48 & 16.55 & 16.52 & 16.61 & 16.51 \\
      \cmidrule(lr){2-9}
      & Avg.      & 13.90 & 13.75 & 13.10 & 13.84 & 13.24 & 13.85 & 13.17 \\
      \midrule
      \multirow{7}{*}{Time (ms/it)}
      & HMMT2025  & 14.31 & 24.74 & 50.62 & 24.43 & 23.98 & 24.41 & 31.99 \\
      & GPQA-D    & 4.17 & 6.12 & 20.10 & 6.15 & 19.74 & 8.25 & 9.36 \\
      & AIME2024  & 11.47 & 18.32 & 37.40 & 18.29 & 18.14 & 18.27 & 23.73 \\
      & AIME2025  & 12.73 & 21.87 & 46.51 & 21.21 & 21.40 & 21.73 & 27.06 \\
      & BRUMO2025 & 11.39 & 18.70 & 37.21 & 18.28 & 32.93 & 18.67 & 29.21 \\
      \cmidrule(lr){2-9}
      & Avg.      & 7.30 & 11.70 & 28.72 & 11.58 & 21.39 & 12.97 & 16.39 \\
      \midrule
      \multirow{7}{*}{Reflections Count}
      & HMMT2025  & 2.747 & 0.791 & 1.764 & 1.354 & 0.271 & 0.064 & 0.085 \\
      & GPQA-D    & 0.931 & 1.103 & 2.757 & 1.116 & 1.275 & 0.227 & 0.559 \\
      & AIME2024  & 2.393 & 0.597 & 1.576 & 1.105 & 0.196 & 0.040 & 0.078 \\
      & AIME2025  & 2.573 & 0.509 & 2.513 & 1.573 & 0.329 & 0.060 & 0.174 \\
      & BRUMO2025 & 3.197 & 0.772 & 1.862 & 1.584 & 0.284 & 0.062 & 0.017 \\
      \cmidrule(lr){2-9}
      & Avg.      & 1.609 & 0.939 & 2.445 & 1.225 & 0.896 & 0.163 & 0.381 \\
      \bottomrule
    \end{tabular}
    }
  \end{center}
  \vskip -0.5in
\end{table}



\end{document}